\documentclass{article}
\newif\ifswitch
\switchfalse
\ifswitch
\usepackage{microtype}
\usepackage{graphicx}
\usepackage{subfigure}
\usepackage{booktabs} 

\usepackage{hyperref}

\usepackage{icml2023}

\usepackage{amsmath}
\usepackage{amssymb}
\usepackage{mathtools}
\usepackage{amsthm}

\usepackage[capitalize,noabbrev]{cleveref}

\theoremstyle{plain}
\newtheorem{theorem}{Theorem}[section]
\newtheorem{proposition}[theorem]{Proposition}
\newtheorem{lemma}[theorem]{Lemma}
\newtheorem{corollary}[theorem]{Corollary}
\theoremstyle{definition}
\newtheorem{definition}[theorem]{Definition}
\newtheorem{assumption}[theorem]{Assumption}
\theoremstyle{remark}
\newtheorem{remark}[theorem]{Remark}

\usepackage[textsize=tiny]{todonotes}
\usepackage[utf8]{inputenc} 
\usepackage[T1]{fontenc}    
\usepackage{hyperref}       
\usepackage{url}            
\usepackage{booktabs}       
\usepackage{amsfonts}       
\usepackage{nicefrac}       
\usepackage{microtype}      
\usepackage{xcolor}         
\usepackage{xspace}
\newenvironment{myitemize}{\begin{list}{$\bullet$}
		{\setlength{\topsep}{1mm}
			\setlength{\itemsep}{0.25mm}
			\setlength{\parsep}{0.25mm}
			\setlength{\itemindent}{0mm}
			\setlength{\partopsep}{0mm}
			\setlength{\labelwidth}{15mm}
			\setlength{\leftmargin}{4mm}}}{\end{list}}

\usepackage{graphicx}
\usepackage{caption}
\usepackage{wrapfig,amsmath,amssymb,bm,comment,color,mathbbol}
\usepackage{breakurl,epsf,fmtcount,semtrans,multirow,boldline}
\usepackage{tcolorbox}
\tcbuselibrary{skins}
\usepackage{tikz}
\usepackage{newtxmath}%

\definecolor{darkred}{RGB}{150,0,0}
\definecolor{darkgreen}{RGB}{0,150,0}
\definecolor{darkblue}{RGB}{0,0,200}
\hypersetup{colorlinks=true, linkcolor=darkred, citecolor=darkgreen, urlcolor=darkblue}

\newtheorem{observation}{Observation}

\newtheorem{hypothesis}{Hypothesis}

\def \endprf{\hfill {\vrule height6pt width6pt depth0pt}}

\newenvironment{proofsk}{\noindent {\bf Proof sketch.} }{\endprf}

\newcommand{\tsn}[1]{{\left\vert\kern-0.25ex\left\vert\kern-0.25ex\left\vert #1 
    \right\vert\kern-0.25ex\right\vert\kern-0.25ex\right\vert}}

\newcommand{\red}{\textcolor{black}}
\newcommand{\icl}{\text{ICL}\xspace}
\newcommand{\riskhm}{\text{risk}(h,m)}

\newcommand{\prm}{\text{prompt}}
\newcommand{\TF}{{\texttt{TF}}}

\newcommand{\eps}{\varepsilon}
\newcommand{\beps}{\boldsymbol{\varepsilon}}
\newcommand{\bbeps}{\bar{\boldsymbol{\varepsilon}}}
\newcommand{\ept}{\eps_{\TF}}

\newcommand{\bxi}{\boldsymbol{\xi}}

\newcommand{\cmt}[1]{}

\newcommand{\st}{\star}

\newcommand{\distas}{\overset{\text{i.i.d.}}{\sim}}

\newcommand{\beq}{\begin{equation}}
\newcommand{\ba}{\begin{align}}
\newcommand{\ea}{\end{align}}

\newcommand{\eeq}{\end{equation}}

\newcommand{\nn}{\nonumber}

\newcommand{\A}{{\mtx{A}}}
\newcommand{\Ab}{{\mtx{\bar{A}}}}

\newcommand{\Bb}{{\mtx{\bar{B}}}}

\newcommand{\V}{{\mtx{V}}}

\newcommand{\B}{{{\mtx{B}}}}

\newcommand{\diag}[1]{\text{diag}(#1)}

\newcommand{\Lc}{{\cal{L}}}

\newcommand{\Lch}{{\widehat{\cal{L}}}}

\newcommand{\Nc}{{\cal{N}}}

\newcommand{\Dc}{{\cal{D}}}

\newcommand{\Pb}{{\mtx{P}}}

\newcommand{\Cb}{{\mtx{C}}}
\newcommand{\Ceb}{C_0}

\newcommand{\Eb}{{\mtx{E}}}
\newcommand{\Hb}{{\mtx{H}}}
\newcommand{\Gc}{{\cal{G}}}
\newcommand{\Qc}{{\cal{Q}}}
\newcommand{\Zc}{{\cal{Z}}}

\newcommand{\bsi}{{\boldsymbol{{\sigma}}}}
\newcommand{\bSi}{{\boldsymbol{{\Sigma}}}}

\newcommand{\onebb}{{\mathbf{1}}}
\newcommand{\Iden}{{\mtx{I}}}
\newcommand{\M}{{\mtx{M}}}

\newcommand{\order}[1]{{\cal{O}}(#1)}
\newcommand{\OR}{\text{OR}}

\newcommand{\z}{{\vct{z}}}

\newcommand{\sft}[1]{\text{softmax}(#1)}
\newcommand{\tn}[1]{\|{#1}\|_{\ell_2}}

\newcommand{\tone}[1]{\|{#1}\|_{\ell_1}}

\newcommand{\tin}[1]{\|{#1}\|_{\ell_\infty}}

\newcommand{\tsub}[1]{\|{#1}\|_{\psi_2}}

\newcommand{\Cc}{\mathcal{C}}
\newcommand{\Ac}{\mathcal{A}}
\newcommand{\Uc}{\mathcal{U}}

\newcommand{\Rc}{\mathcal{R}}

\newcommand{\bt}{{\boldsymbol{\beta}}}
\newcommand{\bT}{{\boldsymbol{\Theta}}}

\newcommand{\bal}{\texttt{Alg}}
\newcommand{\bat}{{\widetilde{\texttt{Alg}}}}

\newcommand{\bah}{{\widehat{\bal}}}
\newcommand{\berm}{\bal^{\texttt{ERM}}}

\newcommand{\Sc}{\mathcal{S}}
\newcommand{\Sca}{{\mathcal{S}_{\text{all}}}}
\newcommand{\Dca}{{\text{MTL}}}
\newcommand{\Fca}{{\mathcal{F}_{\text{all}}}}

\newcommand{\Nn}{\mathcal{N}}

\newcommand{\vb}{\vct{v}}

\newcommand{\FB}{\mathbb{F}}

\newcommand{\Xb}{\mtx{\bar{X}}}
\newcommand{\xb}{\vct{\bar{x}}}

\newcommand{\w}{\vct{w}}

\newcommand{\s}{\vct{s}}
\newcommand{\ab}{\vct{a}}
\newcommand{\bb}{\vct{b}}

\newcommand{\g}{{\vct{g}}}

\newcommand{\Zb}{{\Sc}}

\newcommand{\Tc}{\mathcal{T}}
\newcommand{\tfr}{{\text{TFR}}}
\newcommand{\dtask}{\Dc_{\text{task}}}

\newcommand{\Fc}{\mathcal{F}}

\newcommand{\Xc}{\mathcal{X}}
\newcommand{\Yc}{\mathcal{Y}}
\newcommand{\ys}{\y}

\newcommand{\xp}[1]{\x^{(#1)}_\prm}
\newcommand{\xpp}[1]{\x'^{(#1)}_\prm}
\newcommand{\xpa}[1]{\x^{(#1)}_{\prm, a}}
\newcommand{\xpt}[2]{\x^{(#1)}_{\prm,{#2}}}

\newcommand{\m}{\vct{m}}

\newcommand{\pnorm}[1]{\left\|#1 \right\|_p}

\newcommand{\wnorm}[2]{\left\|#1\right\|_{#2}}

\newcommand{\x}{\vct{x}}

\newcommand{\y}{\vct{y}}

\newcommand{\W}{\mtx{W}}

\newcommand{\Wc}{{\cal{W}}}

\definecolor{emmanuel}{RGB}{255,127,0}

\newcommand{\pb}{{\vct{p}}}

\newcommand{\R}{\mathbb{R}}

\renewcommand{\P}{\operatorname{\mathbb{P}}}
\newcommand{\E}{\operatorname{\mathbb{E}}}

\newcommand{\vct}[1]{\bm{#1}}
\newcommand{\mtx}[1]{\bm{#1}}

\newcommand{\X}{{\mtx{X}}}
\newcommand{\Y}{{\mtx{Y}}}
\newcommand{\Vb}{{\mtx{V}}}

\newcommand\scalemath[2]{\scalebox{#1}{\mbox{\ensuremath{\displaystyle #2}}}}
\newcommand*{\QE}{\hfill\ensuremath{\square}}%

\newcommand{\Scnt}[2]{{{\Sc}^{(#1)}_{#2}}}
\newcommand{\fal}[2]{{f^\bal_{\Scnt{#1}{#2}}}}
\newcommand{\fac}[2]{{f^\bal_{\Sc{#1}{#2}}}}

\newcommand{\falp}[2]{{f^{\bal'}_{\Scnt{#1}{#2}}}}

\newcommand{\yal}[1]{{\f^\bal_{#1}}}

\newcommand{\f}{{\boldsymbol{f}}}

\newcommand{\Rmtl}{{R_\text{MTL}}}

\newcommand{\Gch}{{\widehat\Gc}}

\newcommand{\Rch}{{\widehat\Rc}}

\newcommand{\stt}{{\dagger}}

\icmltitlerunning{Generalization and Stability in In-context Learning}
\else
\usepackage{fullpage}
\title{Transformers as Algorithms:\\ 
Generalization and Stability in In-context Learning}

\newtheorem{theorem}{Theorem}[section]

\newtheorem{lemma}[theorem]{Lemma}

\newtheorem{definition}[theorem]{Definition}
\newtheorem{assumption}[theorem]{Assumption}

\newenvironment{proof}{\noindent {\bf Proof.} }{\endprf\par}

\author{Yingcong Li\thanks{University of California, Riverside, \{yli692@,mildi001@,oymak@ece.\}ucr.edu.} \and
  M. Emrullah Ildiz\footnotemark[1] \and Dimitris Papailiopoulos\thanks{University of Wisconsin-Madison, dimitris@papail.io}\and Samet Oymak\footnotemark[1]}

\date{}
\usepackage{changepage}

\makeatletter
\patchcmd{\@maketitle}{\begin{center}}{\begin{adjustwidth}{-0.5in}{-0.5in}\begin{center}}{}{}
\patchcmd{\@maketitle}{\end{center}}{\end{center}\end{adjustwidth}}{}{}
\makeatother
\fi

\begin{document}
\ifswitch
\twocolumn[
\icmltitle{Transformers as Algorithms:\\ 
Generalization and Stability in In-context Learning}



\icmlsetsymbol{equal}{*}

\begin{icmlauthorlist}
\icmlauthor{Firstname1 Lastname1}{equal,yyy}
\icmlauthor{Firstname2 Lastname2}{equal,yyy,comp}
\icmlauthor{Firstname3 Lastname3}{comp}
\icmlauthor{Firstname4 Lastname4}{sch}
\icmlauthor{Firstname5 Lastname5}{yyy}
\icmlauthor{Firstname6 Lastname6}{sch,yyy,comp}
\icmlauthor{Firstname7 Lastname7}{comp}
\icmlauthor{Firstname8 Lastname8}{sch}
\icmlauthor{Firstname8 Lastname8}{yyy,comp}
\end{icmlauthorlist}

\icmlaffiliation{yyy}{Department of XXX, University of YYY, Location, Country}
\icmlaffiliation{comp}{Company Name, Location, Country}
\icmlaffiliation{sch}{School of ZZZ, Institute of WWW, Location, Country}

\icmlcorrespondingauthor{Firstname1 Lastname1}{first1.last1@xxx.edu}
\icmlcorrespondingauthor{Firstname2 Lastname2}{first2.last2@www.uk}

\icmlkeywords{Machine Learning, ICML}

\vskip 0.3in
]



\printAffiliationsAndNotice{\icmlEqualContribution} 

\else
\maketitle
\fi
\begin{abstract}
In-context learning (ICL) is a type of prompting where a transformer model operates on a sequence of (input, output) examples and performs inference on-the-fly. In this work, we formalize in-context learning as an \emph{algorithm learning} problem where a transformer model implicitly constructs a hypothesis function at inference-time. We first explore the statistical aspects of this abstraction through the lens of multitask learning: We obtain generalization bounds for ICL when the input prompt is (1) a sequence of i.i.d.~(input, label) pairs or (2) a trajectory arising from a dynamical system. The crux of our analysis is relating the excess risk to the stability of the algorithm implemented by the transformer. We characterize when transformer/attention architecture provably obeys the stability condition and also provide empirical verification. For generalization on unseen tasks, we identify an inductive bias phenomenon in which the transfer learning risk is governed by the task complexity and the number of MTL tasks in a highly predictable manner. Finally, we provide numerical evaluations that (1) demonstrate transformers can indeed implement near-optimal algorithms on classical regression problems with i.i.d.~and dynamic data, (2) provide insights on stability, and (3) verify our theoretical predictions.
\end{abstract}

\ifswitch
\begin{figure}[t]
  \begin{center}
    \includegraphics[width=0.45\textwidth]{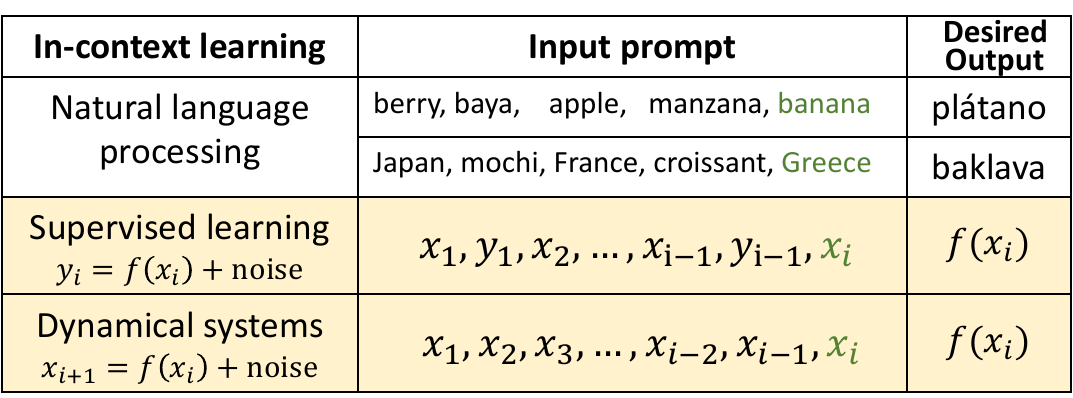}
    \includegraphics[width=0.45\textwidth]{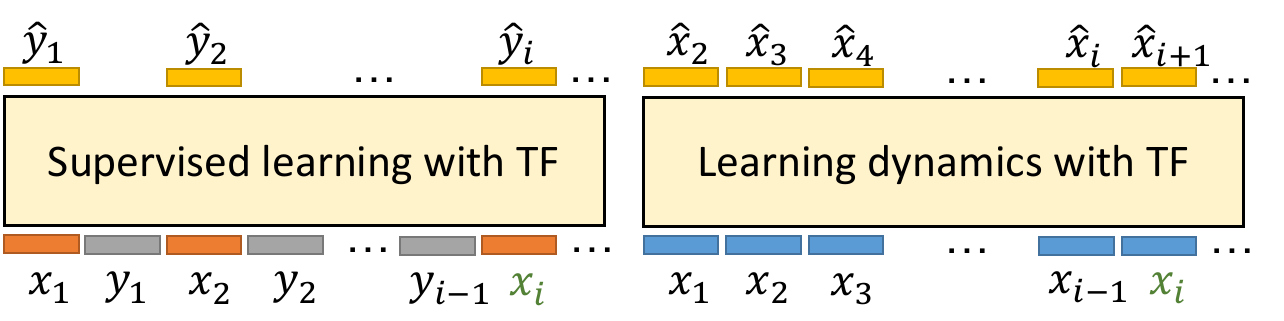}
  \end{center}
  \vspace{-7pt}
  \caption{\small{
  Examples of in-context learning. We focus on the lower two settings in the table where a transformer admits a supervised dataset or dynamical system trajectory as a prompt. Then, it auto-regressively predicts the output following an input example $\x_i$ based on the prompt $(\x_1,\dots,\x_i)$.}}
  \label{fig:icl}
  \vspace{-10pt}
\end{figure} 
\else
\begin{wrapfigure}{r}{0.46\textwidth}\vspace{-32pt}
  \begin{center}
    \includegraphics[width=0.45\textwidth]{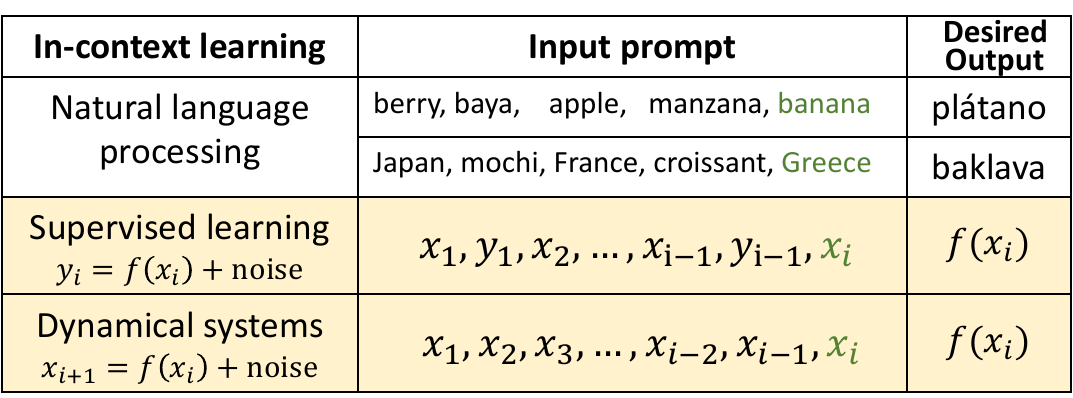}
    \includegraphics[width=0.45\textwidth]{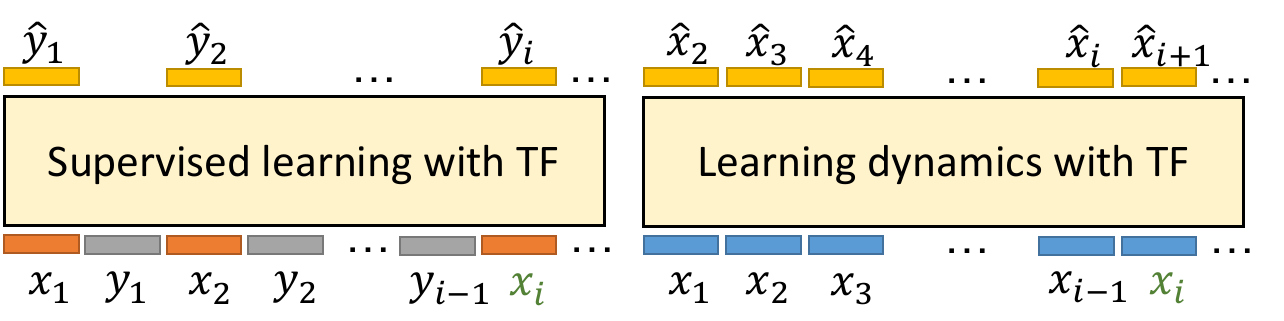}
  \end{center}\vspace{-18pt}
  \caption{\small{
  Examples of ICL. We focus on the lower two settings where a transformer admits a supervised dataset or a dynamical system trajectory as a prompt. Then, it auto-regressively predicts the output following an input example $\x_i$ based on the prompt $(\x_1,\dots,\x_i)$.}}\vspace{-3pt}
  \label{fig:icl}
\end{wrapfigure} 
\fi



\section{Introduction}
Transformer (TF) models were originally developed for NLP problems to address long-range dependencies through the attention mechanism.  In recent years, language models have become increasingly large, with some boasting billions of parameters (e.g., GPT-3 has 175B, and PaLM has 540B parameters~\cite{brown2020language,chowdhery2022palm}). 
It is perhaps not surprising that these large language models (LLMs) have achieved state-of-the-art performance on a wide range of natural language processing tasks. What is surprising is the ability of some of these LLMs to perform \emph{in-context learning} (ICL), i.e., to adapt and perform a specific task given a short prompt, in the form of instructions, and a small number of examples \cite{brown2020language}. 
These models' ability to learn in-context without explicit training allows them to efficiently perform new tasks without a need for updating model weights.

Figure \ref{fig:icl} illustrates examples of ICL where a transformer makes a prediction on an example based on a few (input, output) examples provided within its prompt. For NLP, the examples may correspond to pairs of (question, answer)'s or translations. Recent works \cite{garg2022can,laskin2022context} demonstrate that ICL can also be used to infer general functional relationships. For instance, \cite{hollmann2022tabpfn,garg2022can} aims to solve certain supervised learning problems where they feed an entire training dataset $(\x_i,f(\x_i))_{i=1}^{n-1}$ {as the input prompt}, expecting that conditioning the TF model {on this prompt} would allow it to make an accurate prediction on a new input point $\x_n$. As discussed in \cite{akyurek2022learning,garg2022can}, this provides an implicit optimization flavor to ICL, where the model \emph{implicitly trains} on the data provided within the prompt, and performs inference on test points. 

\begin{figure*}
\centering
\begin{minipage}{0.31\textwidth}
	\centering
	\begin{tikzpicture}
	\centering
	\node at (0,0) {\includegraphics[scale=0.25]{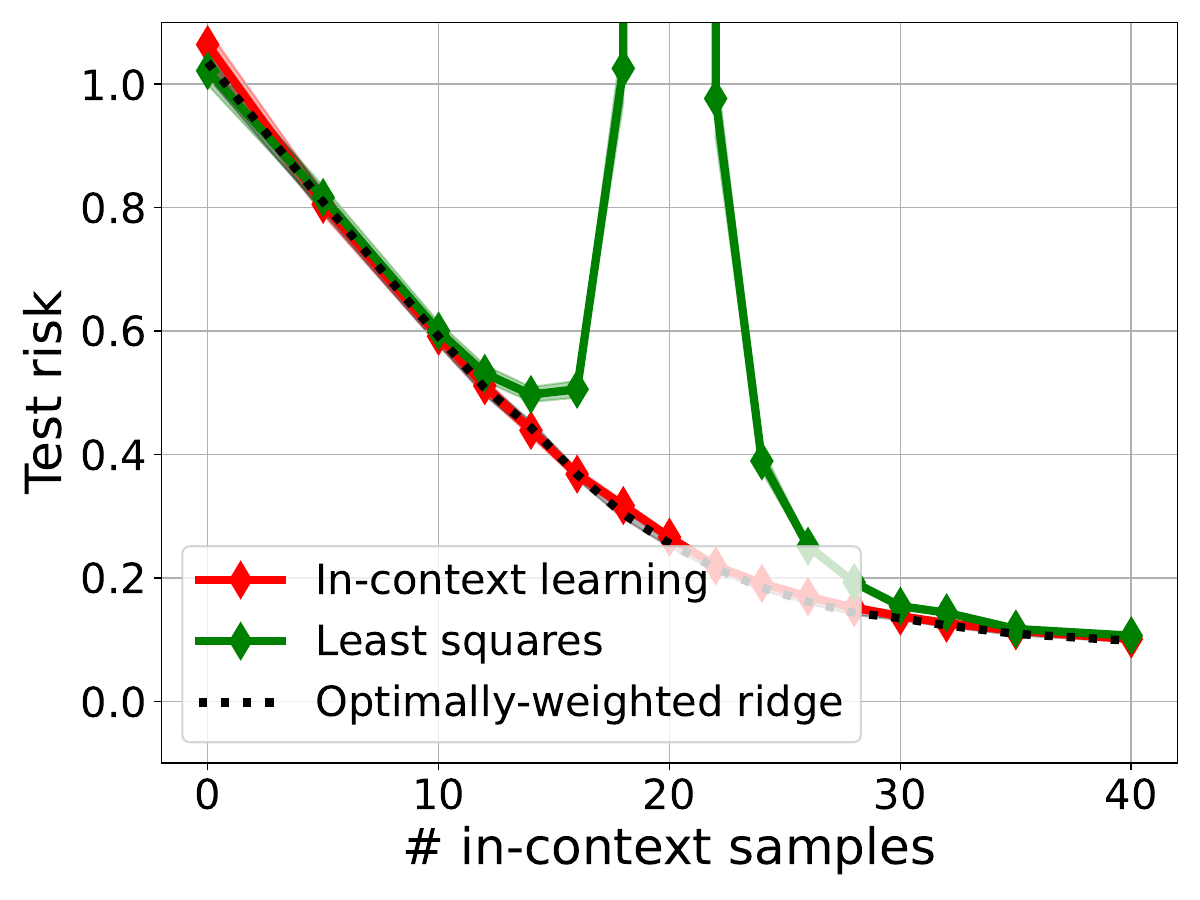}};
	\node at (0, -2.2) {\small{(a) Noisy linear regression}}; 
	\end{tikzpicture}
\end{minipage}
\hspace{5pt}
\begin{minipage}{0.31\textwidth}
	\centering
	\begin{tikzpicture}
		\centering
		\node at (0,0) {\includegraphics[scale=0.25]{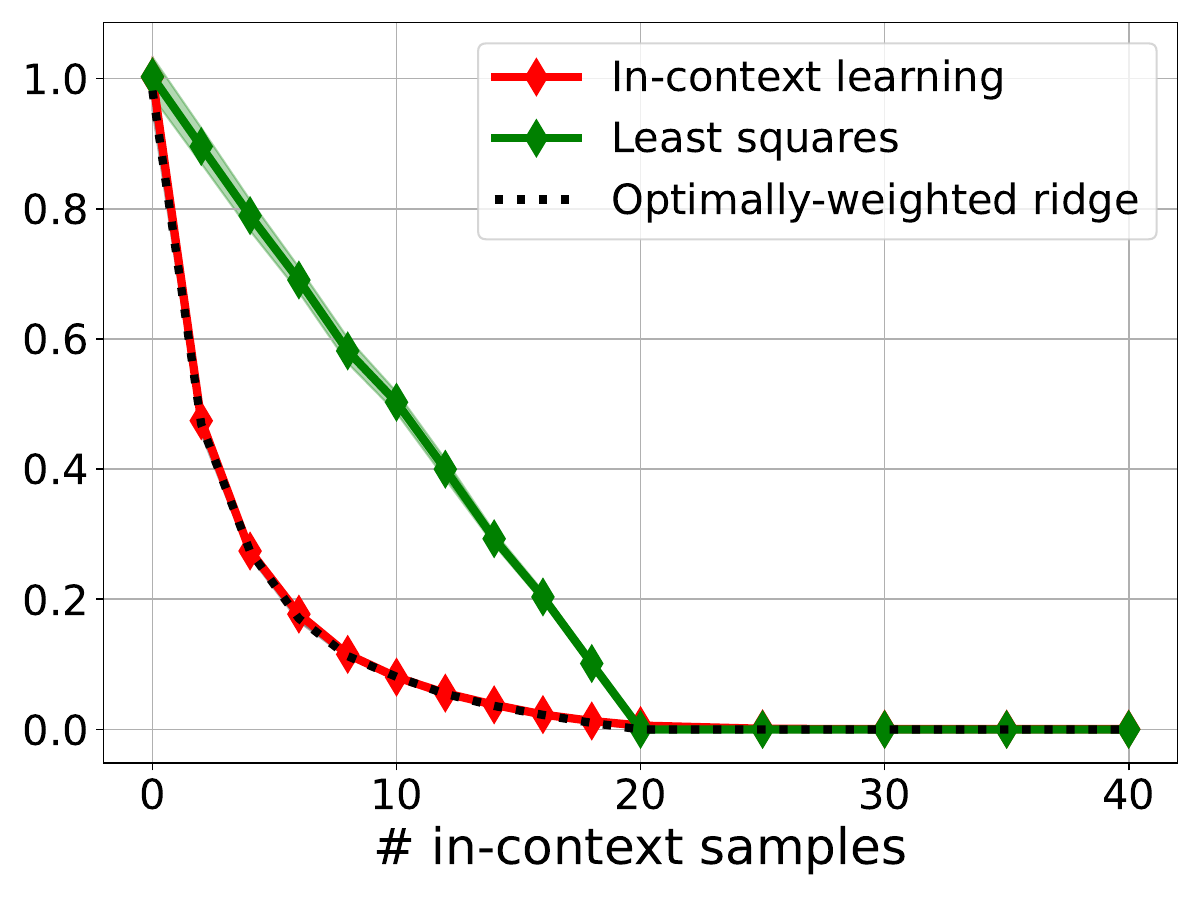}};
		\node at (0.1, -2.2) {\small{(b) Linear tasks with covariance prior}}; 
	\end{tikzpicture}
\end{minipage}
\hspace{5pt}
\begin{minipage}{0.31\textwidth}
	\centering
	\begin{tikzpicture}
		\centering
		\node at (0,0) {\includegraphics[scale=0.25]{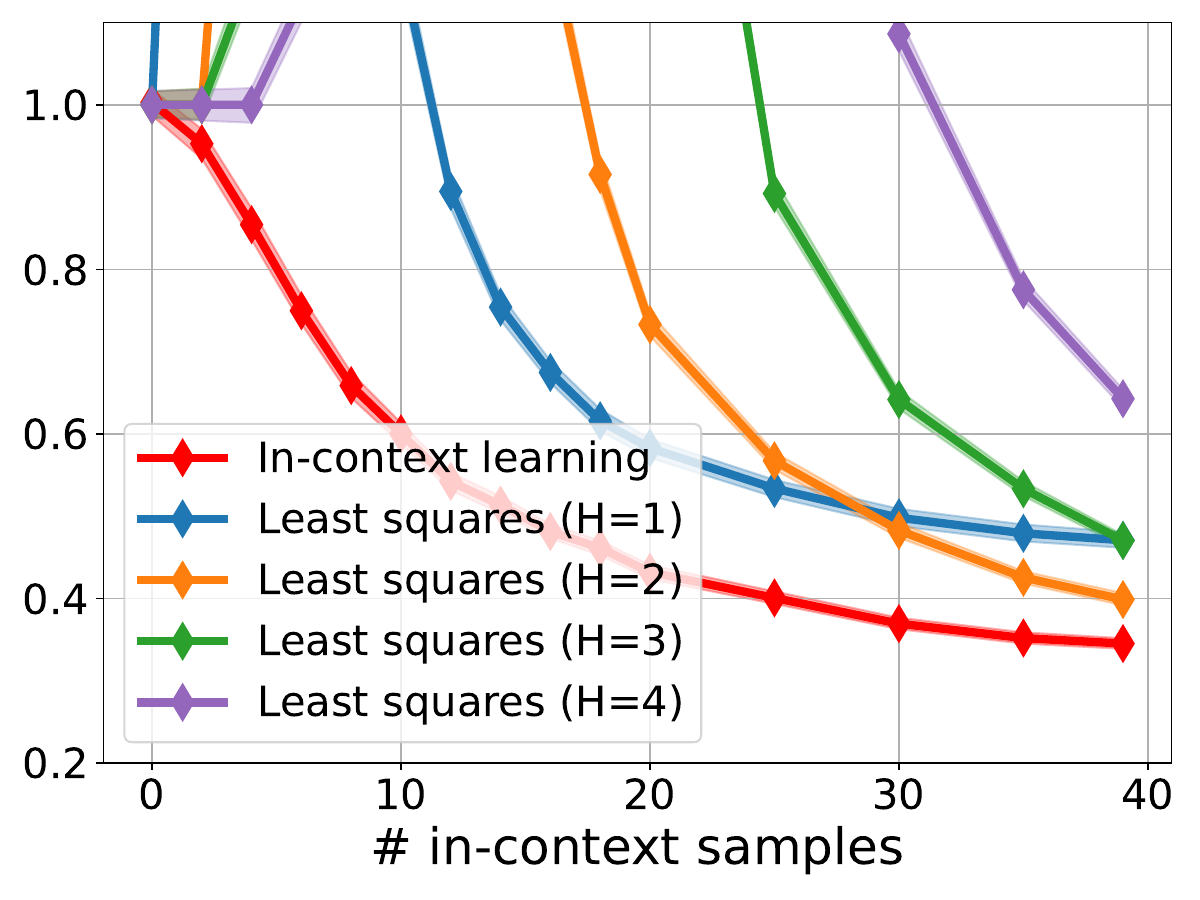}};
		\node at (0, -2.2) {\small{(c) Partially-observed LDS}}; 
	\end{tikzpicture}
\end{minipage}
\vspace{-5pt}
\caption{\small{Examples of {\emph{algorithm learning}} in three ICL settings: \textbf{(a) Noisy linear regression:} $y_i\sim \Nc(\x_i^\top\bt,\sigma^2)$ with $\x_i,\bt\sim\Nc(0,\Iden)$. \textbf{(b) Linear data with covariance prior:} $y_i=\x_i^\top\bt$ with $\bt\sim\Nc(0,\bSi)$ with non-isotropic $\bSi$. \textbf{(c) Partially observed linear dynamics:} $\x_t=\Cb\s_t$ and $\s_{t+1}\sim\Nc(\A\s_t,\sigma^2\Iden)$ with randomly sampled $\Cb,\A$. Each setting trains a transformer with large number of random regression tasks and evaluates on a new task from the same distribution. In (a) and (b), ICL performances match Bayes-optimal decision (weighted linear ridge regression) that adapt to noise level $\sigma$ and covariance prior $\bSi$ on the tasks. (c) shows that ICL outperforms auto-regressive least-squares estimators with varying memory $H$. {ICL is able to implement competitive ML algorithms} by leveraging the task prior learned during training. See Sec \ref{sec exp} for experimental details.}}\label{fig:main_exp}
\vspace{-5pt}
\end{figure*}
Our work formalizes in-context learning from a statistical lens, abstracting the transformer as {a learning algorithm} where the goal is inferring the correct (input, ouput) functional relationship from prompts. We focus on a meta-learning setting where the model is trained on many tasks, allowing ICL to generalize to both new and previously-seen tasks. Our main contributions are:

\begin{myitemize}
\item \textbf{Generalization bounds (Sec~\ref{sec mtl} \& \ref{sec dynamic}):} Suppose the model is trained on $T$ tasks each with a {data-sequence} containing $n$ examples. {During training, each sequence is fed to the model auto-regressively} as depicted in Figure \ref{fig:icl}. By abstracting ICL as an \emph{algorithm learning} problem, we establish a multitask (MTL) generalization rate of $1/\sqrt{nT}$ for i.i.d.~as well as dynamic data. In order to achieve the proper dependence on the sequence length ($1/\sqrt{n}$ factor), we overcome temporal dependencies by relating generalization to algorithmic stability \cite{bousquet2000algorithmic}. Experiments demonstrate that (1) \icl can select near-optimal algorithms for flagship regression problems as illustrated in Figure \ref{fig:main_exp} and (2) \icl indeed benefits from learning across the full task sequence in line with theory. 

\item \textbf{Stability of transformer architectures (Sec~\ref{sec stable}\&\ref{sec exp}):} We verify our stability assumptions that facilitate favorable generalization rates. Theoretically we identify when self-attention enjoys favorable stability properties through a tight analysis that quantify the influence of one token on another. Empirically, we show that ICL predictions become more stable to input perturbations as the prompt length increases. We also find that training with noisy data helps promote stability. 


\item \textbf{From multitask to meta-learning (Sec~\ref{sec transfer}):}  We provide insights into how our MTL bounds can inform generalization ability of ICL on previously unseen tasks (i.e.~transfer learning). Our experiments also reveal an intriguing \emph{inductive bias phenomenon}: The transfer risk is governed by the \emph{task complexity} (i.e.~functions $f$ in Fig \ref{fig:icl}) and the number of MTL tasks $T$ in a highly predictable fashion \red{and exhibits little dependence on the complexity of the TF architecture}.



\end{myitemize}
The remainder of the paper is organized as follows. The next section discusses connections to prior art and Section \ref{sec setup} introduces the problem setup. Section \ref{sec mtl} provides our main theoretical guarantees for ICL and stability of transformers. Section \ref{sec transfer} extends our arguments and experiments to the transfer learning setting. Section \ref{sec dynamic} extends our results to learning stable dynamical systems where each prompt corresponds to a system trajectory. \red{In Section \ref{sec approx}, we explain how ICL can be interpreted as an implicit model selection procedure building on the algorithm learning viewpoint.} Finally, Section \ref{sec exp} provides numerical evaluations.




%

%
%




\subsection{Related work}\label{sec related}

With the success of large language models, prompting methods have witnessed immense interest \cite{lester2021power}. ICL \cite{brown2020language,olsson2022context} is a prompting strategy where a transformer serves as an on-the-fly predictive model through conditioning on a sequence of input/output examples $(\x_1,f(\x_1),\dots\x_{n-1},f(\x_{n-1}),\x_n)$. Our work is inspired by \cite{garg2022can} which studies \icl in synthetic settings and demonstrates transformers can serve as complex classifiers through ICL. In parallel, \cite{hollmann2022tabpfn} uses ICL as an AutoML (i.e.~model-selection, hyperparameter tuning) framework where they plug in a dataset to transformer and use it as a classifier for new test points. Our formalism on \emph{algorithm learning} provides a justification on how transformers can accomplish this with proper meta-training. \cite{xie2021explanation} interprets ICL as implicit Bayesian inference and develops guarantees when the training distribution is a mixture of
HMMs. Very recent works \cite{von2022transformers,akyurek2022learning,dai2022can} aim to relate ICL to running gradient descent algorithm over the input prompt. \cite{akyurek2022learning} also provides related observations regarding the optimal decision making ability of ICL for linear models. 
Unlike prior ICL works, we provide finite sample generalization guarantees and our theory extends to temporally-dependent prompts (e.g.~when prompts are trajectories of dynamical systems). Dynamical systems in turn relate to a recent work by \cite{laskin2022context} who use ICL for reinforcement learning. 

This work is also related to the literature on the statistical aspects of time-series prediction \cite{yu1994rates,kuznetsov2016time,kuznetsov2014generalization,simchowitz2018learning,mohri2008rademacher} and learning (non)linear dynamics \cite{foster2020learning,ziemann2022single,ziemannlearning,tsiamis2022statistical,sarkar2019near,dean2020sample,sun2022finite,mania2020active,matni2019tutorial,oymak2021revisiting,block2023smoothed} among others. Most of these focus on autoregressive models of order $1$ whereas in ICL we allow for arbitrarily long memory/prompt for predictions. Closer works by \cite{mcdonald2017nonparametric,mohri2010stability} identify broad conditions for time-series learning however they still require finite memory as well as $\beta/\phi$-mixing assumptions. We remark that mixing assumptions are not really applicable to training sequences/prompts in ICL due to the meta-learning nature of the problem as the sequence elements are coupled through the (stochastic) task functions (see Sec \ref{sec setup}). Our algorithm learning formulation leads to new challenges and insights when verifying the conditions for Azuma-type inequalities and our results are facilitated through connections to algorithmic stability \cite{bousquet2002stability}. We also provide experiments and theory that justify our stability conditions. Further discussion is under Appendix \ref{app related}.




\section{Problem Setup}\label{sec setup}
\noindent\emph{Notation. } Let $\Xc$ be the input feature space, and $\Yc$ be the output/label space. We use boldface for vector variables. $[n]$ denotes the set $\{1,2,\dots,n\}$. $c,C>0$ denote absolute constants and $\|\cdot\|_{\ell_p}$ denotes the $\ell_p$-norm.

\noindent\textbf{In-context learning setting:} We denote {a length-$m$ prompt containing $m-1$ in-context examples and the $m$'th input} by $\xp{m}=(\z_1,\z_2,\dots,\z_{m-1},\x_{m})$. Here $\x_{m}\in\Xc$ is the input to predict and $\z_i\in\Zc$ is the $i$'th in-context example provided within prompt. Let $\TF(\cdot)$ denote a transformer (more generally an auto-regressive model) that admits $\xp{m}$ as its input and outputs a label $\hat\y_m=\TF(\xp{m})$ in $\Yc$. 


\noindent $\bullet$ \emph{Independent $(\x,\y)$ pairs.} Similar to \cite{garg2022can}, we draw i.i.d.~samples $(\x_i,\y_i)_{i=1}^n\in\Zc=\Xc\times{\Yc}$ from a data distribution. Then {a length-$m$ prompt} is written as $\xp{m}=(\x_1,\y_1,\dots \x_{m-1},\y_{m-1},\x_{m})$, and the model predicts $\hat \y_m=\TF(\xp{m})\in{\Yc}$ for $1\leq m\leq n$. 
\ifswitch\smallskip\else\fi


\noindent $\bullet$ \emph{Dynamical systems.} In this setting, the prompt is simply the trajectory generated by a dynamical system, namely, $\xp{m}=(\x_1,\dots \x_{m-1},\x_{m})$ and therefore, $\Zc=\Xc=\Yc$. Specifically, we investigate the state observed setting that is governed by dynamics $f(\cdot)$ via $\x_{m+1}=f(\x_m)+\text{noise}$. Here, $\y_m:=\x_{m+1}$ is the label associated to $\x_m$, and the model admits $\xp{m}$ as input and predicts the next state $\hat\y_m:=\hat\x_{m+1}=\TF(\xp{m})\in\Xc$. 

We first consider the training phase of \icl where we wish to learn a good $\TF(\cdot)$ model through MTL. Suppose we have $T$ tasks associated with data distributions $\{\Dc_{t}\}_{t=1}^T$. Each task independently samples a training sequence $\Sc_t=(\z_{ti})_{i=1}^n$ according to its distribution. $\Sca=\{\Sc_t\}_{t=1}^T$ denote the set of all training {sequences}. We use $\Scnt{m}{t}=(\z_{t1},\dots,\z_{tm})$ to denote a subsequence of $\Sc_t:=\Scnt{n}{t}$ for $m\leq n$ {and $\Scnt{0}{}$ denotes an empty subsequence. }


ICL can be interpreted as an implicit optimization on the subsequence $\Scnt{m}{}=(\z_{1},\z_{2},\dots,\z_{m})$ to make prediction on $\x_{m+1}$. To model this, we abstract the transformer model as a {learning algorithm that maps a sequence of data to a prediction function (e.g.~gradient descent, empirical risk minimization)}. Concretely, let $\Ac$ be a set of algorithm hypotheses such that algorithm $\bal\in\Ac$ maps a sequence of form $\Scnt{m}{}$ into a prediction function $\fal{m}{}:\Xc\rightarrow\Yc$. With this, we represent $\TF$ via
\begin{align}
\TF(\xp{m+1})=\fal{m}{}(\x_{m+1}).\nn
\end{align}
{This abstraction is without losing generality as we have the explicit relation $\fal{m}{}(\x):=\TF((\Scnt{m}{},\x))$.} 
Given training sequences, $\Sca$ and a loss function $\ell(\y,\hat{\y})$, the \icl training can be interpreted as searching for the optimal algorithm $\bal\in\Ac$, and the training objective becomes
\begin{align}\tag{ERM}
    \bah=\underset{\bal\in\Ac}{\arg\min}~&\Lch_{\Sca}(\bal):=\frac{1}{T}\sum_{t=1}^T\Lch_t(\bal)\label{mtl opt}  \\
    \text{where}\quad&\Lch_t(\bal)=\frac{1}{n}\sum_{i=1}^n \ell(\y_{ti},\fal{i-1}{t}(\x_{ti})).\nn
\end{align}
Here, $\Lch_t(\bal)$ is the training loss of task $t$ and $\Lch_\Sca(\bal)$ is the task-averaged MTL loss. Let $\Lc_t(\bal)=\E_{\Sc_t}[\Lch_t(\bal)]$ and $\Lc_\Dca(\bal)=\E[\Lch_\Sca(\bal)]=\frac{1}{T}\sum_{t=1}^T\Lc_t(\bal)$ be the corresponding population risks. {Observe that, task-specific loss $\Lch_t(\bal)$ is an empirical average of $n$ terms, one for each prompt $\xp{i}$.}

To develop generalization bounds, our primary interest is controlling the gap between empirical and population risks. For problem \eqref{mtl opt}, we wish to bound the excess MTL risk
\begin{align}
    \Rmtl(\bah)=\Lc_\Dca(\bah)-\min_{\bal\in\Ac}\Lc_\Dca(\bal).\label{mtl risk}
\end{align}



Following the MTL training \eqref{mtl opt}, we also evaluate the model on previously-unseen tasks; this can be thought of as the transfer learning problem. Concretely, let $\dtask$ be a distribution over tasks and draw a target task $\Tc\sim\dtask$ with data distribution $\Dc_\Tc$ and {a sequence} $\Sc_\Tc=\{\z_i\}_{i=1}^n\sim\Dc_\Tc$. Define the empirical and population risks on $\Tc$ as $\Lch_\Tc(\bal)=\frac{1}{n}\sum_{i=1}^n\ell(\y_{i},\fal{i-1}{\Tc}(\x_{i}))$ and $\Lc_\Tc(\bal)=\E_{\Sc_\Tc}[\Lch_\Tc(\bal)]$. Then the transfer risk of an algorithm $\bal$ is defined as $\Lc_\tfr(\bal)=\E_{\Tc}[\Lc_\Tc(\bal)]$.
With this setup, we are ready to state our main contributions.

\section{Generalization in In-context Learning}\label{sec mtl}

In this section, we study ICL under the i.i.d.~data setting with {training sequences} $\Sc_t=(\x_{ti},\y_{ti})_{i=1}^n\distas\Dc_t$. Section~\ref{sec dynamic} extends our results to dynamical systems. 
%
\subsection{Algorithmic Stability}\label{sec stable}

In ICL a training example $(\x_i,\y_i)$ in the prompt impacts all future decisions of the algorithm from predictions $i+1$ to $n$. This necessitates us to control the stability to input perturbation of the learning algorithm emulated by the transformer. Our stability condition is borrowed from the algorithmic stability literature. As stated in \cite{bousquet2000algorithmic,bousquet2002stability}, the stability level of an algorithm is typically in the order of $1/m$ (for realistic generalization guarantees) where $m$ is the training sample size (in our setting prompt length). This is formalized in the following assumption that captures the variability of the transformer output.
\begin{assumption}[Error Stability \cite{bousquet2002stability}]\label{assump robust}
    Let $\Zb=(\x_i,\y_i)_{i=1}^m$ be a sequence in $\Xc\times \Yc$ with $m\geq1$ and $\Zb^j$ be the sequence where the $j$'th sample of $\Zb$ is replaced by $(\x'_j,\y'_j)$. Error stability holds for a distribution $(\x, \y)\sim\Dc$ if there exists a $K>0$ such that for any $\Zb,(\x_j',\y_j')\in(\Xc\times\Yc)$, $j\leq m$, and $\bal\in\Ac$, we have
    \begin{align}
        \Big|\E_{(\x, \y)}\big[\ell(\y, f^\bal_\Zb(\x)) -\ell(\y,f^\bal_{\Zb^j}(\x))\big] \Big| \leq \frac{K}{m}.  \label{eq robust 1}
    \end{align}
    Let $\rho$ be a distance metric on $\Ac$. Pairwise error stability holds if for all $\bal,\bal'\in\Ac$ we have
    \ifswitch
    \begin{align}
        \Big|&\E_{(\x, \y)}\big[\ell(\y, f^\bal_\Zb(\x))-\ell(\y, f^{\bal'}_\Zb(\x)) \nonumber \\
        &-\ell(\y,f^\bal_{\Zb^j}(\x))+\ell(\y, f^{\bal'}_{\Zb^j}(\x))\big] \Big| \leq \frac{K\rho(\bal,\bal')}{m}.\nonumber  
    \end{align}
    \else
    \begin{align}
        \Big|\E_{(\x, \y)}\big[\ell(\y, f^\bal_\Zb(\x))-\ell(\y, f^{\bal'}_\Zb(\x))
        -\ell(\y,f^\bal_{\Zb^j}(\x))+\ell(\y, f^{\bal'}_{\Zb^j}(\x))\big] \Big| \leq \frac{K\rho(\bal,\bal')}{m}.\nonumber  
    \end{align}
    \fi
    
\end{assumption}
Here \eqref{eq robust 1} is our primary stability condition borrowed from \cite{bousquet2002stability} and ensures that all algorithms $\bal\in\Ac$ are $K$-stable. We will also use the stronger pairwise stability condition to develop tighter generalization bounds.
The following theorem shows that, under mild assumptions, a multilayer transformer obeys the stability condition \eqref{eq robust 1}. The proof is deferred to Appendix \ref{sec stable proof} and Theorem \ref{stable thm}.
\begin{theorem} \label{main stable}Let $\xp{m},\xpp{m}$ be two prompts that only differ at the inputs $\z_j=(\x_j,\y_j)$ and $\z_j'=(\x'_j,\y'_j)$ where $j < m$. Assume inputs and labels lie within the unit Euclidean ball in $\R^d$~\footnote{Here, we assume $\Xc,\Yc\subset\R^d$, otherwise, inputs and labels are both embedded into $d$-dimensional vectors of proper size.}. Shape these prompts into matrices $\X_{\prm},\X'_{\prm}\in\R^{(2m-1)\times d}$ respectively. Let $\TF(\cdot)$ be a $D$-layer transformer as follows: Setting $\X_{(0)}:=\X_{\prm}$, the $i$'th layer applies MLPs and self-attention\footnote{In self-attention the softmax function is applied to each row.} and outputs
\ifswitch
\begin{align}
    \X_{(i)} =&\texttt{Parallel\_MLPs}(\texttt{ATTN}(\X_{(i-1)})) \nonumber  \\
    &\text{where}\quad \texttt{ATTN}(\X):=\sft{\X \W_i \X^\top} \X\Vb_i. \nonumber
\end{align}
\else
\begin{align}
    \X_{(i)} =\texttt{Parallel\_MLPs}(\texttt{ATTN}(\X_{(i-1)})) \quad
    \text{where}\quad \texttt{ATTN}(\X):=\sft{\X \W_i \X^\top} \X\Vb_i. \nonumber
\end{align}
\fi
 Assume $\TF$ is normalized as $\|\Vb\|\leq 1$, $\|\W\|\leq \Gamma/2$ and MLPs obey $\texttt{MLP}(\x)=\texttt{ReLU}(\M\x)$ with $\|\M\|\leq 1$. \red{Let $\TF$ output the last token of the final layer $\X_{(D)}$ that correspond to the query $\x_m$.} Then,
\[
|\TF(\xp{m})-\TF(\xpp{m})|\leq \frac{2}{2m-1}((1+\Gamma)e^\Gamma)^D.
\]
Thus, assuming loss $\ell(\y,\cdot)$ is $L$-Lipschitz, the algorithm induced by $\TF(\cdot)$ obeys \eqref{eq robust 1} with $K=2L((1+\Gamma)e^\Gamma)^D$.
\end{theorem}
A few remarks are in place. First, the dependence on depth is exponential. However, this is not as prohibitive for typical transformer architectures which tend to not be very deep. For example, the different variants of GPT-2 and BERT have between 12-48 layers \cite{hugface}. \red{In our theorem, the upper bound on $\Gamma$ helps ensure that one token cannot have substantial influence on another one. In Appendix \ref{app stable}, we provide a more general version of this result which also covers our stronger stability assumption for dynamical systems (see Theorem \ref{stable thm}). Importantly, we also show that our theorem is rather tight (see Sec \ref{sec unstable}): (1) Stability can fail if $\Gamma$ is allowed to be logarithmic in $m$ indicating the tightness of our $e^\Gamma/m$ bound. (2) It is also critical that the modified token is not the last one (i.e.~$j<m$ condition), otherwise stability can again fail.} The key technicality in our result is establishing the stability of the self-attention layer which is the central component of a transformer, see Lemma \ref{lem stable}.
Finally, Figure \ref{fig:robust} provides numerical evidence for multiple \icl problems and demonstrate that stability of GPT-2 architecture's predictions with respect to inputs indeed improves with longer prompts in line with theory.


\begin{figure*}
	\centering
	\begin{minipage}{0.50\textwidth}
		\centering
		\begin{tikzpicture}
		\centering
		\node at (0,0) {\includegraphics[scale=0.3]{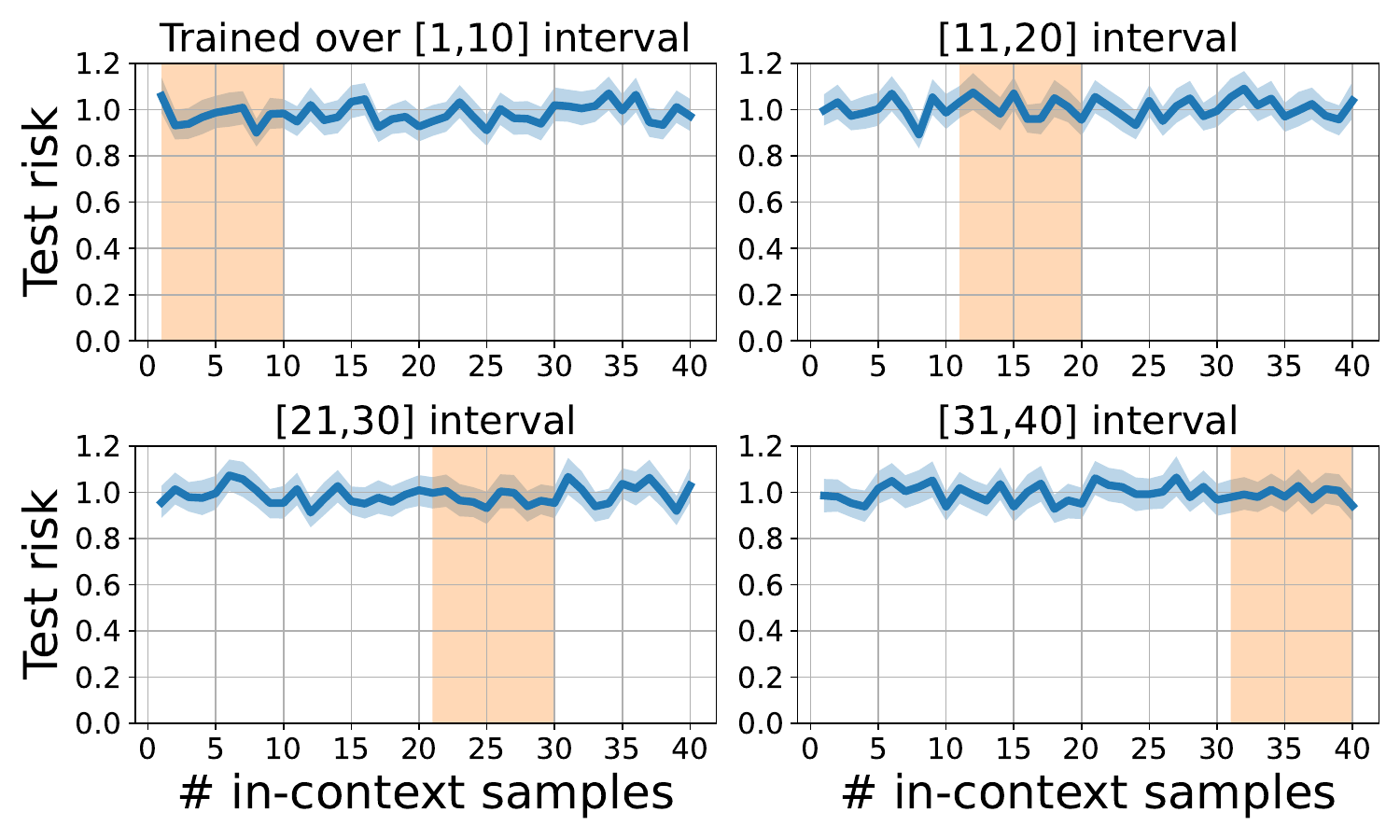}};
		\node at (0, -2.5) {\small{(a) ICL trained with $10$ prompts over a range}}; 
		\end{tikzpicture}
	\end{minipage}
	\hspace{0pt}
	\begin{minipage}{0.35\textwidth}
		\centering
		\begin{tikzpicture}
			\centering
			\node at (0,0) {\includegraphics[scale=0.3]{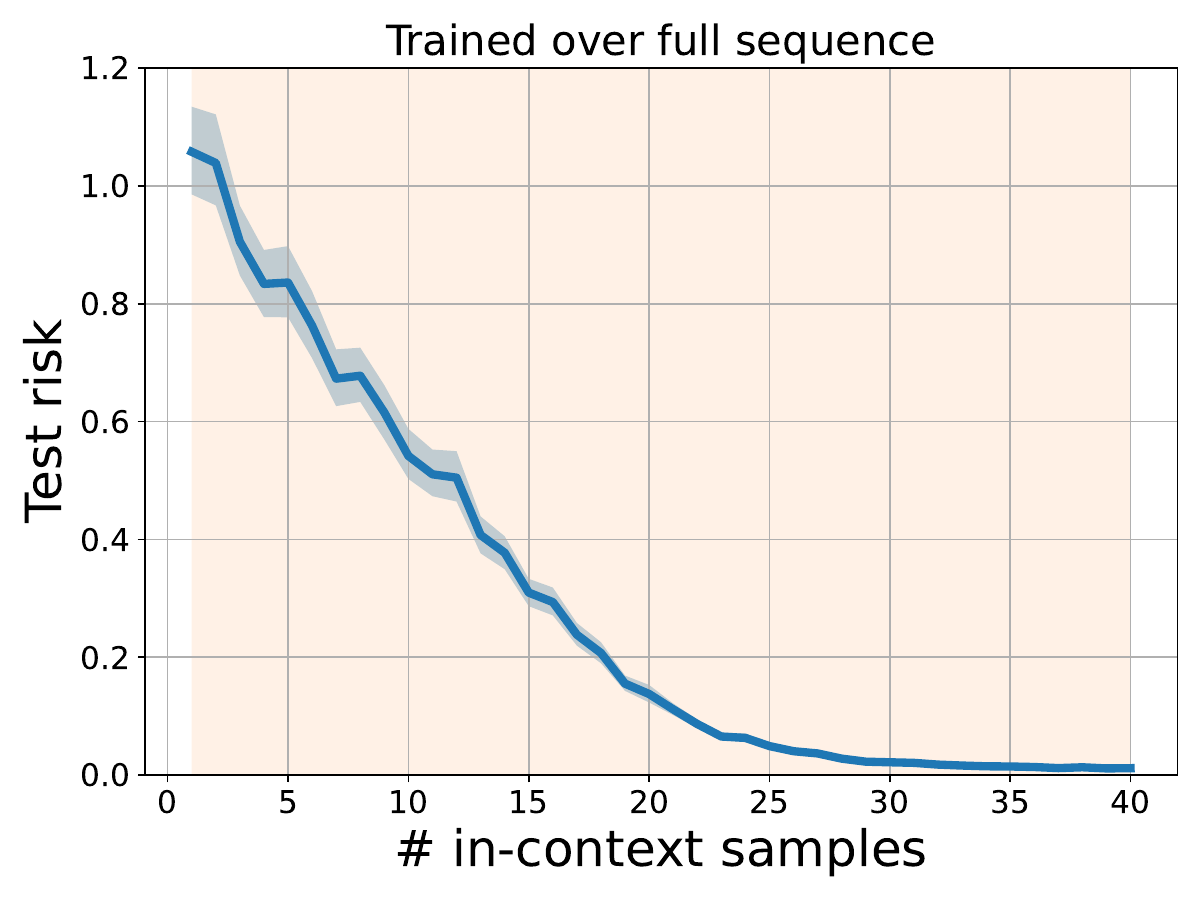}};
			\node at (0.1, -2.5) {\small{(b) ICL with all $n=40$ prompts}}; 
		\end{tikzpicture}
	\end{minipage}
	\vspace{-7pt}
	\caption{\small{The benefit of learning across the full task sequence: \textbf{Right side:} Standard \ref{mtl opt} where each task trains with all $n=40$ prompts. \textbf{Left side:} \ref{mtl opt} focuses on different parts of the trajectory by fitting $n/4=10$ prompts per task over $i\in[1,10]$ to $[31,40]$ (highlighted as the orange ranges). We train with $T=6.4$ million random linear regression tasks and display the performance on new tasks (i.e.~transfer risk). Right side learns to solve linear regression via ICL whereas left side fails to do so even when restricted to their target ranges.}}
 
	%
	 \label{fig:vary_n}
	\end{figure*}
	
\subsection{Generalization Bounds}\label{sec mtl bound}
We are ready to establish generalization bounds by leveraging our stability conditions. We use covering numbers (i.e.~metric entropy) to control model complexity.
\begin{definition}[Covering number]\label{def cover} Let $\Qc$ be any hypothesis set and $d(q,q')\geq0$ be a distance metric over $q,q'\in\Qc$. Then, $\bar\Qc=\{q_1,\dots,q_N\}$ is an $\eps$-cover of $\Qc$ with respect to $d(\cdot,\cdot)$ if for any $q\in\Qc$, there exists $q_i\in\bar\Qc$ such that $d(q,q_i)\leq\eps$. The $\eps$-covering number $\Nc(\Qc,d,\eps)$ is the cardinality of the minimal $\eps$-cover.
\end{definition}
To cover the algorithm space $\Ac$, we need to introduce a distance metric. We formalize this in terms of the prediction difference between the two algorithms on the worst-case data-sequence.
\begin{definition}[Algorithm distance]\label{def distance} Let $\Ac$ be an algorithm hypothesis set and $\Sc=(\x_i,\y_i)_{i=1}^n$ be a sequence that is admissible {for some task $t\in[T]$}. {For any pair $\bal,\bal'\in\Ac$, define the distance metric $\rho(\bal,\bal'):=\sup_{\Sc}\frac{1}{n}\sum_{i=1}^n\tn{\fal{i-1}{}(\x_{i})-\falp{i-1}{}(\x_{i})}$.}
\end{definition}
We note that the distance $\rho$ is controlled by the Lipschitz constant of the transformer architecture (i.e.~the largest gradient norm with respect to the model weights). Following Definitions~\ref{def cover}\&\ref{def distance}, the $\eps$-covering number of the hypothesis set $\Ac$ is $\Nc(\Ac,\rho,\eps)$. This brings us to our main result on the MTL risk of \eqref{mtl opt}.

\begin{theorem} \label{thm mtl}
    Suppose $\Ac$ is $K$-stable per Assumption~\ref{assump robust} for all $T$ tasks and the loss function $\ell(\y,\cdot)$ is $L$-Lipschitz taking values over $[0,1]$. 
    Let $\bah$ be the empirical solution of \eqref{mtl opt}. Then, with probability at least $1-2\delta$, the excess MTL test risk obeys, 
    \ifswitch$\Rmtl(\bah)\leq$
    \begin{equation}\label{main result 1}
    \begin{split}
        \inf_{\eps>0}\Big\{4L\eps + 2(1+K\log n)\sqrt{\frac{\log(\Nc(\Ac,\rho,\eps)/\delta)}{cnT}}\Big\}.
        \end{split}
    \end{equation}
    \else
    \begin{equation}\label{main result 1}
    \begin{split}
        \Rmtl(\bah)\leq\inf_{\eps>0}\left\{4L\eps + 2(1+K\log n)\sqrt{\frac{\log(\Nc(\Ac,\rho,\eps)/\delta)}{cnT}}\right\}.
        \end{split}
    \end{equation}
    \fi
    Additionally suppose $\Ac$ is $K$-pairwise-stable and set diameter $D:=\sup_{\bal,\bal'\in\Ac}\rho(\bal,\bal')$. Using the convention $x_+=\max(x,1)$, with probability at least $1-4\delta$,
    \ifswitch
    \begin{align}\label{main result 2}
        \begin{split}
            \Rmtl&(\bah)\lesssim\inf_{\eps>0}\Bigg\{L\eps+ \frac{L_++K\log n}{\sqrt{nT}}\cdot\\ &\Big(\int_{\eps}^{D/2}\sqrt{\log\Nc(\Ac,\rho,u)}du+D_+\sqrt{\log\frac{1}{\delta}}\Big)\Bigg\}.
        \end{split}
    \end{align}
    \else
    \begin{align}\label{main result 2}
        \begin{split}
            \Rmtl(\bah)\lesssim\inf_{\eps>0}\Bigg\{L\eps+ \frac{L_++K\log n}{\sqrt{nT}}\Big(\int_{\eps}^{D/2}\sqrt{\log\Nc(\Ac,\rho,u)}du+D_+\sqrt{\log\frac{1}{\delta}}\Big)\Bigg\}.
        \end{split}
    \end{align}
    \fi

\end{theorem}
The first bound \eqref{main result 1} achieves $1/\sqrt{nT}$ rate by covering the algorithm space with resolution $\eps$. For Lipschitz architectures with $\text{dim}(\Ac)$ trainable weights we have $\log\Nc(\Ac,\rho,\eps)\sim\text{dim}(\Ac)\log(1/\eps)$. Thus, up to logarithmic factors, the excess risk is bounded by $\sqrt{\frac{\text{dim}(\Ac)}{nT}}$ and will vanish as $n,T\to\infty$. 
{Note that our bound is also task-dependent through $\rho$ in Def.~\ref{def distance}. For instance, suppose tasks are realizable with labels $\y=f(\x)$ and admissible task sequences have the form $\Sc=(\x_i,f(\x_i))_{i=1}^n$. Then, $\rho$ will depend on the function class of $f$ (e.g.~whether $f$ is a linear model, neural net, etc), specifically, as the function class becomes richer, both $\rho$ and the covering number becomes larger.} 

Under the stronger pairwise-stability, we can obtain a bound in terms of Dudley's entropy integral which arises from a chaining argument. This bound is typically in the same order as the Rademacher complexity of the function class with $T\times n$ samples \cite{wainwright2019high}. Note that achieving $1/\sqrt{T}$ dependence is rather straightforward as tasks are sampled independently. Thus, the main feature of Theorem \ref{thm mtl} is obtaining the multiplicative $1/\sqrt{n}$ term by overcoming temporal dependencies. Figure \ref{fig:vary_n} shows that training with full {sequence} is indeed critical for ICL accuracy. \ifswitch\else\smallskip\fi


\begin{proofsk}
The main component of the proof is to find a concentration bound on {$|\Lc_{\Dca}(\bal)-\Lch_{\Sca}(\bal)|$} for a fixed algorithm $\bal \in \Ac$. To achieve this bound, we introduce the  sequence of variables $X_{t,i}=\E\left[\frac{1}{n}\sum_{j=1}^n \ell(\y_{tj},\fal{j-1}{t}(\x_{tj}))~\big|~\Scnt{i}{t}\right]$ for $0\leq i\leq n$, which forms a martingale by construction. The critical stage is upper bounding the martingale differences $|X_{t,i}-X_{t,i-1}|$ through our stability assumption, namely, increments are at most $1+\sum_{j=i}^n\frac{K}{j}\lesssim 1+K\log n$. Then, we utilize Azuma-Hoeffding's inequality to achieve a concentration bound on {$\left|\frac{1}{T}\sum_{t=1}^TX_{t,0} - X_{t, n}\right|$}, which is equivalent to {$|\Lc_{\Dca}(\bal)-\Lch_{\Sca}(\bal)|$}. To conclude, we make use of covering/chaining arguments to establish a uniform concentration bound on {$|\Lc_{\Dca}(\bal)-\Lch_{\Sca}(\bal)|$} for all $\bal \in \Ac$. The details are provided in Appendix \ref{app thm proof}. 
\end{proofsk}\ifswitch\else\smallskip\fi

\noindent\textbf{Multiple sequences per task.} Finally consider a setting where each task is associated with $M$ independent sequences with size $n$. This typically arises in reinforcement learning problems (e.g.~dynamical systems in Sec. \ref{sec dynamic}) where we collect data through multiple rollouts each leading to independent sequences. In this setting, the statistical error rate improves to $1/\sqrt{nMT}$ as discussed in Appendix \ref{app thm proof}. In the next section, we will contrast MTL vs transfer learning by letting $M\rightarrow \infty$. This way, even if $n$ and $T$ are fixed, the model will fully learn the $T$ source tasks during the MTL phase as the excess risk vanishes with $M\rightarrow\infty$.

\begin{figure*}
	\centering
	\begin{minipage}{0.23\textwidth}
		\centering
		\begin{tikzpicture}
		\centering
		\node at (0,0) {\includegraphics[scale=0.2]{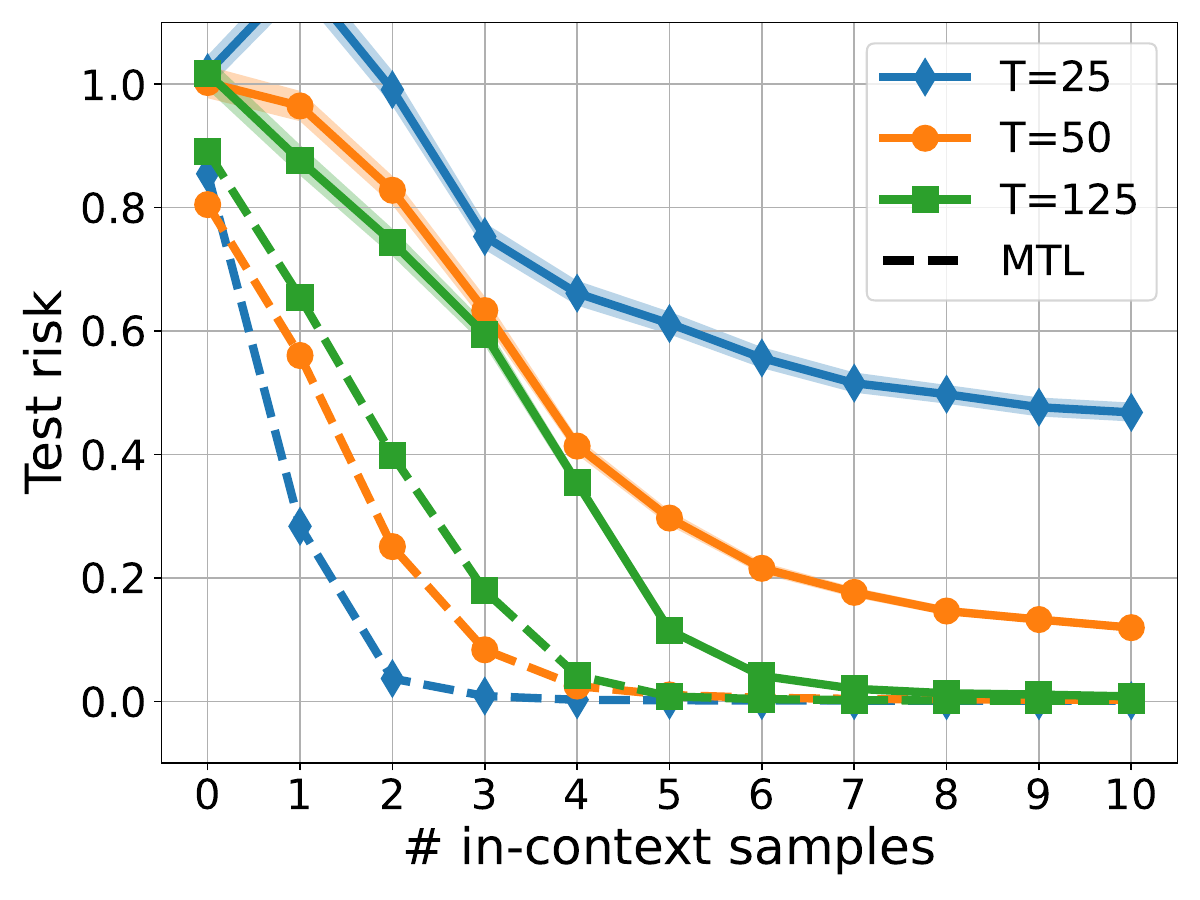}};
		\node at (0, -1.8) {\small{(a) $d=5$}}; 
		\end{tikzpicture}
	\end{minipage}
	\hspace{2pt}
	\begin{minipage}{0.23\textwidth}
		\centering
		\begin{tikzpicture}
			\centering
			\node at (0,0) {\includegraphics[scale=0.2]{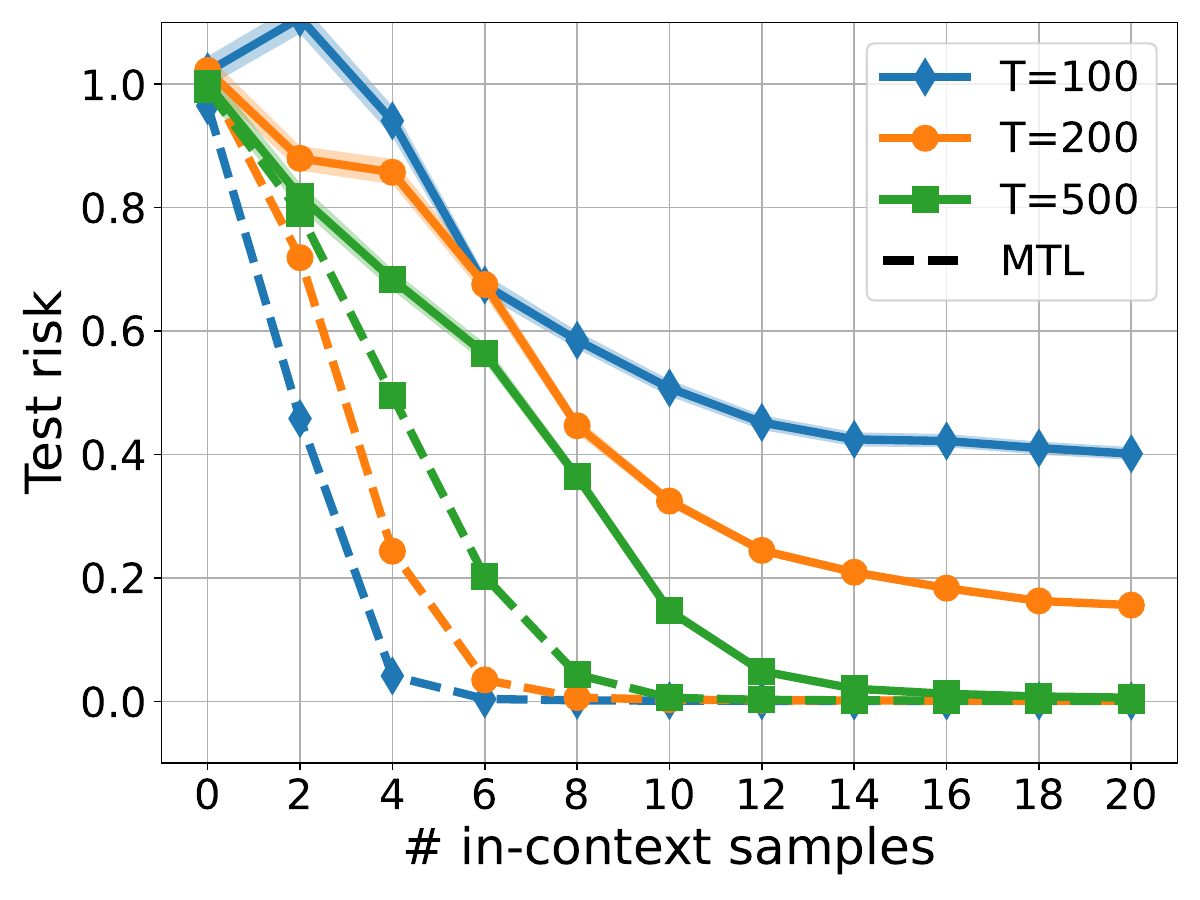}};
			\node at (0.1, -1.8) {\small{(b) $d=10$}}; 
		\end{tikzpicture}
	\end{minipage}
	\hspace{2pt}
	\begin{minipage}{0.23\textwidth}
		\centering
		\begin{tikzpicture}
			\centering
			\node at (0,0) {\includegraphics[scale=0.2]{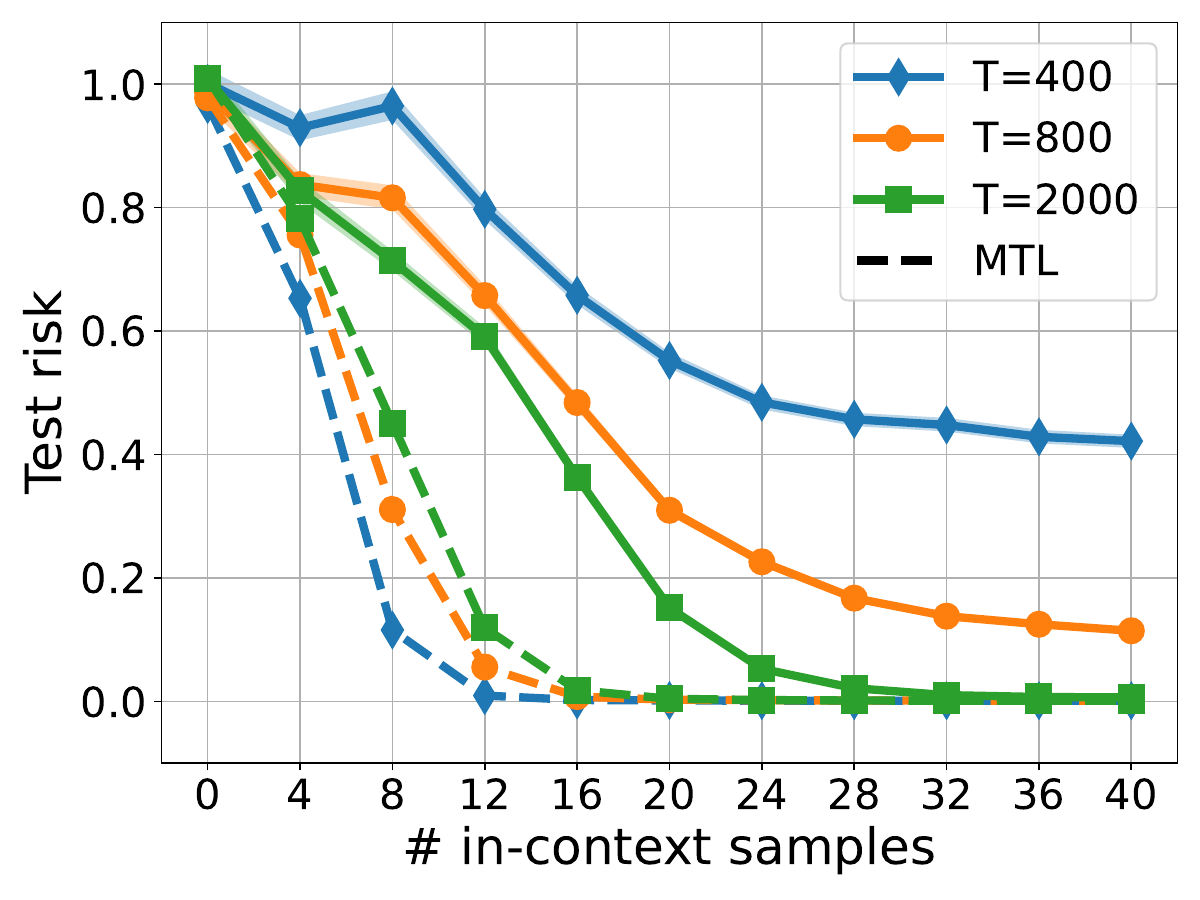}};
			\node at (0, -1.8) {\small{(c) $d=20$}}; 
		\end{tikzpicture}
	\end{minipage}
	\hspace{2pt}
	\begin{minipage}{0.23\textwidth}
		\centering
		\begin{tikzpicture}
			\centering
			\node at (0,0) {\includegraphics[scale=0.2]{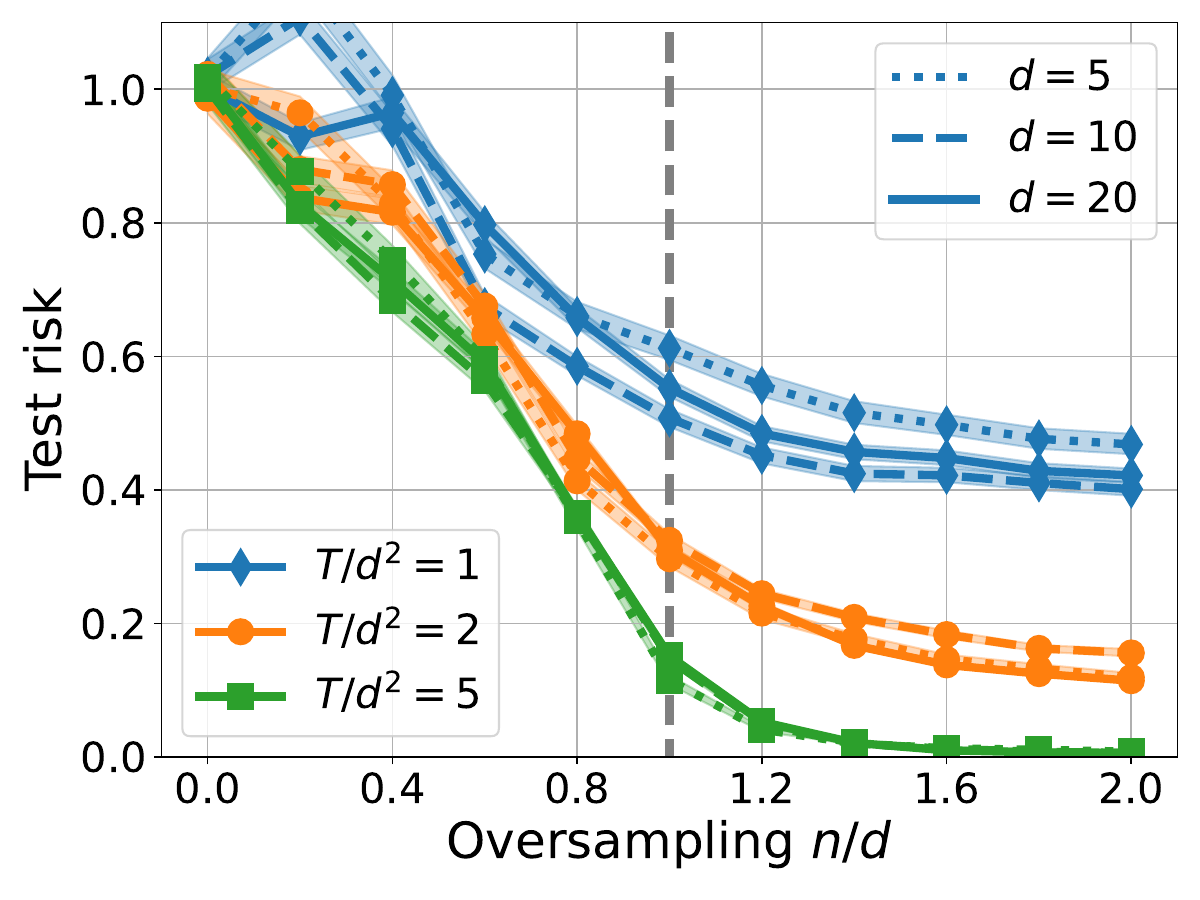}};
			\node at (0, -1.8) {\small{(d) $d=5,10,20$ overlayed}}; 
		\end{tikzpicture}
	\end{minipage}
	\vspace{-7pt}
	\caption{\small{In Figures (a,b,c,), we plot the $d\in \{5,10,20\}$-dimensional results for transfer and MTL risk curves with the same GPT-2 architecture. Figure (d) overlays (a,b,c) to reveal that transfer risks are aligned for fixed $(n/d,T/d^2)$ choice.}}\label{fig:transfer}
	\vspace{-10pt}
	\end{figure*}
\section{Generalization and Inductive Bias on Unseen Tasks}\label{sec transfer}
In this section, we explore transfer learning to assess the performance of ICL on new tasks: The MTL phase generates a model $\bah$ trained on $T$ source tasks and we use $\bah$ to predict a target task $\Tc$. Consider a meta-learning setting where $T$ sources are drawn from the distribution $\dtask$ and we evaluate the transfer risk on a new $\Tc\sim \dtask$. We aim to control the transfer risk $\Lc_\tfr(\bah)=\E[\Lc_\Tc(\bah)]$ in terms of the MTL risk $\Lc_\Dca(\bah)$. When the source tasks are i.i.d, one can use a standard generalization analysis to bound the transfer risk as follows $\Lc_\tfr(\bah)-\Lc_\Dca(\bah)\lesssim \sqrt{{\log(\Nc(\Ac,\rho,\eps)}/{T}}$ (see Thm \ref{tfr thm}).

Here, an important distinction with MTL is that transfer risk decays as $1/\text{poly}(T)$ because the unseen tasks induce a distribution shift, which, typically, cannot be mitigated with more samples $n$ or more sequences-per-task $M$. \ifswitch\else\smallskip\fi

\noindent$\bullet$ \textbf{Inductive Bias in Transfer Risk.} Before investigating distribution shift, let us consider the following question: While $1/\text{poly}(T)$ behavior may be unavoidable, is it possible that dependence on architectural complexity $\text{dim}(\Ac)$ is avoidable? Perhaps surprisingly, we answer this question affirmatively through experiments on linear regression. In what follows, during MTL pretraining, we train with $M\rightarrow\infty$ independent sequences per task to minimize population MTL risk $\Lc_\Dca(\cdot)$. We then evaluate resulting $\bah$ on different dimensions $d$ and numbers of MTL tasks $T$. Figures \ref{fig:transfer}(a,b,c) display the MTL and transfer risks for dimensions $d=5,10,20$. In each figure, we evaluate the results on $T=\{1,2,5\}\times d^2$ and the $x$-axis moves from $0$ to $n=2d$. Each task has isotropic features, noiseless labels and task vectors $\bt\sim\Nn(0,\Iden_d)$. Here, our first observation is that, the Figures \ref{fig:transfer}(a,b,c) seem (almost perfectly) aligned with each other, that is, each figure exhibits identical MTL and transfer risk curves. To further elucidate this, Figure \ref{fig:transfer}(d) integrates the transfer risk curves from $d=5,10,20$ and overlays them together. This alignment indicates that, for a fixed point $\alpha=n/d$ and $\beta=T/d^2$, the transfer risks remain unchanged. Here, $n$ proportional to $d$ can be attributed to linearity, thus, the more surprising aspect is the dependence on $T$: This is because rather than ${\text{dim}(\Ac)}/{T}$ (where $\Ac$ is fixed to a GPT-2 architecture), the generalization risk behaves like ${d^2}/{T}$. Thus, rather than model complexity, what matters seems to be the task complexity $d$. \red{In support of this hypothesis, Figure~\ref{fig:d5 overlay} trains ICL on GPT-2 architectures with up to 64 times different parameter counts and reveals that transfer risk indeed exhibits little dependence on the model complexity $\dim(\Ac)$.}

Inductive bias is a natural explanation of this behavior: Intuitively, the MTL pretraining process identifies a favorable algorithm that lies in the span of the source tasks $\bT_{\Dca}=(\bt_t)_{t=1}^T$. Specifically, while the transformer model can potentially fit MTL tasks through a variety of algorithms, we speculate that the optimization process is implicitly biased  to an algorithm $\bal(\bT_{\Dca})$ (akin to \cite{soudry2018implicit,neyshabur2017geometry}). Such bias would explain the lack of $\text{dim}(\Ac)$ dependence since $\bal(\bT_{\Dca})$ solely depends on the source tasks. While we leave the theoretical exploration of the empirical $d^2/T$ behavior to a future work, below we explain that $d^2/T$ dependence is rather surprising.

To this end, let us first introduce the optimal estimator (in terms of Bayes risk) for linear regression with Gaussian task prior $\bt\sim\Nc(0,\bSi)$. This estimator can be described explicitly \cite{richards2021asymptotics,lindley1972bayes} and is given by the weighted ridge regression solution 
\begin{align}
\hat{\bt}=(\X^\top \X+\sigma^2\bSi^{-1})^{-1}\X^\top\y.
\label{wridge}
\end{align}
Here $\X=[\x_1,\dots,\x_n]^\top\in\R^{n\times d},\y=[\y_1,\dots,\y_n]^\top\in\R^{n}$ are the concatenated features and labels obtained from the task sequence and $\sigma^2$ is the label noise variance. With this in mind, what is the ideal algorithm $\bal(\bT_{\Dca})$ based on the (perfect) knowledge of source tasks? Eqn.~\eqref{wridge} crucially requires the knowledge of the task covariance $\bSi$ and variance $\sigma^2$. Thus, even with the hindsight knowledge that our problem is linear, we have to estimate the task covariance from source tasks. This can be done via the empirical covariance $\hat{\bSi}=\frac{1}{T}\sum_{i=1}^T \bt_i\bt_i^\top$. To ensure $\hat{\bSi}$-weighted LS performs ${\cal{O}}(1)$-close to $\bSi$-weighted LS, we need a spectral norm control, namely, ${\|\bSi-\hat{\bSi}\|}/{\lambda_{\min}(\bSi)}\leq {\cal{O}}(1)$. When $\bSi=\Iden_d$ (as in our experiments) and tasks are isotropic, the latter condition holds with high probability when $T=\Omega(d)$. This is also numerically demonstrated in Figure \ref{fig:hindsight} in the appendix. This behavior is in contrast to the stronger $T\propto d^2$ requirement we observe for ICL and indicates that ICL training may not be sample-optimal in terms of $T$. For instance, $T\propto d^2$ is sufficient to ensure the stronger entrywise control $\tin{\bSi-\hat{\bSi}}\leq O(1)$ rather than spectral norm. \ifswitch\else\smallskip\fi



\ifswitch
\begin{figure}[t]
 \vspace{-0.15cm}
\centering
		\begin{tikzpicture}
			\centering
			\node at (0,0) {\includegraphics[scale=0.26]{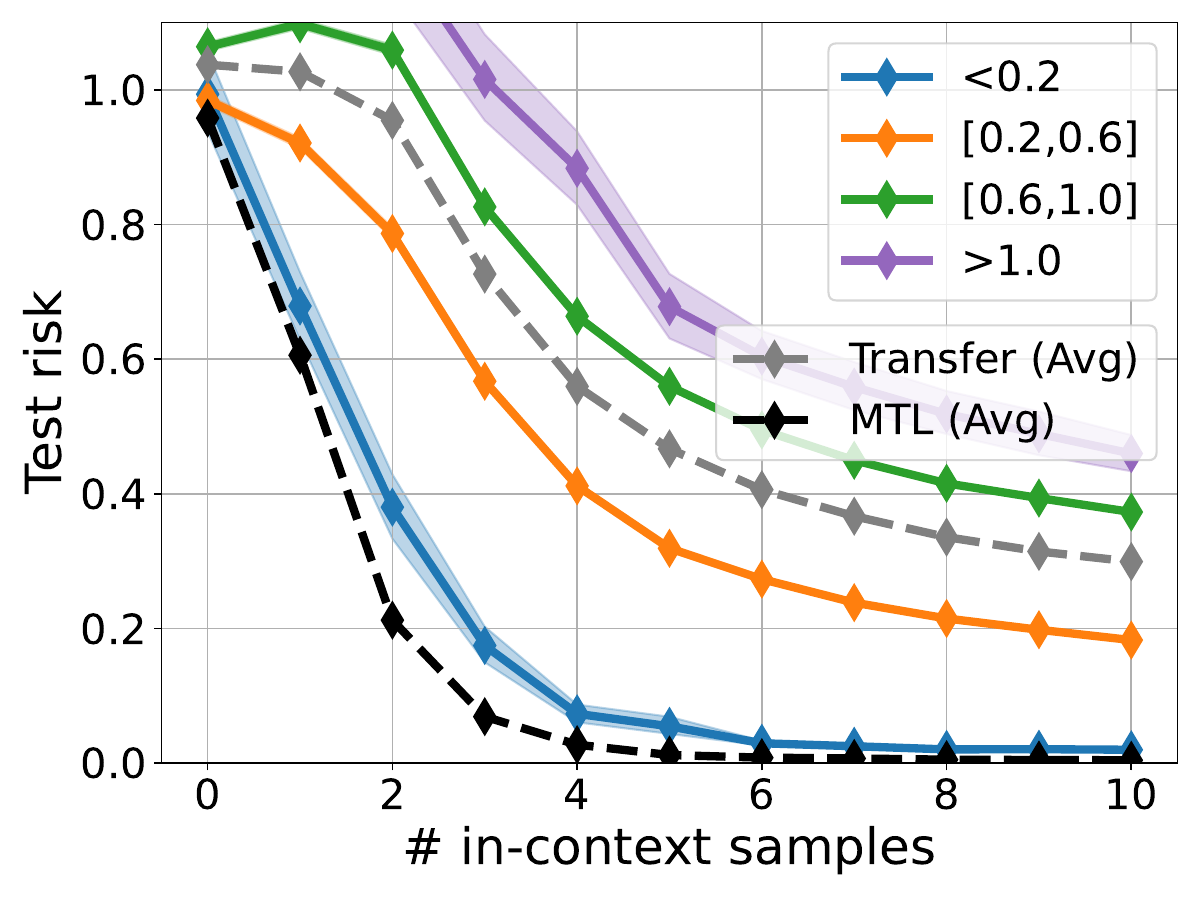}};
		\end{tikzpicture}\vspace{-10pt}
\caption{Transfer risk as a function of the distance to the source (MTL) tasks. Distant tasks (with smaller cosine similarity) generalize worse.}
\label{fig:div}
\vspace{-10pt}
\end{figure}
\else
\begin{wrapfigure}{r}{5.5cm}
\vspace{-0.45cm}
\centering
		\begin{tikzpicture}
			\centering
			\node at (0,0) {\includegraphics[scale=0.26]{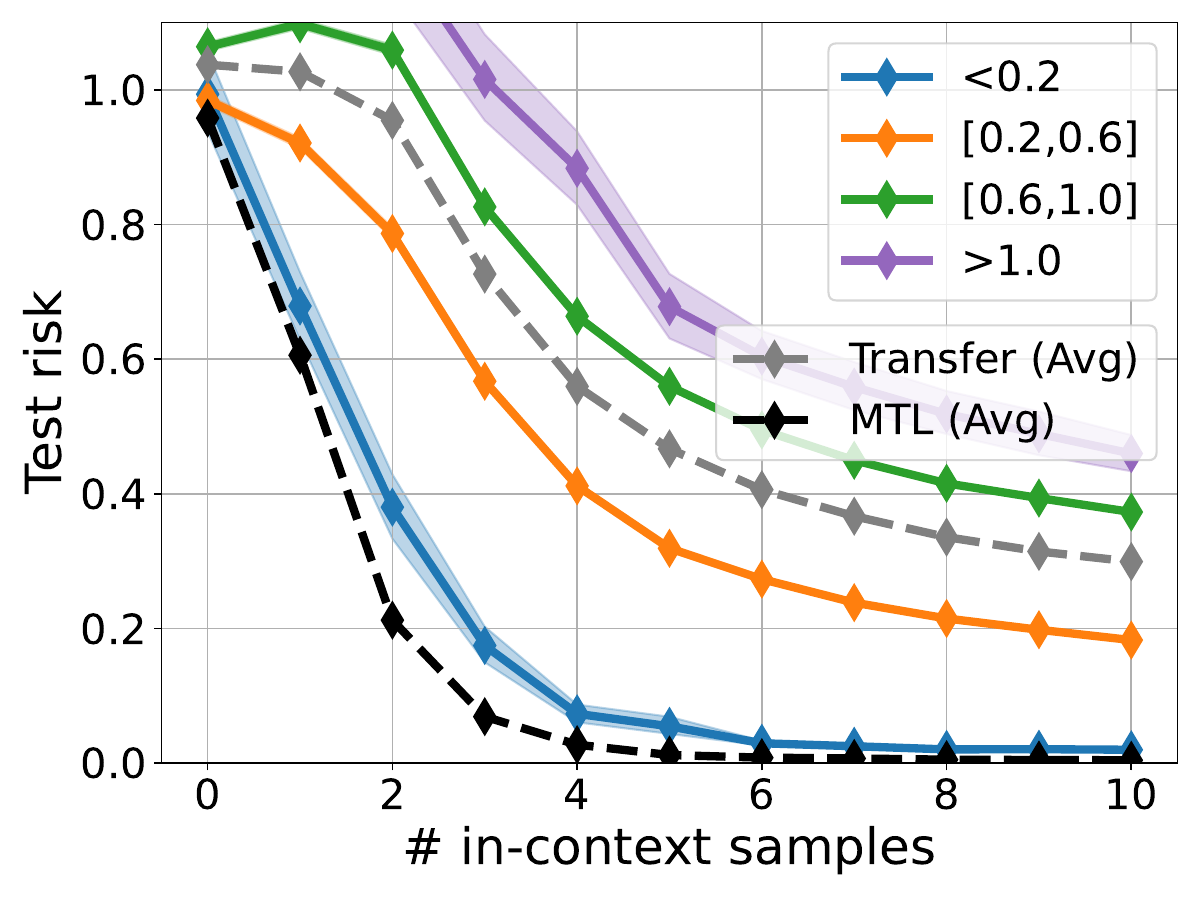}};
		\end{tikzpicture}\vspace{-10pt}
\caption{\small{Transfer risk as a function of distance to the MTL tasks. Distant tasks (with smaller cosine similarity) generalize worse.}}\vspace{-10pt}
\label{fig:div}
\end{wrapfigure}
\fi
\noindent$\bullet$ \textbf{Exploring transfer risk via source-target distance.} Besides drawing source and target tasks from the same $\dtask$, we also investigate transfer risk in an instance specific fashion. Specifically, the population risk of a new task $\Tc$ can be bounded as $\Lc_{\Tc}(\bal)\leq \Lc_{\Dca}(\bal)+\text{dist}(\Tc,(\Dc_t)_{t=1}^T)$. Here, $\text{dist}(\cdot)$ assesses the (distributional) distance of task $\Tc$ to the source tasks $(\Dc_t)_{t=1}^T$ (e.g.~\cite{ben2010theory,hanneke2019value}). In case of linear tasks, we can simply use the Euclidean distance between task vectors, specifically, the distance of target weights $\bt_{\Tc}$ to the nearest source task $\text{dist}(\Tc)=\min_{t\in[T]}\tn{\bt_\Tc-\bt_t}$. In Fig.~\ref{fig:div} we investigate the distance of specific target tasks from source tasks and how the distance affects the transfer performance. Here, all source and target tasks have unit Euclidean norms so that closer distance is equivalent to larger cosine similarity. We again train each MTL task with multiple sequences $M\rightarrow \infty$ (as in Fig.~\ref{fig:transfer}) and use $T=20$ source tasks with {$d=5$ dimensional regression} problems. In a nutshell, Figure \ref{fig:div} shows that Euclidean task similarity is indeed highly predictive of transfer performance across different distance slices (namely $[0,0.2],[0.2,0.6],[0.6,1],[1,2]$). 

\section{Extension to Stable Dynamical Systems}\label{sec dynamic}
Until now, we have studied ICL with sequences of i.i.d. (input, label) pairs. In this section, we investigate the scenario where {prompts are obtained from the trajectories of stable dynamical systems, thus, they consist of dependent data.} Let $\Xc\subset\R^d$ and $\Fc : \Xc \rightarrow \Xc$ be a hypothesis class elements of which are dynamical systems. During MTL phase, suppose that we are given $T$ tasks associated with $(f_t)_{t=1}^T$ where $f_t\in\Fc$, and each contains $n$ in-context samples. Then, the data-{sequence} of $t$'th task is denoted by $\Sc_t = (\x_{ti})_{i=0}^{n}$ where $\x_{ti} = f_t(\x_{t,i-1}) + \w_{ti}$, $\x_{t0}$ is the initial state, and $\w_{ti}\in\Wc\subset\R^d$ are bounded i.i.d. random noise following some distribution. Then, prompts are given by $\xp{i}:= (\x_0,\x_1,\dots \x_{i})$. Let $\Scnt{i}{} = \xp{i}$, and we can make prediction $\hat\x_{ti}=\fal{i-1}{t}(\x_{t,{i-1}})$. We consider the similar optimization problem as \eqref{mtl opt}.

For generalization analysis, we require the system to be stable (which differs from algorithmic stability!). In this work, we use an exponential stability condition \cite{foster2020learning,sattar2022non} that controls the distance between two trajectories initialized from different points.
\begin{definition}[$(C_{\rho}, \rho)$-stability]\label{def stable}
    Denote the $m$'th state resulting from the initial state $\x_{t0}$ and $(\w_{t i})_{i = 1}^{m}$ by $f_t^{(m)}(\x_{t0})$. Let $C_{\rho} \geq 1$ and $\rho \in (0,1)$ be system related constants. We say that the dynamical system for the task $t$ is $(C_{\rho}, \rho)$-stable if, for all $\x_{t0}, \x'_{t0}\in\Xc$, $m \geq1$, and $(\w_{ti})_{i \geq 1}\in\Wc$, we have
    \begin{align}
        \wnorm{f_t^{(m)}(\x_{t0}) - f_t^{(m)}(\x'_{t0})}{\ell_2} \leq C_{\rho}\rho^m\wnorm{\x_{t0} - \x'_{t0}}{\ell_2}
    \end{align}   
\end{definition}

\begin{assumption}\label{assump stable dynamical}
    There exist $\bar{C}_{\rho}$ and $\bar{\rho}< 1$ such that all dynamical systems $f \in \Fc$ are $(\bar{C}_\rho, \bar{\rho})$-stable.
\end{assumption}

In addition to the stability of the hypothesis set $\Fc$, we also leverage the algorithmic-stability of the set $\Ac$ similar to Assumption \ref{assump robust}. Different from Assumption \ref{assump robust}, we restrict the variability of algorithms with respect to Euclidean distance metric,  similar to the definition of Lipschitz stability.

\begin{assumption}[Algorithmic-stability for Dynamics]\label{assump robust dynamical} Let $\Zb=(\x_0,\x_1,\dots,\x_{m+1})$ be a realizable dynamical system trajectory and $\Zb^j$ be the trajectory obtained by swapping $\w_j$ with $\w_{j}'$ ($j=0$ implies that $\x_0$ is swapped with $\x'_0$). As a result, starting with the $j$'th index, the sequence $\Zb^j$ has different samples $(\x'_{j}, \dots, \x'_{m+1})$. 
Let $X:=\ell(\x_{m+1}, f^\bal_{\Zb}(\x_{m}))$ and $X^j:=\ell(\x'_{m+1}, f^\bal_{\Zb^j}(\x'_{m}))$. There exists $K> 0$ such that for any $\Zb$, $\x_0'\in\Xc$, $\w'_j\in\Wc,j\in[m]$, we have 
\[|\E_{\w_{m+1}}[X -X^j ]| \\
         \leq  \frac{K}{m\red{-j+1}} \sum_{i=j}^{m}\wnorm{\x_i - \x'_i}{\ell_2}.
         \]
\end{assumption}

Lemma \ref{TF dynamic stable} fully justifies this assumption for multilayer transformers. To proceed, we state the main result of this section.
\begin{theorem}\label{thm dynamic mtl}
    Suppose $\ell(\x,\hat{\x})=\ell(\x-\hat{\x}):\Xc\times\Xc\rightarrow[0,1]$ is $L$-Lipschitz and Assumptions~\ref{assump stable dynamical}\&\ref{assump robust dynamical} hold. Assume $\Xc$ and $\Wc$ are bounded by $\bar{x}, \bar{w}$, respectively. Then, with the same probability, {the identical bound as in Theorem~\ref{thm mtl} Eqn.~\eqref{main result 1} holds after updating $K$ to be} $\bar K = 2K \frac{\bar{C}_{\rho}}{1-\bar{\rho}} (\bar{w} + \bar{x}/\sqrt{n})$.
\end{theorem}
The proof of this result is similar in spirit to the proof of Theorem \ref{thm mtl} and is provided in Appendix \ref{sec dynamic mtl}. The main difference is that we use system's stability to control the impact of a perturbation on the future trajectory.

\section{Interpreting In-context Learning as a Model Selection Procedure}\label{sec approx}


In Section~\ref{sec mtl}, we study the generalization error of ICL, which can be eliminated by increasing sample size $n$ or number of {sequences $M$} per task. In this section, we will discuss how ICL can be interpreted as an implicit model selection procedure building on the formalism that transformer is a learning algorithm. Following Figure \ref{fig:main_exp} and prior works \cite{garg2022can,laskin2022context,hollmann2022tabpfn}, a plausible assumption is that, transformer can implement ERM algorithms up to a certain accuracy. Then, model selection can be formalized by the selection of the right hypothesis class so that running ERM on that hypothesis class can strike a good bias-variance tradeoff during ICL.




To proceed with our discussion, let us consider the following hypothesis which states that transformer can implement an algorithm competitive with ERM. 
\begin{hypothesis}\label{hypo erm} Let $\FB=(\Fc_h)_{h=1}^H$ be a family of $H$ hypothesis classes. Let $\Sc=(\x_i,\y_i)_{i=1}^{n}$ be a {data-sequence} with $n$ examples sampled i.i.d.~from $\Dc$ and let $\Scnt{m}{}=(\x_i,\y_i)_{i=1}^{m}$ be the first $m$ examples. Consider the risk\footnote{Since the loss $\ell$ is bounded by $1$, $0\leq \text{risk}(h,m)\leq 1$ for all $m$ including the scenario $m=0$ and ERM is vacuous.} associated to ERM with $m$ samples over $\Fc_h\in\FB$: \ifswitch\vspace{-7pt}
\begin{align}
\riskhm=&\E_{(\x,\y,\Scnt{m}{})}[\ell(\y, \hat{f}^{(h)}_{\Scnt{m}{}}(\x))]\nn\\
\text{where}\quad &\hat f^{(h)}_{\Scnt{m}{}}=\arg\min_{f\in \Fc_h}\frac{1}{m}\sum_{i=1}^m\ell(\y_i,f(\x_i)),\nn
\end{align}
\else
\begin{align}
\riskhm=\E_{(\x,\y,\Scnt{m}{})}[\ell(\y, \hat{f}^{(h)}_{\Scnt{m}{}}(\x))]\quad
\text{where}\quad \hat f^{(h)}_{\Scnt{m}{}}=\arg\min_{f\in \Fc_h}\frac{1}{m}\sum_{i=1}^m\ell(\y_i,f(\x_i)),\nn
\end{align}
\fi
Let $(\ept^{h,m})>0$ be approximation errors associated with $(\Fc_h)_{h=1}^H$. There exists $\bal\in\Ac$ such that, for any $m\in[n],h\in[H]$, $f^\bal_{\Scnt{m}{}}$ can approximate ERM in terms of population risk, i.e.
\[
\E_{(\x,\y,\Scnt{m}{})}[\ell(\y, f^\bal_{\Scnt{m}{}}(\x))]\leq \riskhm+ \ept^{h,m}.
\]
\end{hypothesis}


For model selection purposes, these hypothesis classes can be entirely different ML models, for instance, $\Fc_1=\{\text{convolutional-nets}\}$, $\Fc_2=\{\text{fully-connected-nets}\}$, and $\Fc_3=\{\text{decision-trees}\}$. Alternatively, they can be a nested family useful for capacity control purposes. For instance, Figures~\ref{fig:main_exp}(a,b) are learning covariance/noise priors to implement a constrained-ridge regression. Here $\FB$ can be indexed by positive-definite matrices $\bSi$ with linear classes of the form $\Fc_\bSi=\{f(\x)=\x^\top \bt\quad\text{where}\quad\bt^\top\bSi^{-1}\bt\leq 1\}$.


Under Hypothesis \ref{hypo erm}, ICL selects the most suitable class that minimizes the excess risk for each $m\in [n]$. 

\cmt{Following \cite{garg2022can,laskin2022context,hollmann2022tabpfn}, some instructive examples for family $\FB$ are
\begin{myitemize}
    \item $\FB_\lambda=\{\Fc_\lambda: \text{Linear model with parameter bounded by $\tn{\bt}\leq\lambda$}\}$
    \item $\FB_\bSi=\{\Fc_\bSi:\text{Linear model with covariance $\bSi$, $\E[\bt\bt^\top]=\bSi$}\}$.
    \item $\FB^{\text{sparse}}_s=\{\Fc_s: s\text{-sparse linear model}\}$
    \item $\FB^{\text{NN}}_s=\{\Fc_s: 2\text{-layer neural net with width $s$}\}$
    \item $\FB^{\text{RF}}_s=\{\Fc_s: \text{Random forest with $s$ trees}\}$
\end{myitemize}
Here to simplify, we assumed a discrete family of hypothesis classes, $|\FB|=H<\infty$, however, our theory in Section~\ref{sec mtl} is developed for the continuous setting.}


\begin{observation}\label{thm approx} Suppose Hypothesis \ref{hypo erm} holds for a target distribution $\Dc_\Tc$. Let $\Lc_\Tc^\st:=\min_{\bal\in\Ac}\Lc_\Tc(\bal)$ be the risk of the optimal algorithm. We have that 
	\[
	\Lc_\Tc^\st\leq \frac{1}{n}\sum_{m=0}^{n-1}\min_{h\in [H]} \{\riskhm+ \ept^{h,m}\}.
    \]
    Additionally, denote Rademacher complexity of a class $\Fc$ by $\Rc_m(\Fc)$. Define the minimum achievable risk over function set $\Fc_h$ as {$\Lc^\st_h:=\min_{f\in \Fc_h}\E_{\Dc_\Tc}[\ell(\y,f(\x))]$}. Since $\riskhm$ is controlled by $\Rc_m(\Fc_h)$ \cite{mohri2018foundations}, we have that
    \[
    {\Lc_\Tc^\st\leq\frac{1}{n}\sum_{m=0}^{n-1}\min_{h\in [H]}\{\Lc^\st_h+ \ept^{h,m}+\order{\Rc_m(\Fc_h)}\}.}
    \]
\end{observation}
Here, ICL adaptively selects the classes $\arg\min_{h\in [H]}\{\Lc^\st_h+\Rc_m(\Fc_h)+ \ept^{h,m}\}$ to achieve small risk. This is in contrast to training over a single large class $\Fc=\bigcup_{i=1}^H\Fc_i$, which would result in a less favorable bound ${\approx\min_{h\in [H]}\Lc^\st_h+\frac{1}{n}\sum_{m=0}^{n-1}\Rc_m(\Fc)}$. A formal version of this statement is provided in Appendix \ref{app approx}. Hypothesis \ref{hypo erm} assumes a discrete family for simpler exposition ($|\FB|=H<\infty$), however, our theory in Section~\ref{sec mtl} allows for the continuous setting.

We emphasize that, in practice, we need to adapt the hypothesis classes for different sample sizes $m$ (typically, more complex classes for larger $m$). With this in mind, while we have $H$ classes in $\FB$, in total we have $H^n$ different ERM algorithms to compete against. This means that VC-dimension of the algorithm class is as large as $n \log H$. This highlights an insightful benefit of our main result: Theorem \ref{thm mtl} would result in an excess risk $\propto\sqrt{\frac{n\log H}{nT}}=\sqrt{\frac{\log H}{T}}$. In other words, the additional $\times n$ factor achieved through Theorem \ref{thm mtl} facilitates the adaptive selection of hypothesis classes for each sample size and avoids requiring unreasonably large $T$.

\section{Numerical Evaluations}\label{sec exp}

Our experimental setup follows \cite{garg2022can}: 
{All ICL experiments are trained and evaluated using the same GPT-2 architecture with 12 layers, 8 attention heads, and 256 dimensional embeddings. We first explain the details of Fig.~\ref{fig:main_exp} and then provide stability experiments.}\footnote{Our code is available at \url{https://github.com/yingcong-li/transformers-as-algorithms}.}
\ifswitch\else\smallskip\fi

\noindent$\bullet$ \textbf{Linear regression (Figures~\ref{fig:main_exp}(a,b)).} We consider a $d$-dimensional linear regression tasks with in-context examples of the form $\z=(\x,y)\in\R^d\times\R$. Given $t$'th task $\bt_t$, we generate $n$ i.i.d. samples via $y=\bt_t^\top\x+\xi$, where $\x\sim\Nc(0,\Iden),~\xi\sim\Nc(0,\sigma^2)$ and $\sigma$ is the noise level. Tasks are sampled i.i.d.~via~$\bt_t\sim\Nc(0,\bSi),~t\in[T]$. Results are displayed in Figures~\ref{fig:main_exp}(a)\&(b). We set $d=20$, $n=40$ and significantly larger $T$ to make sure model is sufficiently trained and we display meta learning results (i.e.~on unseen tasks) for both experiments. In Fig.~\ref{fig:main_exp}(a), $\sigma=1$ and $\bSi=\Iden$. We also solve ridge-regularized linear regression (with sample size from $1$ to $n$) over the grid $\lambda=[0.01,0.05,0.1,0.5,1]$ and display the results of the best $\lambda$ selection as the optimal ridge curve (Black dotted). Recall from \eqref{wridge} that ridge regression is optimal for isotropic task covariance. In Fig.~\ref{fig:main_exp}(b), we set $\sigma=0$ and $\bSi=\diag{\left[1,\frac{1}{2^2},\frac{1}{3^2},\dots,\frac{1}{20^2}\right]}$. Besides ordinary least squares (Green curve), we also display the optimally-weighted regression according to \eqref{wridge} (dotted curve) as $\sigma\rightarrow0$. In both figures, ICL (Red) outperforms the least squares solutions (Green) and are perfectly aligned with optimal ridge/weighted solutions (Black dotted). This in turn provides evidence for the automated model selection ability of transformers by learning task priors.\ifswitch\else\smallskip\fi





\ifswitch
\else
\begin{wrapfigure}{r}{5.5cm}
\vspace{-0.45cm}
\centering
		\begin{tikzpicture}
			\centering
			\node at (0,0) {\includegraphics[scale=0.26]{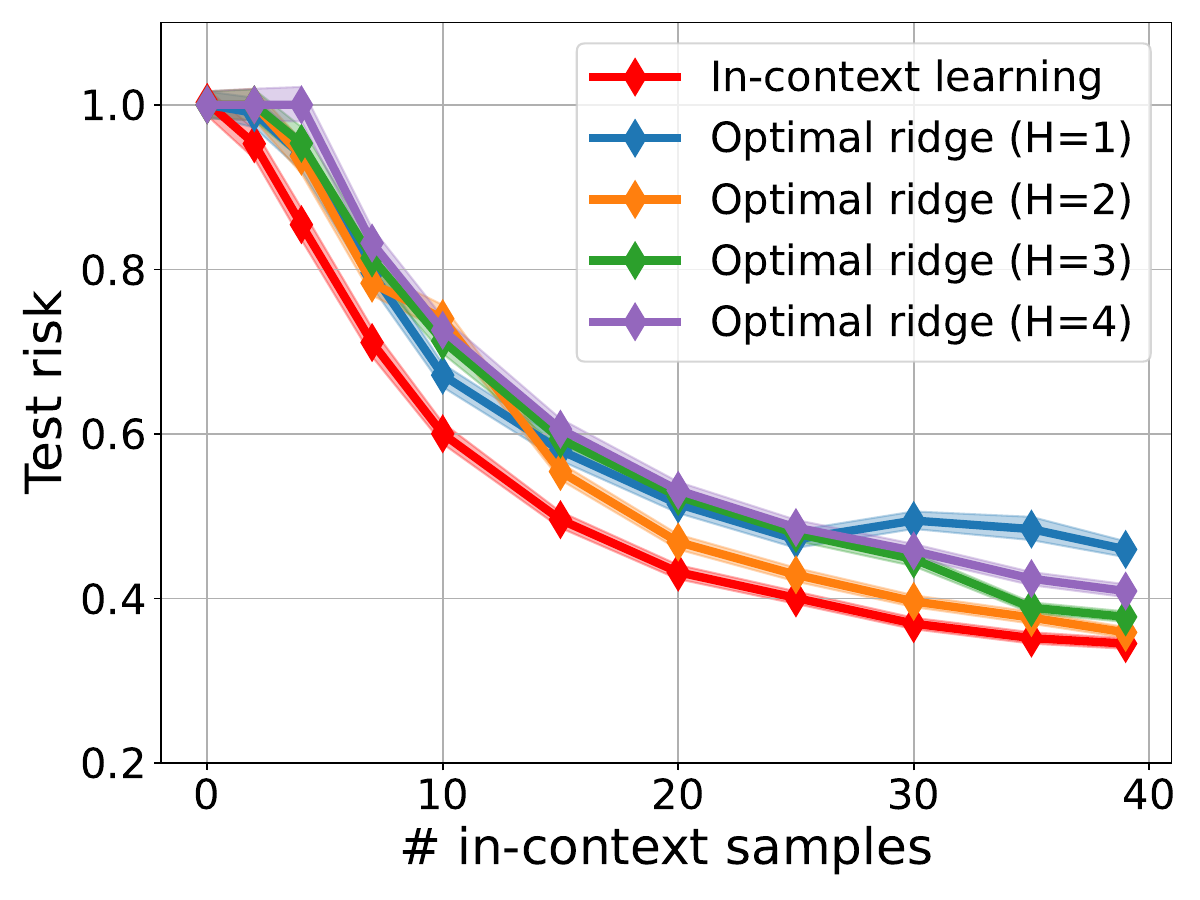}};
		\end{tikzpicture}\vspace{-10pt}
\caption{\small{Dynamical system experiments. The difference from Fig.~\ref{fig:main_exp}(c) is that we compare ICL to the optimally-tuned ridge regression with different history windows $H$.}}\vspace{-10pt}
\label{fig:dyn}
\end{wrapfigure}
\fi
\noindent$\bullet$ \textbf{Partially-observed dynamical systems (Figures \ref{fig:main_exp}(c) \& \ref{fig:dyn}).} We generate in-context examples $\z_i=\x_i\in\R^r,~i\in[n]$ via the partially-observed linear dynamics $\x_i=\Cb\s_i$, $\s_i=\A\s_{i-1}+\bxi_{i}$ with noise $\bxi_i\sim\Nc(0,\sigma^2\Iden_d)$ and initial state $\s_0=\boldsymbol{0}$. Each task is parameterized by $\Cb\in\R^{r\times d}$ and $\A\in\R^{d\times d}$ which are drawn with i.i.d.~$\Nc(0,1)$ entries {and $\A$ is normalized to have spectral radius $0.9$}. In Fig.~\ref{fig:main_exp}(c), we set $d=10$, $r=4$, $\sigma=0$, $n=20$ and use sufficiently large $T$ to train the transformer. For comparison, we solve least-squares regression to predict new observations $\x_i$ via the most recent $H$ observations for varying window sizes $H$. Results show that in-context learning outperforms the least-squares results of all orders $H=1,2,3,4$. In Figure \ref{fig:dyn}, we also solve the dynamical problem using optimal ridge regression for different window sizes. This reveals that ICL can also outperform auto-regressive models with optimal ridge tuning, albeit the performance gap is much narrower. It would be interesting to compare ICL performance to a broader class of system identification algorithms (e.g.~Hankel nuclear norm, kernel-based, atomic norm \cite{ljung1998system,pillonetto2016regularized}) and understand the extent ICL can inform practical algorithm design.\ifswitch\else\smallskip\fi

\ifswitch
\begin{figure}
	\centering
	\includegraphics[scale=0.2]{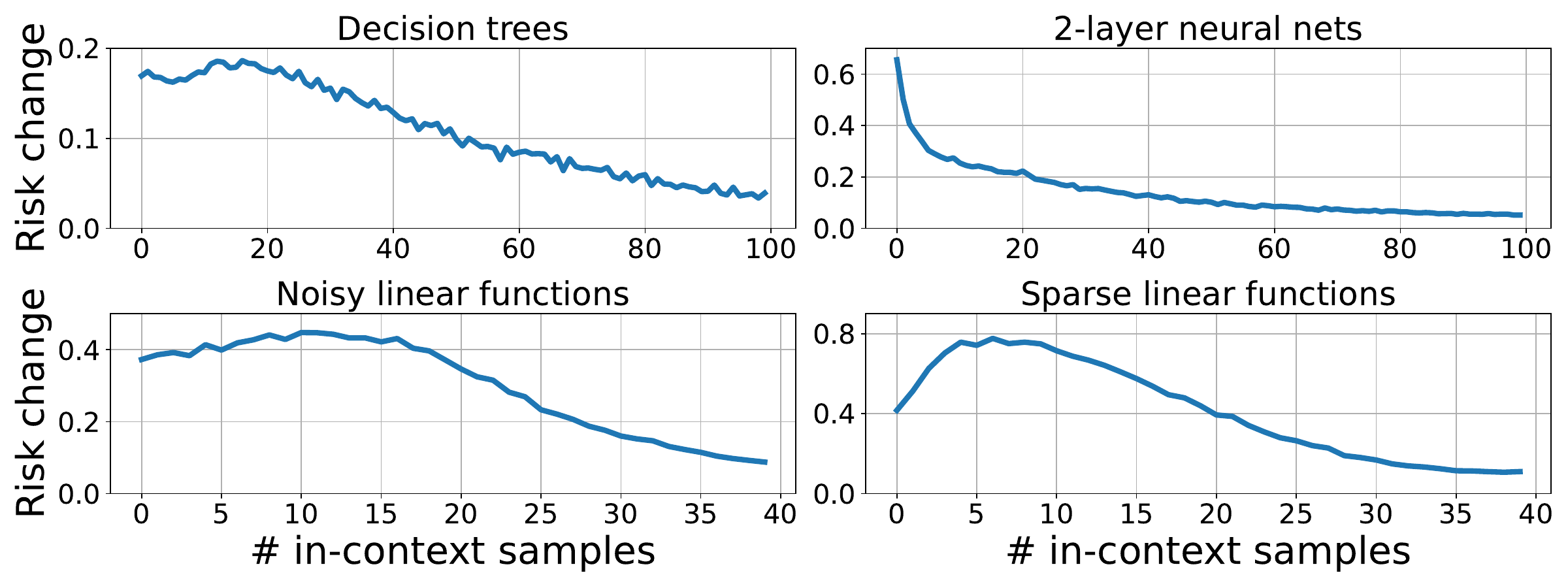}
 \vspace{-7pt}
	\caption{\small{Experiments to assess the algorithmic stability Assumption~\ref{assump robust}. Each figure shows the increase in the risk for varying ICL sample sizes after an example in the prompt is modified. We swap an input example in the prompt and assign a flipped label to this new input, e.g., we move from $(\x,f(\x))$ to $(\x',-f(\x'))$. 
 }}\label{fig:robust}
\vspace{-16pt}
\end{figure}
\else
\begin{figure}
	\centering
	\includegraphics[scale=0.35]{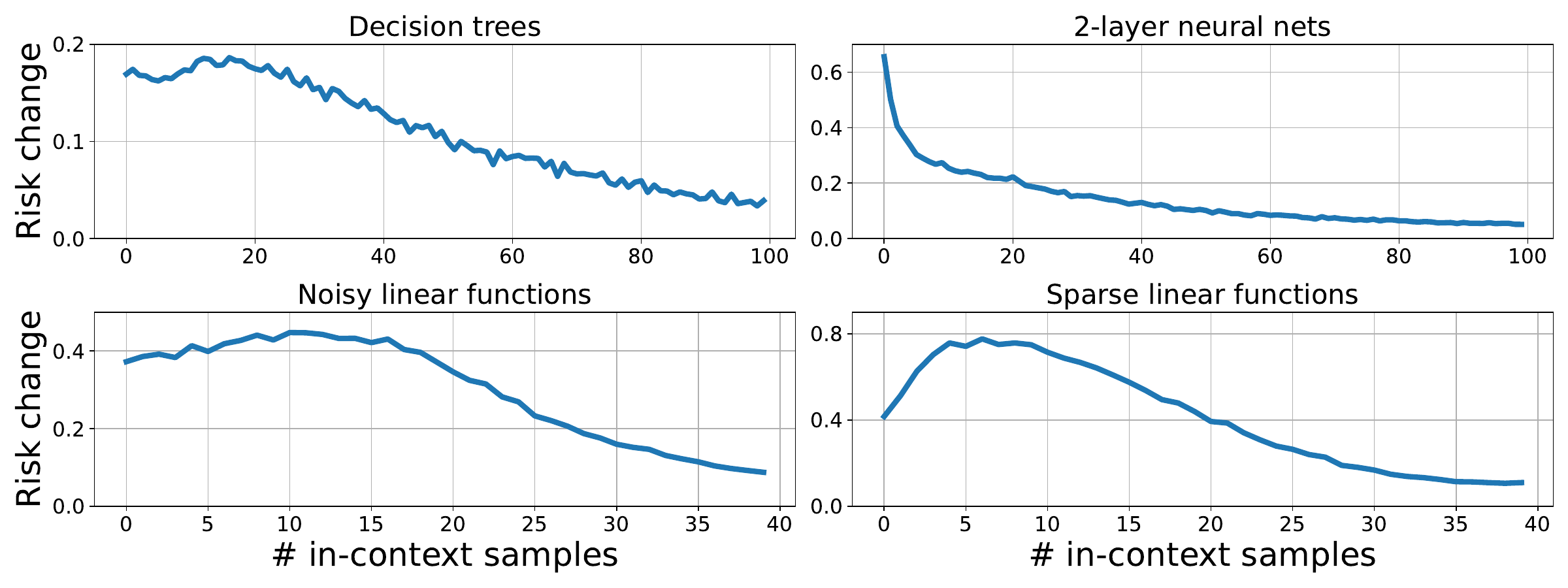}
	\caption{\small{Experiments to assess the algorithmic stability Assumption~\ref{assump robust}. Each figure shows the increase in the risk for varying ICL sample sizes after an example in the prompt is modified. We swap an input example in the prompt and assign a flipped label to this new input, e.g., we move from $(\x,f(\x))$ to $(\x',-f(\x'))$. 
 }}\label{fig:robust}
\end{figure}
\fi
\noindent$\bullet$ \textbf{Stability analysis (Figure~\ref{fig:robust}).} In Assumption~\ref{assump robust}, we require that transformer-induced algorithms are stable to input perturbations, specifically, we require predictions to vary by at most ${\cal{O}}(1/m)$ where {$m$ is the sample size}. This was justified in part by Theorem \ref{main stable}. To understand empirical stability, we run additional experiments where the results are displayed in Fig.~\ref{fig:robust}. We study stability of four function classes: linear models, $3$-sparse linear models, decision trees with depth $4$, and 2-layer ReLU networks with 100 hidden units, all with input dimension of $20$. For each class $\Fc$, a GPT-2 architecture is trained with large number of random tasks $f\in\Fc$ and evaluate on new tasks. With the exception of Fig.~2(a), we use the pretrained models provided by \cite{garg2022can} and the task sequences are noiseless i.e.~sequences obey $y_i=f(\x_i)$. As a coarse approximation of the \emph{worst-case} perturbation, we perturb a prompt $\xp{m}=(\x_1,y_1,\cdots,\x_{m-1},y_{m-1},\x_{m})$ as follows. Draw a random point $(\x'_1,y'_1)\sim (\x_1,y_1)$ and flip its label to obtain $(\x'_1,-y'_1)$. {We obtain the adversarial prompt via $\bar\x_{\text{prompt}}^{(m)}=(\x'_1,-y'_1,\cdots,\x_{m-1},y_{m-1},\x_{m})$}\footnote{To fully verify Assumption~\ref{assump robust} one should adversarially optimize $\x'_1,\y'_1$ and also swap the other indices $m>i>1$.}. In Fig.~\ref{fig:robust}, we plot the test risk change between the adversarial and standard prompts. All figures corroborate that, after a certain sample size, the risk change noticeably decreases as the in-context sample size increases. This behavior is in line with Assumption~\ref{assump robust}; however, further investigation and longer context window are required to accurately characterize the stability profile (e.g.~to verify whether stability is ${\cal{O}}(1/m)$ or not). {Finally, in Figure \ref{fig:robust linear noise} of the  appendix, we show that adding label noise to regression tasks during MTL training  can help improve stability.}

\section{Conclusions}\label{sec discuss}
In this work, we approached in-context learning as an algorithm learning problem with a statistical perspective. We presented generalization bounds for MTL where the model is trained with $T$ tasks each mapped to a sequence containing $n$ examples. Our results build on connections to algorithmic stability which we have verified for transformer architectures empirically as well as theoretically. Our generalization and stability guarantees are also developed for dynamical systems capturing autoregressive nature of transformers.

There are multiple interesting questions related to our findings: (1) Our bounds are mainly useful for capturing multitask learning risk, which motivates us to study the question: How can we control generalization on individual tasks or prompts with specific lengths (rather than average risk)? (2) We provided guarantees for dynamical systems with full-state observations. Can we extend such results to more general dynamic settings such as reinforcement/imitation learning or system identification with partial state observations? (3) Our investigation of transfer learning in Section \ref{sec transfer} revealed that transfer risk is governed by the number of MTL tasks and task complexity however it seems to be independent of the model complexity. It would be interesting to further demystify this inductive bias during pretraining and characterize exactly what algorithm is learned by the transformer.

\section*{Acknowledgements}

This work was supported in part by the NSF grants CCF-2046816 and CCF-2212426, Google Research Scholar award, and Army Research Office grant W911NF2110312.

\bibliography{refs,biblio}
\ifswitch
\bibliographystyle{icml2023}
\else
\bibliographystyle{plain}
\fi

\newpage
\appendix
\onecolumn
\subsection*{Organization of the Appendix}
\begin{myitemize}
\item Supporting experiments and details are provided under Section~\ref{app add exp}.
    \item In Section~\ref{app stable}, we prove and discuss our stability results.
    \item Section~\ref{app main proof} provides proofs of MTL (Section~\ref{sec mtl}) and transfer learning (Section~\ref{sec transfer}) generalization bounds.
    \item Section~\ref{sec dynamic mtl} proves our dynamical generalization bound (Theorem~\ref{thm dynamic mtl}).
    \item In Section~\ref{app approx}, we discuss the model selection aspect of ICL.
    \item We introduce more related work in Section~\ref{app related}.
\end{myitemize}

\begin{figure}[t]
\centering
\begin{minipage}[t]{0.31\textwidth}
\centering
\begin{tikzpicture}
			\node at (0,0) {\includegraphics[scale=0.28]{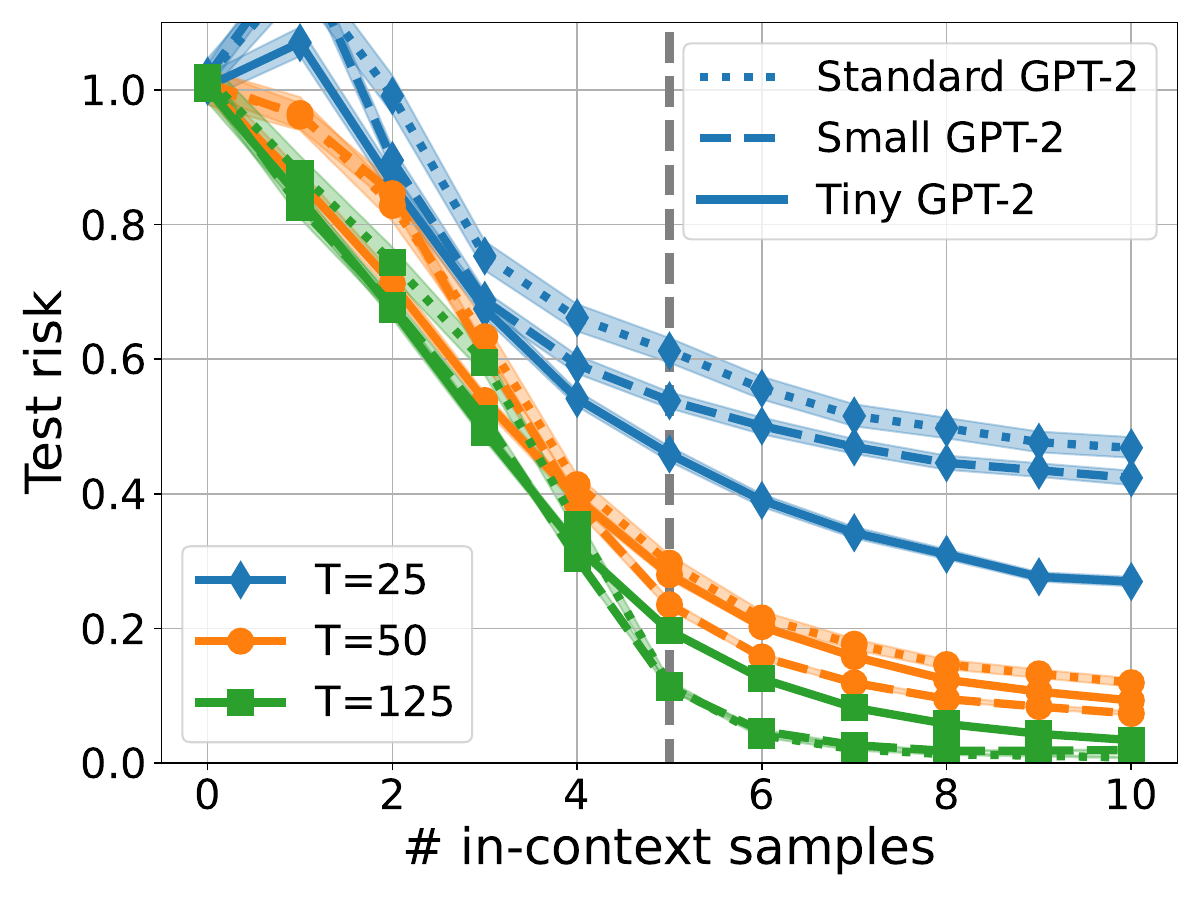}};
		\end{tikzpicture}
		\vspace{-20pt}
\caption{\red{Following Figure~\ref{fig:transfer}, instead we train $d=5$ dimensional linear regression problem with three different GPT-2 architectures and overlay the transfer results.}}\vspace{-10pt}
\label{fig:d5 overlay}
\end{minipage}
\hspace{10pt}
\centering
\begin{minipage}[t]{0.31\textwidth}
\centering
\begin{tikzpicture}
			\node at (0,0) {\includegraphics[scale=0.28]{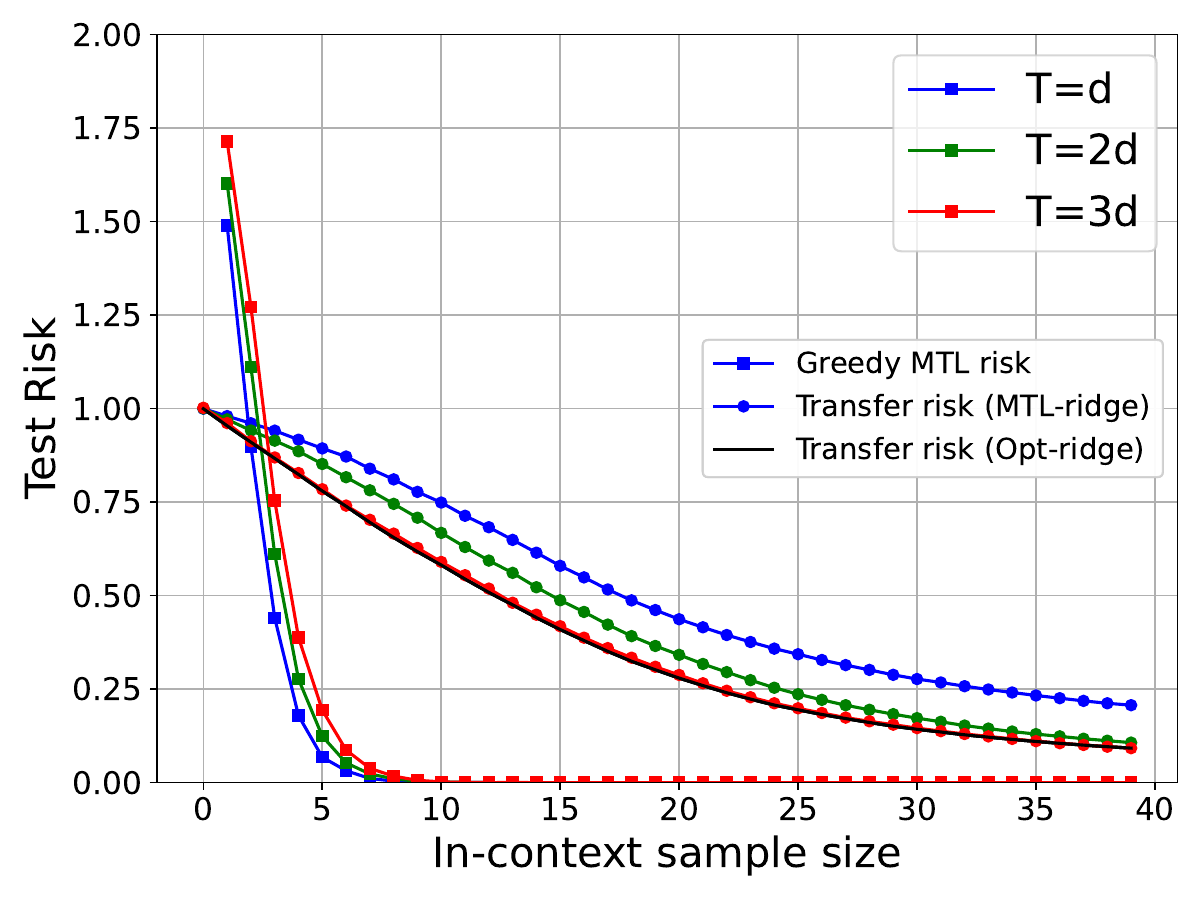}};
		\end{tikzpicture}
		\vspace{-20pt}
\caption{We display the performance of the \emph{idealized} transfer and MTL algorithms described in Section \ref{sec transfer}. Unlike ICL experiments, these require $T\lesssim d$ tasks.}
\label{fig:hindsight}
\end{minipage}
\hspace{10pt}
\centering
\begin{minipage}[t]{0.31\textwidth}
\centering
\begin{tikzpicture}
			\node at (0,0) {\includegraphics[scale=0.28]{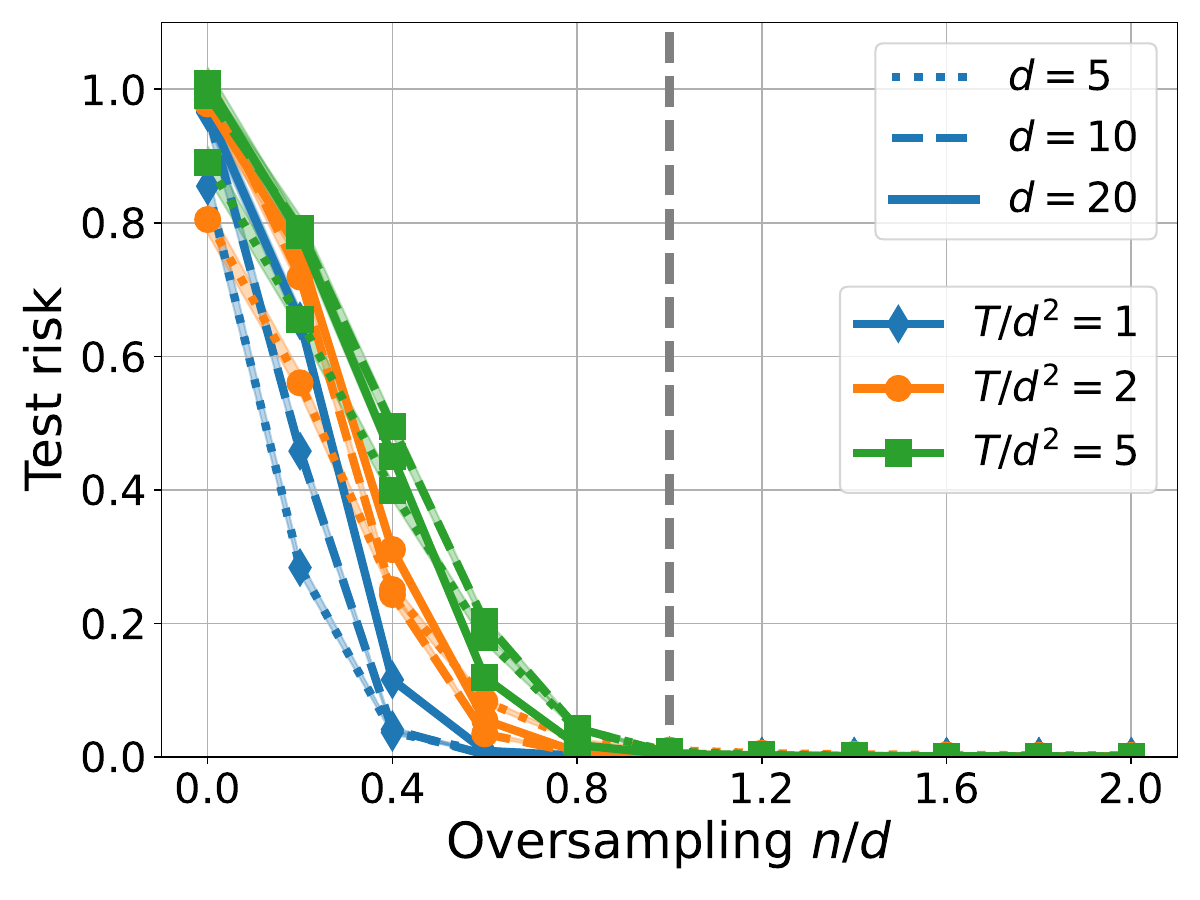}};
		\end{tikzpicture}
		\vspace{-20pt}
\caption{The difference form Fig.~\ref{fig:transfer}(d) is that we overlay the MTL results of  dimensions $d\in\{5,10,20\}$ (dashed curves in Fig.~\ref{fig:transfer} (a,b,c)).}\vspace{-10pt}
\label{fig:overlay mtl}
\end{minipage}

\end{figure}

\section{Additional Experiments}\label{app add exp}
\red{Our linear regression experiments are based on the code released by \cite{garg2022can}, however without curriculum learning. All the inputs and noise are i.i.d.~Gaussian vectors and tasks are i.i.d.~sampled from some distribution. The meta learning results Fig.~\ref{fig:main_exp}(a,b) are trained with $T=32$ million random linear tasks and Fig.~\ref{fig:main_exp}(c) and Fig.~\ref{fig:dyn} are trained with $T=6.4$ million dynamical trajectories (here, we fix the batch size to $64$ and train with $500$k/$100$k iterations). All experiments use learning rate $0.0001$ and Adam optimizer.
}

\subsection{Supporting Experiments for Section \ref{sec transfer}}\label{additional exp}

\noindent\textbf{Architecture dependence of transfer risk:} In Figure~\ref{fig:d5 overlay}, we verify that the transfer risk is (mostly) independent of the model complexity $\dim(\Ac)$ (in contrast to the dependence on task complexity $d$). Following the same setting as in Figure~\ref{fig:transfer}, during the MTL phase, we consider $5$-dimensional linear regression problem and train with $T=25/50/125$ tasks over three different models: tiny/small/standard GPT-2. The standard GPT-2 has the same architecture as used in Fig.~\ref{fig:transfer} and Section~\ref{sec exp}, with $12$ layers, $8$ attention heads and $256$ dimensional embeddings. While, small GPT-2 has $6$ layers, $4$ attention heads and $128$ dimensional embeddings, and tiny GPT-2 has only $3$ layers, $2$ attention heads and $64$ dimensional embeddings. They contain $9.5$M, $1.2$M and $0.15$M parameters respectively, which shows that small GPT-2 has around $8\times$ less parameters than standard GPT-2 and tiny GPT-2 has around $64\times$ less. Overlayed results are displayed in Figure~\ref{fig:d5 overlay}, which demonstrate that although the architectures are substantially different in terms of complexity and expressive power, the performances under the same data setting (same color with different line styles) are approximately aligned.

\smallskip
\noindent\textbf{Contrasting ICL to Idealized Algorithms.} In Section \ref{sec transfer}, we discussed how transfer risk seems to require $T\propto d^2$ source tasks. In contrast, constructing the empirical covariance $\hat{\bSi}=\frac{1}{T}\sum_{i=1}^T \bt_i\bt_i^\top$ can make sure that $\hat{\bSi}$-weighted LS performs ${\cal{O}}(1)$-close to $\bSi$-weighted LS whenever ${\|\bSi-\hat{\bSi}\|}/{\lambda_{\min}(\bSi)}\leq {\cal{O}}(1)$. In Figure \ref{fig:hindsight}, \emph{MTL-ridge} curves with circle markers are referring to the $\hat{\bSi}$-weighted ridge regression. As suspected, $T= 3d$ is already sufficient to get very close performance to the optimal weighting with true $\bSi$ (black curve). We remark that in Figure \ref{fig:hindsight}, we set $d=20$, noise variance obeys $\sigma^2=0.1$, and linear task vectors $\bt$ are uniformly sampled over the sphere.

For MTL, Section \ref{sec transfer} introduces the following simple greedy algorithm to predict a prompt that belong to one of the $T$ source tasks (aka MTL risk): Evaluate each source task parameter $(\bt_t)_{t=1}^T$ on the prompt and select the parameter with the minimum risk. Since there are $T$ choices, this greedy algorithm will determine the optimal task in $n\lesssim \log(T)$ samples\footnote{Note that, this dependence can be even better if the problem is noiseless, in fact, that is why we added label noise in these experiments.}. The experiments of this algorithm is provided under \emph{Greedy MTL} legend (square markers). It can be seen that as $T$ varies ($d,2d,3d$), there is almost no difference in the MTL risk, likely due to the $\log(T)$ dependence. \red{Figure~\ref{fig:overlay mtl} gathers the MTL risk curves from Fig.~\ref{fig:transfer} (a,b,c) and overlays them together. Same as transfer risks shown in Fig.~\ref{fig:transfer}(d), the test risks stay approximately unchanged for fixed point $\alpha=n/d$ and $\beta=T/d^2$. It is also aligned with Fig.~\ref{fig:hindsight}, greedy MTL risk curves, where larger $T$ requires more samples (although their $d$-dependence is very different).} In short, these experiments highlight the contrast between ideal/greedy algorithms and ICL algorithm implemented within the transformer model. 

\begin{figure}[t]
\centering
\begin{minipage}[t]{\ifswitch0.31\else0.45\fi\textwidth}
\centering
\includegraphics[scale=0.28]{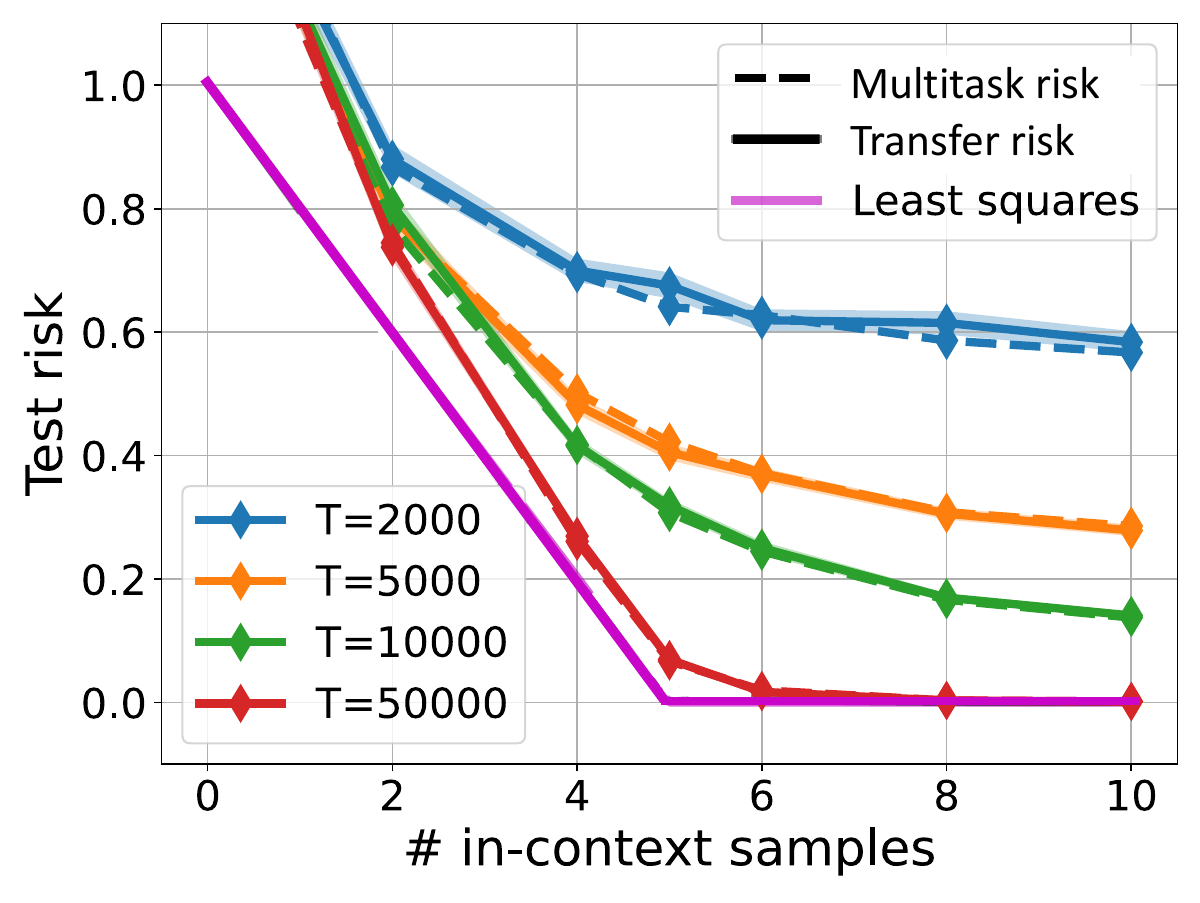}
\caption{Comparing MTL and transfer risks when each task has single trajectory ($M=1$).}
\label{fig:single}
\end{minipage}
\ifswitch
\hspace{10pt}
\begin{minipage}[t]{0.31\textwidth}
\centering
\includegraphics[scale=0.28]{figs/dynamic_partial_ridge.pdf}
\caption{Dynamical system experiments. The difference from Fig.~\ref{fig:main_exp}(c) is that we compare ICL to the optimally-tuned ridge regression with different history windows $H$.}\vspace{-10pt}
\label{fig:dyn}
\end{minipage}
\else
\fi
\hspace{10pt}
\begin{minipage}[t]{\ifswitch0.31\else0.45\textwidth}
\centering
\includegraphics[scale=0.28]
{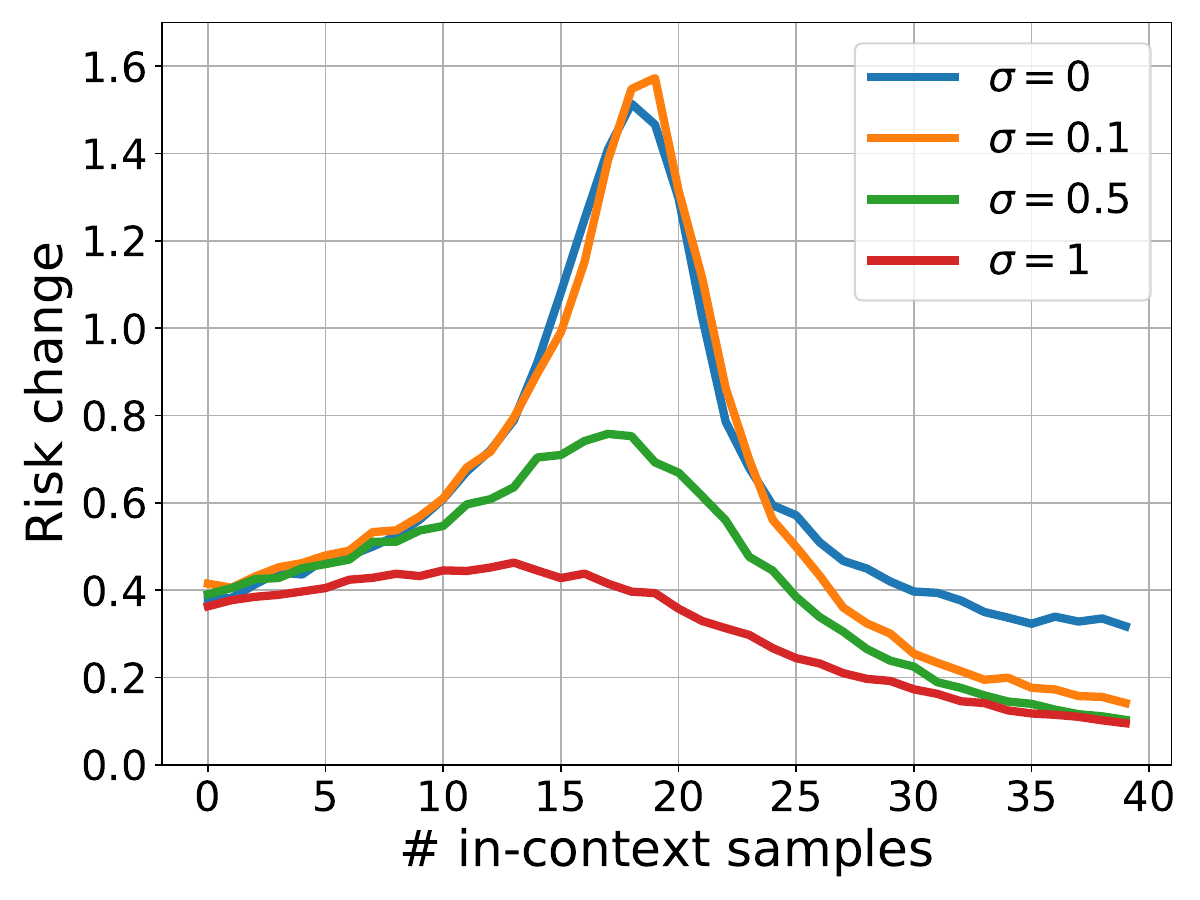}
\caption{\red{Stability experiments on noisy/noiseless linear settings, where $\sigma$ is the label noise level and data is generated by $y\sim\Nc(\x^\top\bt,\sigma^2)$ where $\x,\beta\sim\Nc(0,\Iden)$. Blue curve is noiseless regression. }}\label{fig:robust linear noise}\vspace{-10pt}
\end{minipage}
\end{figure}

In Section \ref{sec transfer}, we also exclusively focused on the setting $M\rightarrow \infty$ i.e. MTL tasks are thoroughly trained. In Figure \ref{fig:single} we consider the other extreme where each task is trained with a single trajectory $M=1$, which is closer to the spirit of Theorem \ref{thm mtl}. We set $d=5$, $n=10$ and $M=1$, and vary the number of linear regression tasks $T$ from $2000$ to $50000$. Not surprisingly, the results show that increasing $T$ helps in reducing the MTL risk. The more interesting observation is that transfer risk and MTL risk are almost perfectly aligned. We believe that this is due to the small $M,n$ choices which would make it difficult to overfit to the MTL tasks. Thus, when $M=1$, the gap between transfer and MTL risk seems to vanish and Theorem \ref{thm mtl} becomes directly informative for the transfer risk. In contrast, as $M$ grows, training process can overfit to the MTL tasks which leads to the split between MTL and transfer risks as in Figure \ref{fig:transfer}.


\subsection{Additional Stability Experiments}
In Section~\ref{sec exp} and Figure~\ref{fig:robust}, we run adversarial experiments demonstrating that our stability assumption (Assumption~\ref{assump robust}) is indeed realistic. In addition, we find that adding noise to the labels can help improve stability. As depicted in Figure~\ref{fig:robust linear noise}, Red curve is much more stable compared to the Blue curve which is trained with noiseless linear regression tasks. One interpretation is that solving noiseless problems might result in an overfitted algorithm (towards noiseless tasks) and a small perturbation/distribution-shift leads to significant error. \red{The peaks in Figure~\ref{fig:robust linear noise} occurs around $n=d$ and (most likely) arise from the double-descent phenomena: When there is no label noise, an interpolating linear model (without ridge regularization) is optimal (recall \eqref{wridge}). However such an interpolating model is susceptible to adversarial perturbations especially when the condition number is poor (which occurs at $n=d$). Here, the key takeaway is that noise has a stabilitizing effect, because under label-noise, optimal model learned by TF is the solution of a weighted ridge regression thus regularizes the transformer's algorithm.}
\section{Stability of Transformer-based ICL}\label{app stable}
\begin{lemma} \label{lem sft}Let $\x,\beps\in\R^n$ be vectors obeying $\tin{\x},\tin{\x+\beps}\leq c$. Then, there exists a constant $C=C(c)$, such that 
\[
\tin{\sft{\x}}\leq e^{2c}/n\quad\text{and}\quad\tone{\sft{\x}-\sft{\x+\beps}}\leq e^{2c}\tone{\beps}/n.
\]
\end{lemma}
\begin{proof} Without losing generality, assume the first coordinate is the largest. Using monotonicity of softmax, we obtain $\tin{\sft{\x}}\leq \frac{e^c}{e^c+\sum_{i=2}^ne^{-c}}\leq \frac{e^{2c}}{n}$. For vectors $\beps$ and $\x$, infitesimal softmax perturbation is bounded via
\[
\lim_{\delta\rightarrow0}[\sft{\x+\delta\beps}-\sft{\x}]/\delta= [\text{diag}(\sft{\x})-\sft{\x}\sft{\x}^\top]\beps.
\]
We use $\tone{[\text{diag}(\sft{\x})-\sft{\x}\sft{\x}^\top]\beps}\leq e^{2c}\tone{\beps}/n$. Integrating the derivative along $\delta=0$ to $1$, we obtain the result.
\end{proof}

For a matrix $\A$, let $\|\A\|_{2,p}$ denote the $\ell_p$ norm of the vector obtained by the $\ell_2$ norms of its rows.
\begin{lemma} \label{lem stable} Let $\X=[\x_1~\dots~\x_n]^\top$ and $\Eb=[\beps_1~\dots~\beps_n]^\top$ be the input and perturbation matrices respectively. Assume that the tokens ($\x_i,\x_i+\beps_i$) lie in unit ball i.e.~$\|\X\|_{2,\infty},\|\X+\Eb\|_{2,\infty}\leq 1$. Let $\Vb,\W\in\R^{d\times d}$ be the weights of the self-attention layer obeying $\|\Vb\|\leq 1$ and $\|\W\|\leq \Gamma$. Define the attention outputs $\A=\sft{\X \W\X^\top}\X \Vb$ and $\Ab=\sft{\Xb \W\Xb^\top}\Xb \Vb$. Define $\bar{\Eb}=\Ab-\A:=[\bbeps_1~\dots~\bbeps_n]^\top$. Let $\Ceb$ be an upper bound on $\|\Eb\|_{2,1}$. We have that
\[
\|\A\|_{2,\infty},\|\Ab\|_{2,\infty}\leq 1,\quad \|\bar{\Eb}\|_{2,1}\leq (2\Gamma+1)e^{2\Gamma}\Ceb.
\]
Additionally, for any $i\in[n]$ such that $\tn{\beps_i}\leq \Ceb/n$, we have $\tn{\bbeps_i}\leq \frac{1}{n}(2\Gamma+1)e^{2\Gamma}\Ceb$.
\end{lemma}
\begin{proof} First observe that $\Vb$ preserves norms i.e.~$\X\Vb$ obeys $\|\X\Vb\|_{2,\infty}\leq \|\X\|_{2,\infty}\leq 1$ and $\|\Eb\Vb\|_{2,1}\leq \|\Eb\|_{2,1}$.

Next, set $\Xb=\X+\Eb$ and define attention outputs $\A=\sft{\X \W\X^\top}\X\Vb$, $\Ab=\sft{\Xb \W\Xb^\top}\Xb\Vb$. Observe that, since softmax applies row-wise to the similarities (e.g.~$\X\W\X$), we preserve the feature norms i.e.~$\|\A\|_{2,\infty},\|\Ab\|_{2,\infty}\leq 1$ as advertised.

Now, consider the attention output difference $\Pb=\Ab-\A$
\begin{align}
\Pb=\underbrace{[\sft{\Xb \W\Xb^\top}-\sft{\X \W\X^\top}]\X\Vb}_{\Pb_1}+\underbrace{\sft{\Xb \W\Xb^\top}\Eb\Vb}_{\Pb_2}.\label{pb decomp}
\end{align}
For any pairs of tokens, we have $|\x_i^\top \W\x_j|\leq \Gamma$. Using Lemma \ref{lem sft}
\begin{align}
\|\Pb_2\|_{2,1}=\|\sft{\Xb \W\Xb^\top}\Eb\Vb\|_{2,1}\leq n\|\sft{\Xb \W\Xb^\top}\|_{\infty}\|\Eb\|_{2,1}\leq e^{2\Gamma}\|\Eb\|_{2,1}.\label{Pb2 term}
\end{align}
Secondly, set $\Pb_1=[\sft{\Xb \W\Xb^\top}-\sft{\X \W\X^\top}]\X\Vb$. We have that
\begin{align*}
\|\Pb_1\|_{2,1}&\leq \tone{\sft{\Xb \W\Xb^\top}-\sft{\X \W\X^\top}}\|\X\Vb\|_{2,\infty}\\
&\leq  \tone{\sft{\Xb \W\Xb^\top}-\sft{\X \W\X^\top}}.
\end{align*}
To proceed, define the $\delta$-scaled perturbation $\Eb'=\delta\Eb=\Xb'-\X$ for some $0\leq \delta\leq 1$. We will bound the derivative via $\delta\rightarrow 0$ and then integrate this derivative bound along $\Eb$ (i.e.~from $\delta=0$ to $\delta=1$). Clearly, as $\delta\rightarrow 0$, the quadratic-terms involving $\delta^2\Eb$ disappear and $\tone{\sft{\Xb' \W\Xb'^\top}-\sft{\X \W\X^\top}}$
\[
\leq \tone{\sft{\Xb' \W\X^\top}-\sft{\X \W\X^\top}}+\tone{\sft{\X \W\Xb'^\top}-\sft{\X \W\X^\top}}.
\]
To bound the latter, consider each row individually, namely pick a row from $\X,\X+\Eb'$ each denoted by the pair $(\x,\x+\beps')$. Note that for any cross-product, we are guaranteed to have $|(\x+\beps')^\top \W\x_i|,|\x^\top \W\x_i|\leq \Gamma,\tone{\beps'^\top \W\X}\leq \Gamma n\tn{\beps'}$, $\tone{\x^\top \W\Eb'^\top}\leq \Gamma \|\Eb'\|_{2,1}$. Applying perturbation bound of Lemma \ref{lem sft}, we get
\begin{align}
&\tone{\sft{(\x+\beps')^\top \W\X^\top}-\sft{\x^\top \W\X^\top}}\leq \Gamma e^{2\Gamma}\tn{\beps'}\label{main pert eq1}\\
&\tone{\sft{\x^\top \W(\X+\Eb')^\top}-\sft{\x^\top \W\X^\top}}\leq \Gamma e^{2\Gamma} \|\Eb'\|_{2,1}/n\label{main pert eq2}.
\end{align}
Adding up all $n$ rows, we obtain 
\[
\lim_{\delta\rightarrow 0}\tone{\sft{(\X+\delta\Eb) \W\Xb^\top}-\sft{\X \W\X^\top}}/\delta\leq 2\Gamma e^{2\Gamma}\|\Eb\|_{2,1}.
\]
Integrating the derivative along $\delta=0$ to $\delta=1$, we obtain $\|\Pb_1\|_{2,1}\leq 2\Gamma e^{2\Gamma}\|\Eb\|_{2,1}$. Together with \eqref{Pb2 term}, we obtain the main claim $\|\Pb\|_{2,1}\leq (2\Gamma+1)e^{2\Gamma}\|\Eb\|_{2,1}\leq (2\Gamma+1)e^{2\Gamma}\Ceb$. \red{To proceed, we control the individual output $i$ for which the input perturbation is small i.e.~$\tn{\beps_i}\leq \Ceb/n$. To this end, let us repeat the identical argument focusing on $i$th token. Suppose $i$'th token inputs are (dropping subscripts $i$) $\x,\xb,\beps=\xb-\x$ and outputs are $\ab,\bar{\ab},\bbeps=\bar{\ab}-\ab$. Similar to \eqref{pb decomp}, we write (after transposing)
\[
\bbeps=\underbrace{\Vb^\top\X^\top  [\sft{\Xb \W^\top \xb}-\sft{\X \W^\top \x}]}_{\pb_1}+\underbrace{\Vb^\top \Eb^\top\sft{\Xb \W^\top\xb}}_{\pb_2}.
\]}
Using $|\x_i^\top \W\x_j|\leq \Gamma$ for all $i,j$ and using Lemma \ref{lem sft}, similar to \eqref{Pb2 term}, we bound
\[
\tn{\pb_2}\leq \tn{ \Eb^\top\sft{\Xb \W^\top\xb}}\leq \frac{e^{2\Gamma}}{n}\|\Eb\|_{2,1}.
\]
To proceed, we will again study the $\pb_1$
\begin{align*}
\tn{\pb_1}&\leq \tn{\X^\top  [\sft{\Xb \W^\top \xb}-\sft{\X \W^\top \x}]}\\
&\leq \|\X\|_{2,\infty}\tone{\sft{\Xb \W^\top \xb}-\sft{\X \W^\top \x}}\\
&\leq \tone{\sft{\Xb \W^\top \xb}-\sft{\X \W^\top \x}}.
\end{align*}
Now, considering perturbation $\Eb'=\delta\Eb$, letting $\delta\rightarrow 0$, and from triangle inequality, we obtain
\begin{align*}
&\lim_{\delta\rightarrow 0}\delta^{-1}\tone{\sft{(\X+\delta\Eb) \W^\top (\x+\delta\beps)}-\sft{\X \W^\top \x}}\leq \\
&\lim_{\delta\rightarrow 0}\delta^{-1}\tone{\sft{(\X+\delta\Eb) \W^\top \x}-\sft{\X \W^\top \x}}+\delta^{-1}\tone{\sft{\X \W^\top (\x+\delta\beps)}-\sft{\X \W^\top \x}}\\
&\leq \Gamma e^{2\Gamma} \|\Eb\|_{2,1}/n+\Gamma e^{2\Gamma}\tn{\beps} \leq 2\Gamma e^{2\Gamma} \Ceb/n.
\end{align*}
For the last line, we re-used \eqref{main pert eq1} and \eqref{main pert eq2}. To conclude, combining with $\pb_2$ bound, we obtained the desired result.
\end{proof}

\begin{lemma}[Single-layer transformer stability] \label{trans stable}Consider the setup of Lemma \ref{lem stable}. Let $\phi$ be a $1$-Lipschitz activation function with $\phi(0)=0$ (e.g.~ReLU or Identity). Let $(\M_i)_{i=1}^n\in\R^{d\times d}$ be weights of the parallel MLPs following self-attention. Suppose $\|\M_i\|\leq 1$ and denote the MLP outputs associated to $\A,\Ab$ by $\B,\Bb$. We have that
\[
\|\B\|_{2,\infty},\|\Bb\|_{2,\infty}\leq 1,\quad \|\B-\Bb\|_{2,1}\leq (2\Gamma+1)e^{2\Gamma}\|\Eb\|_{2,1}.
\]
\red{Additionally, for any $i\in[n]$ such that $\tn{\beps_i}\leq \Ceb/n$, we have $\tn{\B_i-\Bb_i}\leq \frac{1}{n}(2\Gamma+1)e^{2\Gamma}\Ceb$ where $\B_i$ denotes the $i$th row of $\B$.}
\end{lemma}
\begin{proof} First note that each row of $\Bb$ is given by $\bb_i=\phi(\M_i\ab_i)$ thus $\tn{\bb_i}\leq \tn{\phi(\M_i\ab_i)}\leq \tn{\M_i\ab_i}\leq \tn{\ab_i}\leq 1$. Secondly, we can write $\tn{\bb_i-\bar{\bb}_i}\leq \tn{\phi(\M_i\ab_i)-\phi(\M_i\bar{\ab}_i)}\leq \tn{\M_i(\ab_i-\bar{\ab}_i)}\leq \tn{\ab_i-\bar{\ab}_i}$. Thus, we conclude via Lemma \ref{lem stable} because all row perturbations of $\B$ are dominated by those of $\A$ and $\|\B-\Bb\|_{2,1}\leq \|\A-\Ab\|_{2,1}$.
\end{proof}

\begin{theorem}\label{stable thm} Consider an $L$-layer transformer $\TF$ that maps $n$ tokens into $n$ tokens with (1) self-attention weights: combined key-query weights $(\W_i)_{i=1}^L\in\R^{d\times d}$ and value weights $(\Vb_i)_{i=1}^L\in\R^{d\times d}$, (3) MLP weights $(\M^{(i)}_j)_{i=1,j=1}^{L,n}\in\R^{d\times d}$ with $1$-Lipschitz activations $\phi^{(i)}$ obeying $\phi^{(i)}(0)=0$. For some $\Gamma>0$, assume $\|\Vb_i\|\leq 1,\|\M^{(i)}_j\|\leq 1, \|\W_i\|\leq \Gamma/2$. Suppose input space is $\Zb=[\z_1~\z_2~\dots~\z_n]^\top$ with $\tn{\z_i}\leq 1$. The model prediction is given as follows
\begin{itemize}
\item $\Zb_{(0)}=\Zb$. Layer $i$ outputs $\Zb_{(i)}=\texttt{Parallel\_MLP}_{\M^{(i)}}(\texttt{Att}_{\W_i,\Vb_i}(\Zb_{(i-1)})))$. Here the self-attention layer is given by $\texttt{Att}_{\W_i,\Vb_i}(\Zb)=\sft{\Zb \W_i \Zb^\top} \Zb\Vb$ and $\texttt{Parallel\_MLP}$ applies $f(\x)=\phi^{(i)}(\M^{(i)}_{j}\x)$ on $j^{th}$ token of the $\texttt{Att}$ output.
\item $\TF(\Zb)=\Zb_{(L)}$ and denote the $i$'th token output by $\TF_{(i)}(\Zb)$. 
\end{itemize}
The following statements hold
\begin{enumerate}
\item Assume activations are $\phi^{(i)}\in\{\texttt{ReLU},\text{Identity}\}$ with final layer $\phi^{(L)}=\text{Identity}$. This model is properly normalized in the sense that $\TF_{(i)}(\Zb)$ can output any vector $\tn{\vb}\leq 1$ despite no residual/skip connections. 
\item Let $\Zb'$ be a perturbation on $\Zb$ where all tokens are allowed to change however the change over the last token obeys $\tn{\z_n-\z'_n}\leq \Ceb/n$ where $\Ceb$ is also an upper bound on $\|\Zb-\Zb'\|_{2,1}$. This model obeys the stability guarantee
\begin{align}
|\TF_{(n)}(\Zb)-\TF_{(n)}(\Zb')|\leq \frac{1}{n}((1+\Gamma)e^\Gamma)^L \Ceb.\label{last token}
\end{align}
\end{enumerate}
\end{theorem}
\begin{proof} To see the first claim, let us set $\Vb_i=\M^{(i)}_{l}=\Iden$ (except for $\M^{(L)}$) and set all tokens $\z_i$ to be identical i.e.~$\Zb=\onebb_n\z^\top$. Additionally choose a $\z$ with $\tn{\z}=1$ and nonnegative entries. Observe that, thanks to the softmax structure, regardless of $\W_i$, we have that $\Zb=\texttt{Att}_{\W_i,\Vb_i}(\Zb)=\sft{\Zb \W_i \Zb^\top} \Zb$. After attention, MLPs again preserves the tokens i.e.~$\phi(\M_{i,l}\z_i)=\z$ for $\phi\in\{\texttt{ReLU},\text{Identity}\}$. Thus, after proceeding $L$ layers of this, right before the final MLP, the model outputs $\Zb=\onebb_n\z^\top$. Then, given a target vector $\tn{\vb}\leq 1$, choose the final MLP to $\M^{(L)}=\vb\z^\top$ to output an all $\vb$'s sequence.

Note that, in general $\Zb$ can be arbitrary (they don't have to be all same tokens): We can let $\W\rightarrow\infty$ (by allowing a larger $\Gamma$). This way the attention matrix implements $\sft{\Zb \W \Zb^\top}\rightarrow \Iden$ and we end up with the same argument of $\Zb$ being (almost perfectly) transmitted across the layers so that we obtain any target sequence in $\R^{n\times d}$.



\noindent \textbf{Main claim \eqref{last token}:} To show the stability guarantee, we use Lemmas \ref{lem stable} and \ref{trans stable}. Set $\Ceb=\|\Zb-\Zb'\|_{2,1}$ and recall the last token is not modified. Recall that Lemma \ref{trans stable} guarantees that 
\begin{itemize}
\item After each layer we are guaranteed to have $\|\Zb_{(i)}\|_{2,\infty},\|\Zb'_{(i)}\|_{2,\infty}\leq 1$.
\item After each layer we are guaranteed to have $\|\Zb_{(i)}-\Zb'_{(i)}\|_{2,1} \leq  (1+\Gamma)e^{\Gamma}\|\Zb_{(i-1)}-\Zb'_{(i-1)}\|_{2,1}$.
\end{itemize}
The latter implies that, for all layers, we have
\begin{align}
\|\Zb_{(i)}-\Zb'_{(i)}\|_{2,1} \leq  ((1+\Gamma)e^{\Gamma})^i\Ceb.\label{ZbCeb bound}
\end{align}
\red{What remains is running induction on the last tokens $\z^{(i)}_n-\z'^{(i)}_n$. We claim that, at all layers $\tn{\z^{(i)}_n-\z'^{(i)}_n}\leq \frac{1}{n}((1+\Gamma)e^{\Gamma})^i\Ceb$. This claim is true at $i=0$ due to the change over last token being at most $\|\Zb-\Zb'\|_{2,1}/n$. Assuming true at $i$ and since \eqref{ZbCeb bound} holds, for $i+1$, we apply Lemma \ref{trans stable}'s last line to obtain $\tn{\z^{(i+1)}_n-\z'^{(i+1)}_n}\leq \frac{1}{n}((1+\Gamma)e^{\Gamma})^{i+1}\Ceb$. Consequently, induction holds and we conclude with the proof by setting $i=L$.}
\end{proof}
\subsection{Proof of Theorem \ref{main stable}}\label{sec stable proof}
\begin{proof} We need to specialize Theorem \ref{stable thm} to obtain the result where the model outputs the last token thus we would like to apply \eqref{last token}. Observe that when prompts differ only at the inputs $\z_j=(\x_j,y_j)$ with $\z_j'=(\x'_j,y'_j)$, we have that $\|\X_{\prm}-\X'_{\prm}\|_{2,1}\leq 2$. This implies that $|\TF(\X_{\prm})-\TF(\X'_{\prm})|\leq \frac{2}{2m-1}((1+\Gamma)e^\Gamma)^D$ for a depth $D$ transformer. Finally, since the loss function $\ell$ is $L$-Lipschitz, we obtain the result $K=2L((1+\Gamma)e^\Gamma)^D$.
\end{proof}

The next lemma verifies our stability Assumption \ref{assump robust dynamical} for dynamical systems. In this below, we will assume that trajectories have bounded states almost surely (i.e.~$\bar{x}\leq 1$) so that Thm \ref{stable thm} is directly applicable. This can be guaranteed by choosing noise and initial state upper bounds  (respectively $\tn{\w_j}\leq \bar{w}$, $\tn{\x_0}\leq \bar{x}_0$) appropriately. We have the relation\footnote{Observe that each point in the trajectory is trivially bounded as $\tn{\x_i}\leq \bar{x}\leq C_{\rho}(\rho^i \bar{x}_0+\frac{1}{1-\rho}\bar{w})\leq C_{\rho}(\bar{x}_0+\frac{1}{1-\rho}\bar{w})$.} $\bar{x}\leq C_{\rho}(\bar{x}_0+\frac{1}{1-\rho}\bar{w})$. 

\begin{lemma}[Transformer Stability for Dynamical Systems]\label{TF dynamic stable} Consider the stable dynamical system setting of Section \ref{sec dynamic} and suppose that Assumption \ref{assump robust dynamical} holds. Let $\ell(\x,\hat{\x})=\ell(\x-\hat{\x})$ be $L$-Lipschitz in $\x-\hat{\x}$. Let $\xp{n} = (\x_0~\x_1~\dots~\x_{n})$ be a realizable  $(C_{\rho}, \rho<1)$-stable dynamical system trajectory and $\xpp{n}$ be the trajectory obtained by swapping $\w_j$ with $\w_{j}'$ ($j=0$ implies that $\x_0$ is swapped with $\x'_0$). As a result, starting with the $j$'th index, the prompt $\xpp{n}$ has different samples $(\x'_{j}, \dots, \x'_{n})$. Assume $\bar{x}\leq 1$ i.e.~all trajectory $(\x_i, \x'_i)_{i\geq 0}$ lie within the unit Euclidean ball in $\R^d$. Shape these prompts into matrices $\X_{\prm},\X'_{\prm}\in\R^{n\times d}$ respectively. Let $\TF(\cdot)$ be a $D$-layer transformer as described in Theorem \ref{stable thm}. Let $\TF$ output the last token of the final layer $\X_{(D)}$ that correspond to the query $\x_n$. Then Assumption \ref{assump robust dynamical} holds with $K=((1+\Gamma)e^\Gamma)^DC_\rho L$.
\end{lemma}
\begin{proof}
We again specialize Theorem \ref{stable thm} to obtain the result. Observe that when $\w_j$ is modified to $\w'_j$, then all the subsequent tokens will change. Also recall that due to unit ball assumption $\bar{w},\bar{x}_0,\bar{x}\leq 1$. Set $B_0=\tn{\w_j - \w'_j}$ if $j>0$ and $B_0=\tn{\x_0 - \x'_0}$ otherwise. Either way $B_0\leq 2$. Additionally, set $B_i=\tn{\x_{j+i}-\x'_{j+i}}$ for $n-j\geq i\geq 0$. From stability, we know that $B_i\leq C_\rho \rho^k B_{i-k}$. This means that
\begin{align}
&\tn{\x_n - \x'_n}\leq \frac{1}{n-j+1}\sum_{i=0}^{n-j} C_\rho \rho^i\tn{\x_{n-i} - \x'_{n-i}}\leq \frac{C_{\rho}}{n-j+1}\|\X_{\prm}-\X'_{\prm}\|_{2,1}.
\end{align}
Set $\Theta=\|\X_{\prm}-\X'_{\prm}\|_{2,1}$. To proceed, we choose 
\[
\max(\Theta,\frac{C_\rho n}{n-j+1}\Theta)=\frac{C_\rho n}{n-j+1}\Theta:=\Ceb,
\]
which satisfies the requirement of Theorem \ref{stable thm}. Now applying Theorem \ref{stable thm}, we find that, $n$'th output token perturbation obeys
\[
\tn{\TF_{(n)}(\xp{n})-\TF_{(n)}(\xpp{n})}\leq \frac{1}{n}((1+\Gamma)e^\Gamma)^D \Ceb\leq \frac{C_\rho((1+\Gamma)e^\Gamma)^D}{n-j+1}\Theta.
\]
Consequently, for any excitation $\w_{n+1}$ and using $L$-Lipschitzness of the loss, we find
\[
|\ell(\x_{m+1},\TF_{(n)}(\xp{n}))-\ell(\x'_{m+1},\TF_{(n)}(\xpp{n}))|\leq \frac{LC_\rho((1+\Gamma)e^\Gamma)^D}{n-j+1}\sum_{i=j}^n \tn{\x_i-\x'_i}.
\]
This means that stability holds with $K=((1+\Gamma)e^\Gamma)^DC_\rho L$.
\end{proof}
\subsection{Understanding when transformer-based ICL becomes unstable}\label{sec unstable}

\paragraph{Instability when attention weights are large.} We have the following lemma that complements our stability theorem and shows that instability can indeed arise when $\Gamma$ is large.
\red{\begin{lemma}Consider a length-$n$ input sequence $\X=[\x_1~\cdots~\x_n]^\top$ and a single self-attention layer with $\W = \Gamma \Iden,\V=\Iden$. Suppose all tokens are unit norm and the tokens from $2$ to $n-1$ are uncorrelated with the last token. Thus, $\Y=\X\X^\top$ has all ones diagonal, $\Y_{1,n},\Y_{n,1}=\rho$, and all remaining entries of the last row are zero. Suppose $\x_1$ is changed into $\x'_1=\gamma\x_1$ for some $1\geq \gamma\geq -1$. Let $\A=\text{softmax}(\X\W\X^\top)\X\Vb$ and $\ab_n$ denotes the last token. When $\rho= 1$, we have that
\[
\tn{\ab_n-\ab'_n}\geq \frac{\tn{\x_1-\x'_1}}{2+(n-2)e^{-\Gamma}}.
\]
\end{lemma}}
Thus, as soon as $\Gamma\geq\log(n-2)$, instability $\frac{\tn{\ab_n-\ab'_n}}{\tn{\x_1-\x'_1}}$ becomes $O(1)$ (specifically $\geq 1/3$).
\begin{proof} Let $\m=\sum_{i=2}^{n-1}\x_i$. Let $\rho'=\gamma\rho$. The self-attention outputs are given by
\[
\ab_n=\frac{e^{\Gamma}\x_n+e^{\rho\Gamma}\x_1+\m}{e^{\Gamma}+e^{\rho\Gamma}+(n-2)},\quad \ab'_n=\frac{e^{\Gamma}\x_n+e^{\rho'\Gamma}\x'_1+\m}{e^{\Gamma}+e^{\rho'\Gamma}+(n-2)}.
\]
Suppose $\rho=1$. By construction $\m^\top \x_n=0$, $\x_1=\x_n$. Also note that $\tn{\x_1-\x'_1}=1-\gamma$. 
With these, by only studying the change along the $\x_n$ direction (thanks to orthogonality) and setting $\rho=1$, we find that
\begin{align}
\frac{\tn{\ab_n-\ab'_n}}{\tn{\x_n}}&\geq \frac{2}{2+(n-2)e^{-\Gamma}}-\frac{1+\gamma e^{(\gamma-1)\Gamma}}{1+e^{(\gamma-1)\Gamma}+(n-2)e^{-\Gamma}}\nn \\
&\geq \frac{2}{2+(n-2)e^{-\Gamma}}-\frac{1+\gamma}{2}\frac{1+e^{(\gamma-1)\Gamma}}{1+e^{(\gamma-1)\Gamma}+(n-2)e^{-\Gamma}} \nn \\
&\geq \frac{2}{2+(n-2)e^{-\Gamma}}-\frac{1+\gamma}{2}\frac{2}{2+(n-2)e^{-\Gamma}} \nn \\
&\geq \frac{\tn{\x_1-\x'_1}}{2+(n-2)e^{-\Gamma}}. \nn
\end{align}
The final line is the advertised result.
\end{proof}
\textbf{Stability fails if we modify the last token (rather than earlier tokens).} Consider the setting of Theorem \ref{stable thm} and the statement \eqref{last token}. Below we show that, the requirement that last token should not be perturbed too much is indeed tight. This follows from the fact that, each token has a large say on their respective self-attention output, thus, perturbing them significantly perturbs their respective output (even if it cannot perturb other outputs too much).
\begin{lemma} Consider a single self-attention layer with $\W,\Vb=\Iden$ so that it outputs $\A=\sft{\X\X^\top}\X$. The last token outputs $\ab=\X^\top\sft{\X\x_n}$. Suppose $n$ is odd (for simplicity). There exists $\X$ with unit tokens/rows such that, for any perturbation amount $0\leq \eps\leq 1$, changing $\x_n$ to $\x'_n$ with $\tn{\x_n-\x'_n}=\eps$ can result in an output perturbation of
\[
\tn{\ab-\ab'}\geq \tn{\x_n-\x'_n}/5.
\]
Setting $\eps=1$, perturbing $\x_n$ results in $\geq 0.2$ perturbation regardless of $n$.
\end{lemma}
\begin{proof} If $n=1$, the model outputs $\ab=\x_n$ thus $\tn{\ab-\ab'}=\eps$. Now let $n'=(n-1)/2$ and $\vb\in\R^d$ with $\tn{\vb}=1$. Consider a toy setting where $\x_n=0$, the first $n'$ tokens are equal to $\vb$ and the next $n'$ tokens are equal to $-\vb$. Original attention output is $\ab=0$ due to symmetry. Now change the last token to $\eps\vb$ and using $\tn{\vb}=1$ and all tokens being aligned with $\vb$ observe that, for all $0\leq \eps\leq 1$
\[
\tn{\ab'}= \frac{e^{\eps}+(1/n')\eps e^{\eps^2}-e^{-\eps}}{e^{\eps}+e^{-\eps}+(1/n')e^{\eps^2}}\geq \frac{n-1}{2n}\frac{e^{\eps}-e^{-\eps}}{e^{\eps}}=\frac{n-1}{2n}(1-e^{-2\eps}) \geq 0.8\frac{n-1}{2n}\eps\geq \eps/5.
\]
\end{proof}

\section{Proofs and Supplementary Results for Sections~\ref{sec mtl} and \ref{sec transfer}}\label{app main proof}
\subsection{Proof of Theorem~\ref{thm mtl}}\label{app thm proof}


\begin{theorem}[Theorem~\ref{thm mtl} restated] \label{thm mtl recap}
    Suppose Assumption~\ref{assump robust} holds and assume loss function $\ell(\ys,\hat \ys)$ is $L$-Lipschitz for all $\y\in\Yc$ and takes values in $[0,B]$. 
    Let $\bah$ be the empirical solution of \eqref{mtl opt} and $\Nc(\Ac,\rho,u)$ be the covering number of the algorithm space $\Ac$ following Definition~\ref{def cover}\&\ref{def distance}. Then with probability at least $1-2\delta$, the excess MTL risk in \eqref{mtl risk} obeys
	\[
		\scalemath{1}{\Rmtl(\bah)\leq\inf_{\eps>0}\left\{4L\eps+2(B+K\log n)\sqrt{\frac{\log(\Nc(\Ac,\rho,\eps)/\delta)}{cnT}}\right\}}.\nn
	\]
	Additionally, set $D:=\sup_{\bal,\bal'\in\Ac}\rho(\bal,\bal')$ and assume $D<\infty$. With probability at least $1-4\delta$,
    \[
        \scalemath{1}{\Rmtl(\bah)\leq\inf_{\eps>0}\left\{8L\eps+8(2L+K\log n)\int_{\eps}^{D/2}\sqrt{\frac{\log\left(\log\frac{D}{\eps}\cdot\Nc(\Ac,\rho,u)/\delta\right)}{c'nT}}du\right\}+2(B+K\log n)\sqrt{\frac{\log(1/\delta)}{cnT}}.}
    \]
\end{theorem}
\begin{proof}
Recall the MTL problem setting of independent (input, label) pairs in Section~\ref{sec setup}: There are $T$ tasks each with $n$ in-context training samples denoted by $(\Sc_t)_{t=1}^T\distas(\Dc_t)_{t=1}^T$ where $\Sc_t=\{(\x_{ti},\ys_{ti})\}_{i=1}^n$, and let $\Sca=\bigcup_{t=1}^T\Sc_t$. 
We use $\Ac$ to denote the algorithm set. For a $\bal\in\Ac$, we define the training risk $\Lch_\Sca(\bal)=\frac{1}{nT}\sum_{t=1}^{T}\sum_{i=1}^n\ell(\ys_{ti},\fal{i-1}{t}(\x_{ti}))$, and the test risk $\Lc_\Dca(\bal)=\E[\Lch_\Sca(\bal)]$. Define empirical risk minima $\bah=\arg\min_{\bal\in\Ac}\Lch_\Sca(\bal)$ and population minima $\bal^\st=\arg\min_{\bal\in\Ac}\Lc_\Dca(\bal)$. For cleaner exposition, in the following discussion, we drop the subscripts $\Dca$ and $\Sca$. The excess MTL risk is decomposed as follows:
\begin{align*}
	\Rmtl(\bah)&=\Lc(\bah)-\Lc(\bal^\st)\\
	&=\underset{a}{\underbrace{\Lc(\bah)-\Lch(\bah)}}+\underset{b}{\underbrace{\Lch(\bah)-\Lch(\bal^\st)}}+\underset{c}{\underbrace{\Lch(\bal^\st)-\Lc(\bal^\st)}}.
\end{align*}
Since $\bah$ is the minimizer of empirical risk, we have $b\leq0$. To proceed, we consider the concentration problem of upper bounding $\sup_{\bal\in\Ac}|\Lc(\bal)-\Lch(\bal)|$. 


\noindent\textbf{Step 1: We start with a concentration bound {\normalfont{$|\Lc(\bal)-\Lch(\bal)|$}} for a fixed {\normalfont$\bal\in\Ac$}.}
Recall that we define the test/train risks of each task as follows:
\begin{align}
	&\Lch_t(\bal) :=\frac{1}{n}\sum_{i=1}^n \ell(\ys_{ti},\fal{i-1}{t}(\x_{ti})), \quad \text{and} \nonumber \\
	&\Lc_t(\bal) :=\E_{\Sc_t}\left[\Lch_t(\bal)\right]= \E_{\Sc_t}\left[\frac{1}{n}\sum_{i=1}^n\ell(\ys_{ti},\fal{i-1}{t}(\x_{ti}))\right], \quad \forall t \in [T]. \nonumber
\end{align}
Define the random variables $X_{t,i}=\E[\Lch_t(\bal)|\Scnt{i}{t}]$ for $i \in [n]$ and $t \in [T]$, that is, $X_{t,i}$ is the expectation over $\Lch_t(\bal)$ given training sequence $\Scnt{i}{t}=\{(\x_{tj},\ys_{tj})\}_{j=1}^i$ (which are the filtrations). With this, we have that $X_{t,n} = \E[\Lch_t(\bal)|\Scnt{n}{t}]=\Lch_t(\bal)$ and $X_{t,0} =\E[\Lch_t(\bal)]=\Lc_t(\bal)$. More generally, $(X_{t,0}, X_{t,1}, 
\dots, X_{t,n})$ is a martingale sequence since, for every $t$ in $[T]$, we have that $\E[X_{t,i} | \Scnt{i-1}{t}]= X_{t,i-1}$. 

For notational simplicitly, in the following discussion, we omit the subscript $t$ from $\x,\ys$ and $\Zb$ as they will be clear from left hand-side variable $X_{t,i}$. We have that
\begin{align*}
	X_{t,i}&=\E\bigg[\frac{1}{n}\sum_{j=1}^n \ell(\ys_{j},\fal{j-1}{}(\x_{j})) \bigg| \Scnt{i}{}\bigg]\\
	&=\frac{1}{n}\sum_{j=1}^{i}\ell(\ys_{j},\fal{j-1}{}(\x_{j}))+\frac{1}{n}\sum_{j=i+1}^{n}\E\left[\ell(\ys_{j},\fal{j-1}{}(\x_{j}))\bigg|\Scnt{i}{}\right]
\end{align*}
Next, we wish to upper bound the martingale increments i.e.~the difference of neighbors. Let $\Scnt{i:i}{}=\Scnt{i}{}-\Scnt{i-1}{}$ denote the $i$'th element.
\begin{align}
	|X_{t,i} - X_{t,i-1}| &= \left|\E\bigg[\frac{1}{n}\sum_{j=1}^n \ell(\ys_{j},\fal{j-1}{}(\x_{j})) \bigg| \Scnt{i}{}\bigg] - \E\bigg[\frac{1}{n}\sum_{j=1}^n \ell(\ys_{j},\fal{j-1}{}(\x_{j})) \bigg| \Scnt{i-1}{}\bigg]\right|\nonumber \\
	&\scalemath{1}{\leq\frac{1}{n}\sum_{j=i}^n\left| \E\left[\ell (\ys_j, \fal{j-1}{}(\x_j)) \bigg| \Scnt{i}{}\right] - \E\left[\ell(\ys_j, \fal{j-1}{}(\x_j))\bigg| \Scnt{i-1}{}\right]\right|} \nonumber \\
	&\scalemath{1}{\stackrel{(a)}{\leq}\frac{B}{n} + \frac{1}{n}\sum_{j=i+1}^n \left|\E\left[\ell (\ys_j, \fal{j-1}{}(\x_j)) \bigg| \Scnt{i}{}\right] - \E\left[\ell(\ys_j, \fal{j-1}{}(\x_j))\bigg| \Scnt{i-1}{}\right]\right|} \nonumber.
\end{align}
Here, $(a)$ follows from the fact that loss function $\ell(\cdot,\cdot)$ is bounded by $B$. 
To proceed, call the right side terms $D_j:=|\E[\ell (\ys_j, \fal{j-1}{}(\x_j)) \big| \Scnt{i}{}] - \E[\ell(\ys_j, \fal{j-1}{}(\x_j))\big| \Scnt{i-1}{}]|$. Denote $\z'_\ell$ to be the realized values of the variables $\z_\ell=(\y_\ell,\x_\ell)$ given $\Scnt{i}{}$. Let $\Sc:=(\z'_1,\dots,\z'_i,\z_{i+1},\dots,\z_j)$ and $\Sc':=(\z'_1,\dots,\z'_{i-1},\z_{i},\dots,\z_j)$. Observe that, $\Sc'$ and $\Sc$ differs in only at $i$th index and $i<j$, thus, utilizing Assumption \ref{assump robust},
\begin{align}\label{Dj bound}
D_j:=|\E[\ell (\y_j, \fac{}{}(\x_j))] - \E[\ell(\y_j, \fac{'}{}(\x_j))]|\leq \frac{K}{j}.
\end{align}
Combining above, for any $n\geq i\geq 1$, we obtain
\[
|X_{t,i} - X_{t,i-1}|\leq \frac{B}{n}+\sum_{j=i+1}^n \frac{K}{jn}\leq \frac{B+K\log n}{n}.
\]
Recall that $|\Lc_t(\bal) - \Lch_t(\bal)| = | X_{t,0} - X_{t,n}|$ and for every $t \in [T]$, we have $\sum_{i=1}^{n}|X_{t,i} - X_{t,i-1}|^2\leq\frac{(B+K\log n)^2}{n}$. As a result, applying Azuma-Hoeffding's inequality, we obtain 
\begin{align}
	\P(|\Lc_t(\bal)-\Lch_t(\bal)|\geq \tau)\leq2 e^{-\frac{n\tau^2}{2(B+K\log n)^2}}, \quad \forall t \in [T].\label{azuma prob}
\end{align}
Let us consider $Y_t := \Lc_t(\bal)-\Lch_t(\bal)$ for $t \in [T]$.
Then, $(Y_t)_{t=1}^T$ are i.i.d. zero mean sub-Gaussian random variables. There exists an absolute constant $c_1 > 0$ such that, the subgaussian norm, denoted by $\tsub{\cdot}$, obeys $\tsub{Y_t}^2 < \frac{c_1 (B+K\log n)^2}{n} $ via Proposition 2.5.2 of \cite{vershynin2018high}. Applying Hoeffding's inequality, we derive
\begin{align}
	\P\left(\bigg|\frac{1}{T} \sum_{t=1}^T Y_{t} \bigg|\geq \tau \right) \leq 2e^{-\frac{cnT\tau^2}{(B+K\log n)^2}} \Longrightarrow\P(|\Lch(\bal)-\Lc(\bal)|\geq \tau)\leq2 e^{-\frac{cnT\tau^2}{(B+K\log n)^2}}
\end{align}
where $c>0$ is an absolute constant. Therefore, we have that for any $\bal\in\Ac$, with probability at least $1-2\delta$,
\begin{align}
	|\Lch(\bal)-\Lc(\bal)|\leq (B+K\log n)\sqrt{\frac{\log(1/\delta)}{cnT}}.\label{single con bound}
\end{align}

\noindent\textbf{Step 2: Next, we turn to bound {\normalfont$\sup_{\bal\in\Ac}|\Lc(\bal)-\Lch(\bal)|$} where $\Ac$ is assumed to be a continuous search space.} To start with, set $g(\bal):=\Lc(\bal)-\Lch(\bal)$ and we aim to bound $\sup_{\bal\in\Ac}|g(\bal)|$. Following Definition~\ref{def distance}, for $\eps>0$, let $\Ac_\eps$ be a minimal $\eps$-cover of $\Ac$ in terms of distance metric $\rho$. Therefore, $\Ac_\eps$ is a discrete set with cardinality $|\Ac_\eps|:=\Nc(\Ac,\rho,\eps)$. Then, we have
\[
	\sup_{\bal\in\Ac}|\Lc(\bal)-\Lch(\bal)|\leq\sup_{\bal\in\Ac}\min_{\bal'\in\Ac_\eps}\left|g(\bal)-g(\bal')\right|+\max_{\bal\in\Ac_\eps}\left|g(\bal)\right|.
\]

\noindent$\bullet$ We start by bounding $\sup_{\bal\in\Ac}\min_{\bal'\in\Ac_\eps}\left|g(\bal)-g(\bal')\right|$. We will utilize that loss function $\ell(\cdot,\cdot)$ is $L$-Lipschitz. For any $\bal\in\Ac$, let $\bal'\in\Ac_\eps$ be its neighbor following Definition~\ref{def distance}. We have that
\begin{align*}
	\left|\Lch(\bal)-\Lch(\bal')\right|&=\left|\frac{1}{nT}\sum_{t=1}^T\sum_{i=1}^n\left(\ell(\ys_{ti},\fal{i-1}{t}(\x_{ti}))-\ell(\ys_{ti},\falp{i-1}{t}(\x_{ti}))\right)\right|\\
	&\leq \frac{L}{nT}\sum_{t=1}^T\sum_{i=1}^n\left\|\fal{i-1}{t}(\x_{ti})-\falp{i-1}{t}(\x_{ti})\right\|_{\ell_2}\\
	&\leq L\eps.
\end{align*}
Since the same bound applies to all data-sequences, we also obtain that for any $\bal\in\Ac$,
\[
	\left|\Lc(\bal)-\Lc(\bal')\right|\leq L\eps.
\]
Therefore, 
\begin{align}
\scalemath{1}{\sup_{\bal\in\Ac}\min_{\bal'\in\Ac_\eps}\left|g(\bal)-g(\bal')\right|\leq\sup_{\bal\in\Ac}\min_{\bal'\in\Ac_\eps}\left|\Lch(\bal)-\Lch(\bal')\right|+\left|\Lc(\bal)-\Lc(\bal')\right|\leq 2L\eps}.\label{cover bound}
\end{align}

\noindent$\bullet$ Next, we turn to bound the second term $\max_{\bal\in\Ac_\eps}\left|g(\bal)\right|$. Applying union bound directly on $\Ac_\eps$ and combining it with \eqref{single con bound}, then we will have that with probability at least $1-2\delta$, 
\begin{align}
	\max_{\bal\in\Ac_\eps}\left|g(\bal)\right|\leq (B+K\log n)\sqrt{\frac{\log(\Nc(\Ac,\rho,\eps)/\delta)}{cnT}}.\label{non chain bound}
\end{align}
{\textbf{Proof of Eq.~\eqref{main result 1}:} Combining the upper bound above with the perturbation bound \eqref{cover bound}, we obtain that
\begin{align}
\max_{\bal\in\Ac}\left|g(\bal)\right|\leq 2L\eps+(B+K\log n)\sqrt{\frac{\log(\Nc(\Ac,\rho,\eps)/\delta)}{cnT}}.\label{basic cover bound}
\end{align}
This in turn concludes the proof of \eqref{main result 1} since $\Rmtl(\bah)\leq 2\sup_{\bal\in\Ac}|\Lc(\bal)-\Lch(\bal)|.$}

\noindent\textbf{Proof of Eq.~\eqref{main result 2}:} To conclude, we aim to establish \eqref{main result 2}. Specifically, the precise statement we will establish is stated below
\begin{align}\label{main result 22}
        \Rmtl(\bah)\leq\inf_{\eps>0}\Bigg\{8L\eps+ \frac{L_++K\log n}{\sqrt{cnT}}\Big(\int_{\eps}^{D/2}\sqrt{\log\Nc(\Ac,\rho,u)}du+D_+\sqrt{\log(\log({D}/{\eps})/\delta)}\Big)\Bigg\}.
\end{align}
where we use the convention $x_+=\max(x,1)$. To this end, we will bound $\max_{\bal\in\Ac_\eps}\left|g(\bal)\right|$ via successive $\eps$-covers which is the chaining argument. Following Definition~\ref{def distance}, let $D:=\sup_{\bal,\bal'\in\Ac}\rho(\bal,\bal')$. Define $M:=\min\{m:2^m\eps\geq D\}$, and for any $m\in[M]$, let $\Uc_m$ denote the minimal $2^m\eps$-cover of $\Uc_{m-1}$, where $\Uc_0:=\Ac_\eps$. Since $\Uc_M\subseteq\Uc_{M-1}\cdots\subseteq\Uc_0 \subset \Ac$, we have $|\Uc_m| = \Nc(\Uc_{m-1},\rho,2^{m}\eps) \leq \Nc(\Ac, \rho, 2^m\eps)$ and $|\Uc_M|\leq\Nc(\Ac, \rho, D)=1$. Let $\bal^M\in\Uc_M$ denote the unique algorithm hypothesis in $\Uc_M$. We have that
\begin{align}
	\max_{\bal\in\Ac_\eps}\left|g(\bal)\right|&\leq\max_{\bal\in\Uc_0}\left|g(\bal)-g(\bal^M)\right|+\left|g(\bal^M)\right|\nn\\
	&\leq\sum_{m=0}^{M-1}\max_{\bal\in\Uc_{m}}\min_{\bal'\in\Uc_{m+1}}\left|g(\bal)-g(\bal')\right|+\left|g(\bal^M)\right|.\label{M over M}
\end{align}

In what follows, we will prove that for any $\bal,\bal'$ satisfying $\rho(\bal,\bal')\leq u~(u>0)$, with high probability, $\left|g(\bal)-g(\bal')\right|$ is bounded by $\frac{u}{\sqrt{nT}}$ up to logarithmic terms.

Apply similar martingale sequence analysis as in Step 1. This time, we set $X_{t,i}=\E[\Lch_t(\bal)-\Lch_t(\bal')|\Scnt{i}{t}]$ where we assume $\rho(\bal,\bal')\leq u$. Similarly, we have that $X_{t,n} =\Lch_t(\bal)-\Lch_t(\bal')$, and $X_{t,0} = \E[\Lch_t(\bal)-\Lch_t(\bal')]=\Lc_t(\bal)-\Lc_t(\bal')$. Therefore, the sequences $\{X_{t,0},\dots, X_{t,n}\},t\in[T]$ are Martingale sequences with respect to $\E[X_{t,i} | \Scnt{i-1}{t}]= X_{t,i-1}$. We again omit the subscript $t$ for $\x,\ys$ and $\Sc$ in the following and try to bound the difference of neighbors. 
\begin{align}
	&\scalemath{1}{|X_{t,i} - X_{t,i-1}| = \left|\E\bigg[\Lch_t(\bal)-\Lch_t(\bal') \bigg| \Scnt{i}{}\bigg] - \E\bigg[\Lch_t(\bal)-\Lch_t(\bal') \bigg| \Scnt{i-1}{}\bigg]\right|}\nonumber \\
	&\scalemath{1}{\leq\frac{1}{n}\sum_{j=i}^n\left| \E\left[\ell (\ys_j, \fal{j-1}{}(\x_j))-\ell (\ys_j, \falp{j-1}{}(\x_j)) \bigg| \Scnt{i}{}\right] - \E\left[\ell(\ys_j, \fal{j-1}{}(\x_j))-\ell (\ys_j, \falp{j-1}{}(\x_j))\bigg| \Scnt{i-1}{}\right]\right|} \nonumber \\
	&\scalemath{1}{\stackrel{(d)}{\leq}\frac{2Lu}{n} + \frac{1}{n}\sum_{j=i+1}^n \left|\E\left[\ell (\ys_j, \fal{j-1}{}(\x_j))-\ell (\ys_j, \falp{j-1}{}(\x_j))\bigg| \Scnt{i}{}\right] - \E\left[\ell(\ys_j, \fal{j-1}{}(\x_j))-\ell (\ys_j, \falp{j-1}{}(\x_j))\bigg| \Scnt{i-1}{}\right]\right|} \nonumber \\
	&\scalemath{1}{\stackrel{(e)}{\leq} \frac{2Lu}{n} + \frac{1}{n}\sum_{j=i+1}^{n} \frac{Ku}{j}< \frac{2Lu+ Ku\log n}{n}}. \nonumber
	\end{align}
for $i<n$. Here, $(d)$ is from the facts that loss function $\ell(\cdot,\cdot)$ is $L$-Lipschitzness and $\rho(\bal,\bal')\leq u$ by following the same analysis in deriving \eqref{cover bound}, and $(e)$ follows Assumption~\ref{assump robust}. Then we have 
\[
|X_{t,n} - X_{t,n-1}|\leq\frac{2Lu}{n}< \frac{2Lu+ Ku\log n}{n}.
\]
Note that $|\Lc_t(\bal)-\Lc_t(\bal') - (\Lch_t(\bal)-\Lch_t(\bal'))| = | X_{t,0} - X_{t,n}|$ and for every $t \in [T]$, we have $\sum_{i=1}^{n}|X_{t,i} - X_{t,i-1}|^2\leq\frac{u^2(2L+K\log n)^2}{n}$. As a result of applying Azuma-Hoeffding's inequality, we obtain 
\begin{align*}
	\P(|\Lc_t(\bal)-\Lc_t(\bal') - (\Lch_t(\bal)-\Lch_t(\bal'))|\geq \tau)\leq2 e^{-\frac{n\tau^2}{2u^2(2L+K\log n)^2}}, \quad \forall t \in [T].
\end{align*}
Now let us instead consider $Y_t := g(\bal)-g(\bal')$ for $t \in [T]$. 
Then following proof as in Step 1, we derive
\begin{align*}
	\P\left(\bigg|\frac{1}{T} \sum_{t=1}^T Y_{t} \bigg|\geq \tau \right) < 2e^{-\frac{c'nT\tau^2}{u^2(2L+K\log n)^2}} \Longrightarrow\P(|g(\bal)-g(\bal')|\geq \tau)\leq2 e^{-\frac{c'nT\tau^2}{u^2(2L+K\log n)^2}}
\end{align*}
where $c'>0$ is an absolute constant. 
Consider the discrete set $\Uc_m$ with cardinality $|\Uc_m|=\Nc(\Uc_m,\rho,2^m\eps)\leq\Nc(\Ac,\rho,2^m\eps)$ and its $2^{m+1}\eps$-cover $\Uc_{m+1}$. Applying union bound over $\Uc_m$, we have that with probability at least $1-2\delta$,
\begin{align*}
	\max_{\bal\in\Uc_m}\min_{\bal'\in\Uc_{m+1}}|g(\bal)-g(\bal')|&\leq 2^{m+1}\eps(2L+K\log n)\sqrt{\frac{\log(\Nc(\Ac,\rho,2^m\eps)/\delta)}{c'nT}}.
\end{align*}
Now by again applying union bound, with probability at least $1-2\delta$, the first term in \eqref{M over M} is bounded by
\begin{align}
	\sum_{m=0}^{M-1}\max_{\bal\in\Uc_{m}}\min_{\bal'\in\Uc_{m+1}}\left|g(\bal)-g(\bal')\right|&\leq (2L+K\log n)\sum_{m=0}^{M-1}2^{m+1}\eps\sqrt{\frac{\log(M\cdot\Nc(\Ac,\rho,2^m\eps)/\delta)}{c'nT}}\nn\\
	&\leq4(2L+K\log n)\int_{\eps/2}^{D/2}\sqrt{\frac{\log(M\cdot\Nc(\Ac,\rho,u)/\delta)}{c'nT}}du.\label{int bound}
\end{align}


Now combining the results of \eqref{single con bound}, \eqref{M over M} and \eqref{int bound}, and following the evidence that $\bal^M\in\Uc_M$ is unique, we bound $\sup_{\bal\in\Ac_\eps}|g(\bal)|$ as follows, that with probability at least $1-4\delta$ 

\begin{align}
	\scalemath{1}{\sup_{\bal\in\Ac_\eps}\left|g(\bal)\right|\leq4(2L+K\log n)\int_{\eps/2}^{D/2}\sqrt{\frac{\log(M\cdot\Nc(\Ac,\rho,u)/\delta)}{c'nT}}du+(B+K\log n)\sqrt{\frac{\log(1/\delta)}{cnT}}}.\label{chain bound}
\end{align}
Here $D:=\sup_{\bal,\bal'\in\Ac}\rho(\bal,\bal')$ and $M:=\min\{m:2^m\eps\geq D\}$.

$\bullet$ 
Combining \eqref{cover bound} and \eqref{chain bound}, we obtain that with probability at least $1-4\delta$,
\begin{align}
	\scalemath{1}{\sup_{\bal\in\Ac}\left|\Lc(\bal)-\Lch(\bal)\right|\leq\inf_{\eps>0}\left\{4L\eps+4(2L+K\log n)\int_{\eps}^{D/2}\sqrt{\frac{\log(M\cdot\Nc(\Ac,\rho,u)/\delta)}{c'nT}}du+(B+K\log n)\sqrt{\frac{\log(1/\delta)}{cnT}}\right\}},\nn
\end{align}
where $D:=\sup_{\bal,\bal'\in\Ac}\rho(\bal,\bal')$ and $M:=\min\{m:2^{m+1}\eps\geq D\}\leq\log\frac{D}{\eps}$.

Applying $\Rmtl(\bah)\leq2\sup_{\bal\in\Ac}\left|\Lc(\bal)-\Lch(\bal)\right|$ completes the proof.
\end{proof}
Till now, we consider the setting where each task is trained with only one trajectory. In the following, we also consider the case where each task contains multiple trajectories. To start with, we define the following objective function as an extension of \eqref{mtl opt} to the multi-trajectory setting.
\begin{align}\
    \bah=\arg\min_{\bal\in\Ac}&\Lch_{\Sca}(\bal):=\frac{1}{TM}\sum_{t=1}^T\sum_{m=1}^M\Lch_{t,m}(\bal)\label{mtl opt with M}\\
    \text{where}~~~~&\Lch_{t,m}(\bal)=\frac{1}{n}\sum_{i=1}^n \ell(\ys_{tmi},\fal{i-1}{t}(\x_{tmi})).\nn
\end{align}
Here, we assume each task $t\in[T]$ contains $M$ trajectories, and $\Sca=\{\{\Sc_{t,m}\}_{m=1}^M\}_{t=1}^T$ where $\Sc_{t,m}=\{(\x_{tmi},\ys_{tmi})\}_{i=1}^n$ and $(\x_{tmi},\ys_{tmi})\distas\Dc_t$. Then the following theorem states a more general version of Theorem~\ref{thm mtl}.
\begin{theorem}
Suppose the same assumptions as in Theorem~\ref{thm mtl} hold and let $\bah$ be the empirical solution of \eqref{mtl opt with M}. Then, with the same probability, we {obtain identical bounds to} Theorem~\ref{thm mtl} by updating $T$ with $TM$ in Equations \eqref{main result 1} and \eqref{main result 2}. 
\end{theorem}
Choosing $M=1$ results in the exactly same bound as in Theorem~\ref{thm mtl}.

\begin{proof}
Following the same proof steps, and then we derive similar result as \eqref{azuma prob}:
\begin{align}
	\P(|\Lc_{t,m}(\bal)-\Lch_{t,m}(\bal)|\geq \tau)\leq2 e^{-\frac{n\tau^2}{2(B+K\log n)^2}}, \quad \forall t \in [T], m\in[M].
\end{align}
Let $Y_{t,m}:=\Lc_{t,m}(\bal)-\Lch_{t,m}(\bal)$. Since in-context samples are independent, $((Y_{t,m})_{m=1}^M)_{t=1}^T$ are independent zero mean sub-Gaussian random variables, with norm $\tsub{Y_{t,m}} < \frac{c_1 (B+K\log n)^2}{n}$. Applying Hoeffding's inequality, we derive
\begin{align}
	\P(|\Lch(\bal)-\Lc(\bal)|\geq \tau)\leq2 e^{-\frac{cnMT\tau^2}{(B+K\log n)^2}}
\end{align}
where $c>0$ is an absolute constant. Therefore, we have that for any $\bal\in\Ac$, with probability at least $1-2\delta$,
\begin{align}
	|\Lch(\bal)-\Lc(\bal)|\leq (B+K\log n)\sqrt{\frac{\log(1/\delta)}{cnMT}}.
\end{align}
The result is simply replacing $T$ with $MT$ in \eqref{single con bound}. It is from the fact that trajectories are all independent no matter they are from the same task or not. By applying the similar analysis, the proof is competed.
\end{proof}
\subsection{Transfer Learning Bound with i.i.d.~Tasks}\label{sec tl bound}
Following training with \eqref{mtl opt}, suppose source tasks are i.i.d. sampled from a task distribution $\dtask$, and let $\bah$ be the empirical MTL solution. We consider the following transfer learning problem. Concretely, assume a target task $\Tc$ with a distribution $\Tc\sim\dtask$ and training sequence $\Sc_\Tc=(\z_i)_{i=1}^n\sim\Dc_\Tc$. Define the empirical and population risks on $\Tc$ as $\Lch_\Tc(\bal)=\frac{1}{n}\sum_{i=1}^n\ell(\y_{i},\fal{i-1}{\Tc}(\x_{i}))$ and $\Lc_\Tc(\bal)=\E_{\Sc_\Tc}[\Lch_\Tc(\bal)]$. Then the expected excess transfer risk following \eqref{mtl opt} is defined as
\begin{align}
    \E_{\Tc}[\Rc_\Tc(\bah)]=\E_{\Tc}[\Lc_\Tc(\bah)]-\arg\min_{\bal\in\Ac}\E_{\Tc}[\Lc_\Tc(\bal)].\label{tf risk expected}
\end{align}
\begin{theorem}\label{tfr thm}
    Consider the setting of Theorem~\ref{thm mtl} and assume the source tasks are independently drawn from task distribution $\dtask$. Let $\bah$ be the empirical solution of \eqref{mtl opt} and $\Tc\sim\dtask$. Then with probability at least $1-2\delta$, the expected excess transfer learning risk \eqref{tf risk expected} obeys
    \begin{align}
        \E_{\Tc}[\Rc_\Tc(\bah)]\leq\min_{\eps\geq0}\left\{4L\eps+B\sqrt{\frac{2\log(\Nc(\Ac,\rho,\eps)/\delta)}{T}}\right\}.
    \end{align}
\end{theorem}
\begin{proof}
    Recap the problem setting in Section~\ref{sec setup} and let $\bal^\dagger=\arg\min_{\bal\in\Ac}\E_{\Tc}[\Lc_\Tc(\bal)]$. The expected transfer learning excess test risk of given algorithm $\bah\in\Ac$ is formulated as
    \begin{align}
        \E_\Tc[\Rc_\Tc(\bah)]&=\E_\Tc[\Lc_\Tc(\bah)]-\E_\Tc[\Lc_\Tc(\bal^\dagger)]\\
        &=\underset{a}{\underbrace{\E_\Tc[\Lc_\Tc(\bah)]-\Lch_\Sca(\bah)}}+\underset{b}{\underbrace{\Lch_\Sca(\bah)-\Lch_\Sca(\bal^\dagger)}}+\underset{c}{\underbrace{\Lch_\Sca(\bal^\dagger)-\E_\Tc[\Lc_\Tc(\bal^\dagger)]}}.
    \end{align}
    Here since $\bah$ is the minimizer of training risk, $b<0$. Then we obtain
    \begin{align}
        \E_\Tc[\Rc_\Tc(\bah)]\leq2\sup_{\bal\in\Ac}\left|\E_\Tc[\Lc_{\Tc}(\bal)]-\frac{1}{T}\sum_{t=1}^T\Lch_t(\bal)\right|.\label{bound main T}
    \end{align}
    For any $\bal\in\Ac$, let $X_t=\Lch_t(\bal)$ and we observe that 
    \[
        \E_{t\sim\dtask}[X_t]=\E_{t\sim\dtask}[\Lch_t(\bal)]=\E_{t\sim\dtask}[\Lc_t(\bal)]=\E_\Tc[\Lc_\Tc(\bal)].
    \]
    Since $X_t$, $t\in[T]$ are independent, and $0\leq X_t\leq B$, applying Hoeffding's inequality obeys
    \begin{align}
        \P\left(\left|\E_\Tc[\Lc_{\Tc}(\bal)]-\frac{1}{T}\sum_{t=1}^T\Lch_t(\bal)\right|\geq\tau\right)\leq 2e^{-\frac{2T\tau^2}{B^2}}.
    \end{align}
    Then with probability at least $1-2\delta$, we have that for any $\bal\in\Ac$,
    \begin{align}
        \left|\E_\Tc[\Lc_{\Tc}(\bal)]-\frac{1}{T}\sum_{t=1}^T\Lch_t(\bal)\right|\leq B\sqrt{\frac{\log(1/\delta)}{2T}}.\label{bound T}
    \end{align}
    Next, let $\Ac_\eps$ be the minimal $\eps$-cover of $\Ac$ following Definition~\ref{def cover}, which implies that for any task $\Tc\sim\dtask$, and any $\bal\in\Ac$, there exists $\bal'\in\Ac_\eps$
    \[
        |\Lc_\Tc(\bal)-\Lc_\Tc(\bal')|,|\Lch_\Tc(\bal)-\Lch_\Tc(\bal')|\leq L\eps.
    \]
    Since the distance metric following Definition~\ref{def distance} is defined by the worst-case datasets, then there exists $\bal'\in\Ac_\eps$ such that
    \begin{align}
        \left|\E_\Tc[\Lc_{\Tc}(\bal)]-\frac{1}{T}\sum_{t=1}^T\Lch_t(\bal)\right|\leq2L\eps.\label{bound cover T}
    \end{align}
    Let $\Nc(\Ac,\rho,\eps)=|\Ac_\eps|$ be the $\eps$-covering number. Combining the above inequalities (\eqref{bound main T}, \eqref{bound T} and \eqref{bound cover T}), and applying union bound, we have that with probability at least $1-2\delta$,
    \[
        \E_\Tc[\Rc_\Tc(\bah)]\leq\min_{\eps\geq0}\left\{4L\eps+B\sqrt{\frac{2\log(\Nc(\Ac,\rho,\eps)/\delta)}{T}}\right\}.
    \]
\end{proof}

\textbf{Understanding the MTL performance in Figure \ref{fig:transfer}:} Following transfer learning discussion in Sec \ref{sec transfer}, let us ask the same question for the MTL algorithm: If the transformer perfectly learns the MTL tasks $\bT_{\Dca}=(\bt_t)_{t=1}^T$, it does not actually need $n=\Omega(d)$ samples to perform well on new prompts drawn from source tasks. To see this, consider the following algorithm: Given a prompt, $\bal(\bT_{\Dca})$ conducts a discrete search over $(\bt_t)_{t=1}^T$ and returns the source task that best fits to the prompt. Thanks to the discrete search space, it is not hard to see that, we need $n\propto \log(T)$ samples rather than $n\propto d$ (also see Figure \ref{fig:hindsight}). {In contrast, based on Figures \ref{fig:transfer}(a,b,c), MTL behaves closer to $n\propto d$ empirically.} On the other hand, $\bal(\bT_{\Dca})$ implemented by the transformer is rather intelligent: This is because MTL risks for $d\in\{5,10,20\}$ are all strictly better than implementing least-squares\footnote{Ordinary least-squares achieves the minimum risk for transfer learning ($T=\infty$) however it is not optimal for finite $T$.} and the performance improves as $T$ gets smaller. We leave the thorough exploration of the inductive bias of the MTL training and characterization of $\bal(\bT_{\Dca})$ as an intriguing future direction.

\subsection{Transfer Learning from the Lens of Task Diversity}\label{sec task diversity}
In Section \ref{sec transfer}, we motivated the fact that transfer risk is controlled in terms of MTL risk and an additive term that captures the distributional distance i.e.~$\Lc_{\Tc}(\bal)\leq \Lc_{\Dca}(\bal)+\text{dist}(\Tc,(\Dc_t)_{t=1}^T)$. The following definition is a generalization of this relation which can be used to formally control the transfer risk in terms of MTL risk.
\begin{definition}[Task diversity]\label{def diversity} Following Section~\ref{sec setup}, we say that task $\Tc$ is $(\nu,\epsilon)$-diverse over the $T$ source tasks if for any $\bal,\bal'\in\Ac$,
    \begin{align*}
        \Lc_\Tc(\bal)-\Lc_\Tc(\bal')\leq\left(\frac{1}{T}\sum_{t=1}^T\left(\Lc_t(\bal)-\Lc_t(\bal')\right)\right)/{\nu}+\epsilon.
    \end{align*}
\end{definition}
Now let us discuss transferability in light of this assumption and Thm~\ref{thm mtl}. Consider the scenario where $n$ is small and $T\to\infty$. The excess MTL risk will be small thanks to infinitely many tasks. The transfer risk would also be small because larger $T$ results in higher diversity covering the task space. However, if the target task uses a different/longer prompt length, transfer may fail since the model never saw prompts longer than $n$. Conversely, if we let $n\to\infty$ and $T$ to be small, although the MTL risk is again zero, due to lack of diversity, it may not benefit transfer learning strongly. Task diversity assumption leads to the following lemma that bounds transfer learning in terms of MTL risk.
\begin{lemma}\label{thm tf} Consider the setting of Theorem~\ref{thm mtl}. Let $\bah$ be the solution of \eqref{mtl opt} and assume that target task $\Tc$ is $(\nu, \epsilon)$-diverse over $T$ source tasks. Then with the same probability as in Theorem~\ref{thm mtl}, the excess transfer learning risk {$R_\Tc(\bah)=\Lc_\Tc(\bah)-\min_{\bal\in\Ac}\Lc_{\Tc}(\bal)$} obeys $R_\Tc(\bah)\leq\frac{\Rmtl(\bah)}{\nu}+2\epsilon$.
\end{lemma}
Here we emphasize that the statement holds for arbitrary source and target tasks; however the challenge is verifying the assumption which is left as an interesting and challenging future direction. On the bright side, as illustrated in Figures \ref{fig:transfer}\&\ref{fig:div}, we indeed observe that, transfer learning can work with reasonably small $T$ and it works better if the target task is closer to the source tasks.

\begin{proof} 
Let $\bah,\bal^\st$ be the empirical and population solutions of \eqref{mtl opt} and let $\bal^\stt:=\underset{\bal\in\Ac}{\arg\min}\Lc_\Tc(\bal)$. Then the transfer learning excess test risk of given algorithm $\bah\in\Ac$ is formulated as
\begin{align*}
    R_\Tc(\bah)&=\Lc_\Tc(\bah)-\Lc_\Tc(\bal^\stt)\\
    &=\Lc_\Tc(\bah)-\Lc_\Tc(\bal^\st)+\Lc_\Tc(\bal^\st)-\Lc_\Tc(\bal^\stt).
\end{align*}
Since target task $\Tc$ is $(\nu,\eps)$-diverse over source tasks, following Definition~\ref{def diversity}, we derive that
\begin{align*}
    &\Lc_\Tc(\bah)-\Lc_\Tc(\bal^\st)\leq\frac{\Lc_\Dca(\bah)-\Lc_\Dca(\bal^\st)}{\nu}+\eps=\frac{\Rmtl(\bah)}{\nu}+\eps\\
    &\Lc_\Tc(\bal^\st)-\Lc_\Tc(\bal^\stt)\leq\frac{\Lc_\Dca(\bal^\st)-\Lc_\Dca(\bal^\stt)}{\nu}+\eps\leq\eps.
\end{align*}
Here, since $\bal^\st$ is the minimizer of $\Lc_\Dca(\bal)$,  $\Lc_\Dca(\bal^\st)-\Lc_\Dca(\bal^\stt)\leq0$. Then, Lemma~\ref{thm tf} is easily proved by combining the above two inequalities.
\end{proof}

\section{Proof of Theorem~\ref{thm dynamic mtl}}\label{sec dynamic mtl}

\begin{lemma}\label{dynamic assump lemma}
Suppose Assumptions \ref{assump stable dynamical} and \ref{assump robust dynamical} hold. Assume input and noise spaces $\Xc,\Wc$ are bounded by $\bar x,\bar w$. Let $W = (\w_1, \dots, \w_j, \w_{j+1}, \dots, \w_{m})$ and $W' = (\w_1, \dots, \w_{j-1}, \w'_j, \w_{j+1}, \dots, \w_{m})$ be two arbitrary sequences and the only difference between $W$ and $W'$ is the $j$'th term of the sequence. Allow the final excitation term $\w_{m+1}$ to be stochastic (and so are $\x_{m+1},\x'_{m+1}$).
Then, for any $f \in \Fc$, $\bal \in \Ac$, $W$, $W'$, $m$, and $j < m$, we have the following:
    \begin{align}
        \Bigg|\E_{\w_{m+1}}\left[\ell(\x_{m+1}, f^\bal_\Scnt{m}{}(\x_m)) \right] -\E_{\w_{m+1}}\left[\ell(\x'_{m+1},f^\bal_{\Scnt{m}{j}}(\x'_m)) \right] \Bigg|< \frac{K}{m\red{-j+1}} \frac{2\bar{C}_{\rho}\bar{w}}{1- \bar{\rho}}. \nonumber
    \end{align}
Additionally, for the sequences that differ only at their initial states, for any $\x_0,\x_0'\in\Xc$, we have
\begin{align}
    \Bigg|\E_{\w_{m+1}}\left[\ell(\x_{m+1}, f^\bal_\Scnt{m}{}(\x_m))\right] -\E_{\w_{m+1}}\left[\ell(\x'_{m+1},f^\bal_{\Scnt{m}{j}}(\x'_m))\right] \Bigg|< \frac{K}{m\red{-j+1}} \frac{2\bar{C}_{\rho}\bar{x}}{1- \bar{\rho}}. \nonumber
\end{align}
\end{lemma}
\begin{proof}
    First, let us bound $\tn{\x_i - \x_i'}$ for every $i = j, \dots, n$. For $i = j$, since $\Wc$ is bounded by $\bar w$, we have 
    \begin{align}
        \tn{\x_{j} - \x_{j}'} = \tn{f(\x_{j-1}) + \w_j - f(\x_{j-1})- \w_j'} \leq 2 \bar{w}\leq 2 \bar{C}_{\rho}\bar{w}. \nonumber
    \end{align}
    For $i > j$, we have the following from Assumption \ref{assump stable dynamical}:
    \begin{align}
        \tn{\x_i - \x_i'} \leq \bar{C}_{\rho} \bar{\rho}^{(i-j)} \tn{\x_{j} - \x_{j}'} \leq 2 \bar{C}_{\rho} \bar{\rho}^{(i-j)} \bar{w}. \nonumber
    \end{align}
    Finally, using Assumption \ref{assump robust dynamical}, we obtain
    \begin{align}
        \hspace{0cm}\Bigg|\E_{(\w_{m+1})}\left[\ell(\x_{m+1}, f^\bal_\Scnt{m}{}(\x_m)) \right] -& \E_{(\w_{m+1})}\left[\ell(\x'_{m+1},f^\bal_{\Scnt{m}{j}}(\x_m')) \right] \Bigg|  \nonumber \\
        &\leq \frac{K}{m\red{-j+1}} \sum_{i=j}^m\tn{\x_i - \x'_i}\nonumber \\
        &  \leq \frac{K}{m\red{-j+1}} 2\bar{C}_{\rho}\bar{w}\sum_{i=j}^m \bar{\rho}^{i-j} < \frac{K}{m\red{-j+1}} \frac{2\bar{C}_{\rho}\bar{w}}{1- \bar{\rho}}. \nonumber 
    \end{align}
    To prove the second part of the lemma, similarly we have
    \[
        \tn{\x_0-\x_0'}\leq2\bar x\quad\text{and then,}\quad\tn{\x_i-\x_i'}\leq2\bar C_\rho\bar\rho^i\bar x.
    \]
    Again using Assumption \ref{assump robust dynamical}, we obtain
    \begin{align}
        \hspace{0cm}\Bigg|\E_{(\w_{m+1})}\left[\ell(\x_{m+1}, f^\bal_\Scnt{m}{}(\x_m))\right] -& \E_{(\w_{m+1})}\left[\ell(\x'_{m+1},f^\bal_{\Scnt{m}{j}}(\x_m')) \right] \Bigg|  \nonumber \\
        &\leq \frac{K}{m\red{-j+1}} \sum_{i=0}^m\tn{\x_i - \x'_i} \nonumber \\
        &  \leq \frac{K}{m\red{-j+1}} 2\bar{C}_{\rho}\bar{x}\sum_{i=0}^m \bar{\rho}^{i} < \frac{K}{m\red{-j+1}} \frac{2\bar{C}_{\rho}\bar{x}}{1- \bar{\rho}}. \label{Martingale Increments Upper Bound Utilize}
    \end{align}
\end{proof}
\begin{theorem}[Theorem~\ref{thm dynamic mtl} restated] \label{thm mtl dynamic recap}
    Suppose Assumptions~\ref{assump stable dynamical} and \ref{assump robust dynamical} hold and assume loss function $\ell(\x,\hat\x):\Xc\times\Xc\to[0,B]$ is $L$-Lipschitz for all $\x\in\Xc$. 
    Let $\bah$ be the solution of \eqref{mtl opt} under the dynamical setting as described in Section~\ref{sec dynamic}. Then with probability at least $1-2\delta$, the excess MTL test risk \eqref{mtl risk} obeys
    \[
		\scalemath{1}{\Rmtl(\bah)\leq\inf_{\eps>0}\left\{4L\eps+2(B+\bar K\log n)\sqrt{\frac{\log(\Nc(\Ac,\rho,\eps)/\delta)}{cnT}}\right\}}.\nn
	\]
    where $\bar K=2K\frac{{\bar C}_\rho}{1-\bar\rho}(\bar w+\bar x /\sqrt{n})$.
\end{theorem}
\begin{proof} We follow the similar strategy as in the proof of Theorem~\ref{thm mtl}. The main difference is that we need to consider the dynamical system setting. Therefore, let us recall the dynamical problem setting in Sections~\ref{sec setup}\&\ref{sec dynamic}. Suppose there are $T$ independent trajectories generated by $T$ dynamical systems, denoted by $\Sc_t=(\x_{t0},\x_{t1},\cdots,\x_{tn})$, $t\in[T]$ where $\x_{ti}=f_t(\x_{t,i-1})+\w_{ti}$. Here, we consider the prediction function $\fal{i}{}:\Xc\rightarrow\Xc$, and denote the previously observed sequences with $\Scnt{i}{t}:=(\x_{t0},\cdots,\x_{ti})$. The objective function in \eqref{mtl opt} can be rewritten as follows:
\begin{align}
    \bah=\arg\min_{\bal\in\Ac}&\Lch_{\Sca}(\bal):=\frac{1}{T}\sum_{t=1}^T\Lch_t(\bal)\label{mtl dynamical opt}\\
    \text{where}\quad &\Lch_t(\bal)=\frac{1}{n}\sum_{i=1}^n \ell(\x_{ti},\fal{i-1}{t}(\x_{t,i-1})).\nn
\end{align}
Following the same argument as in the proof of Theorem~\ref{thm mtl}, the excess MTL risk is bounded by: 
\[
    \Rmtl(\bah)\leq2\sup_{\bal\in\Ac}|\Lc(\bal)-\Lch(\bal)|.
\]

\noindent\textbf{Step 1: We start with the concentration bound {\normalfont{$|\Lc(\bal)-\Lch(\bal)|$}} for any {\normalfont$\bal\in\Ac$}.}
Define the random variables $X_{t,i}=\E[\Lch_t(\bal)|\x_{t0}, (\w_{tk})_{k=1}^{i}]$ for $i \in [n]$ and $t \in [T]$, that is, $X_{t,i}$ is the expectation over $\Lch_t(\bal)$ given the filtration of $\x_{t0}$ and $(\w_{tk})_{k=1}^i$. Then, we have that $X_{t,n} = \E[\Lch_t(\bal)|\x_{t0}, (\w_{tk})_{k=1}^n]=\Lch_t(\bal)$. Let $X_{t,0} = \E[\Lch_t(\bal)]$. Then, for every $t$ in $[T]$, the sequences $\{X_{t,0}, X_{t,1}, \dots, X_{t,n}\}$ are Martingale sequences. Here we emphasize that $X_{t,0}=\E[X_{t,1}|\x_{t0},\w_{t1}]$. 
For the sake of simplicity, in the following notation, we omit the subscript $t$ for $\x,\w$ and $\Zb$, and look at the difference of neighbors for $1<i\leq n$. Here, observe that ``given ${\cal{F}}_i:=\{\x_0,(\w_k)_{k=1}^{i}\}$'' implies $\{\x_0,\cdots,\x_i\}$ are known with respect to this filtration.
\begin{align}
	|X_{t,i} - X_{t,i-1}| &= \left|\E\bigg[\frac{1}{n}\sum_{j=1}^n \ell(\x_{j},\fal{j-1}{}(\x_{j-1})) \bigg| \x_0, (\w_k)_{k=1}^{i} \bigg] - \E\bigg[\frac{1}{n}\sum_{j=1}^n \ell(\x_{j},\fal{j-1}{}(\x_{j-1})) \bigg| \x_0, (\w_k)_{k=1}^{i-1}\bigg]\right|\nonumber \\
	&\scalemath{1}{\leq\frac{1}{n}\sum_{j=i}^n\left| \E\left[\ell (\x_{j}, \fal{j-1}{}(\x_{j-1})) \bigg| \x_0, (\w_k)_{k=1}^{i}\right] - \E\left[\ell(\x_{j}, \fal{j-1}{}(\x_{j-1}))\bigg| \x_0, (\w_k)_{k=1}^{i-1}\right]\right|} \nonumber \\
	&\scalemath{1}{\stackrel{(a)}{\leq}\frac{B}{n} + \frac{1}{n}\sum_{j=i+1}^n \left|\E\left[\ell (\x_{j}, \fal{j-1}{}(\x_{j-1})) \bigg| \x_0, (\w_k)_{k=1}^{i}\right] - \E\left[\ell(\x_{j}, \fal{j-1}{}(\x_{j-1}))\bigg| \x_0, (\w_k)_{k=1}^{i-1}\right]\right|} \nonumber 
\end{align}
Here, $(a)$ follows from the fact that loss function $\ell(\cdot,\cdot)$ is bounded over $[0,B]$. 
To proceed, call the right side terms $D_{ji}:=|\E[\ell (\x_j, \fal{j-1}{}(\x_{j-1})) \big| \x_0, (\w_k)_{k=1}^i] -\E[\ell (\x_j, \fal{j-1}{}(\x_{j-1})) \big| \x_0, (\w_k)_{k=1}^{i-1}]|$. We now use the fact that $D_j$ is an expectation over the sequence pairs that differ exactly at $\w_i$. For any realization $\x'_0,(\w^{'}_k)_{k=1}^i$, we use the first part of Lemma \ref{dynamic assump lemma} to obtain 
\begin{align}
    \bigg|\E[\ell (\x_j, \fal{j-1}{}(\x_{j-1})) &\big| \x_0', (\w'_k)_{k=1}^i, (\w_k)_{k=i+1}^n] \nonumber \\ -&\E[\ell (\x_j, \fal{j-1}{}(\x_{j-1})) \big| \x'_0, (\w'_k)_{k=1}^{i-1}, (\w_k)_{k=i}^n  ]\bigg| \leq \frac{K}{j-\red{i}} \frac{2\bar{C}_{\rho}\bar{w}}{1- \bar{\rho}}.  \nonumber
\end{align}
Now taking expectation over $(\w_k)_{k=i}^n$, we obtain 
\begin{align}
    D_{ji} \leq \frac{K}{j-\red{i}} \frac{2\bar{C}_{\rho}\bar{w}}{1- \bar{\rho}}. \nonumber
\end{align}
Combining above, for any $n \geq i > 1$, we obtain
\[
|X_{t,i} - X_{t,i-1}|\leq \frac{B}{n} + \frac{1}{n}\sum_{j=i+1}^{n} \frac{K}{j-\red{i}} \frac{2\bar{C}_{\rho}\bar{w}}{1- \bar{\rho}}< \frac{B}{n} + \frac{K \log n}{n}\frac{2\bar{C}_{\rho}\bar{w}}{1- \bar{\rho}}.
\]
If we use the same argument as above and apply the second part of Lemma \ref{dynamic assump lemma}, we obtain the following bound for $|X_{t,1}-X_{t,0}|$:
\[
|X_{t,1} - X_{t,0}|< \frac{B}{n} + \frac{K \log n}{n}\frac{2\bar{C}_{\rho}(\bar{w} + \bar{x})}{1- \bar{\rho}}.
\]
Moreover, as the loss function is bounded by $B$, we have
\[
|X_{t,n} - X_{t,n-1}|\leq \frac{B}{n} < \frac{B}{n} + \frac{K \log n}{n}\frac{2\bar{C}_{\rho}\bar{w}}{1- \bar{\rho}}.
\]
Note that $|\Lc_t(\bal) - \Lch_t(\bal)| = | X_{t,0} - X_{t,n}|$ and for every $t \in [T]$, we obtain
\[
    \scalemath{1}{\sum_{i=1}^n\left|X_{t,i}-X_{t,i-1}\right|^2 \leq\frac{(n-1)\left(B+K \frac{2\bar{C}_{\rho}\bar{w}}{1- \bar{\rho}}\log n\right)^2+\left(B+K\frac{2\bar{C}_{\rho}(\bar{w}+\bar{x})}{1- \bar{\rho}}\log n\right)^2}{n^2}\leq\frac{\bigg(B+2K\frac{\bar{C}_{\rho}(\bar{w}+\bar{x}/\sqrt{n})}{1- \bar{\rho}}\log n\bigg)^2}{n}}.
\]
Armed with this bound on increments, we can now apply Azuma-Hoeffding and obtain the result equivalent to Eq.~\eqref{azuma prob} in the proof of Theorem~\ref{thm mtl} by swapping $K$ with $\bar K=2K\frac{{\bar C}_\rho}{1-\bar\rho}(\bar w+\bar x /\sqrt{n})$.

\noindent\textbf{Step 2: Next, we turn to bound {\normalfont$\sup_{\bal\in\Ac}|\Lc(\bal)-\Lch(\bal)|$} where $\Ac$ is assumed to be a continuous search space.} We follow the analysis in Step 2 of the proof of Theorem~\ref{thm mtl} verbatim: By applying an $\eps$-covering argument in an identical fashion (e.g.~until obtaining \eqref{basic cover bound}), we conclude with the result.
\end{proof}
\section{Model Selection and Approximation Error Analysis}\label{app approx}
To proceed with our analysis, we need to make assumptions about what kind of algorithms are realizable by transformers. Given ERM is the work-horse of modern machine learning with general hypothesis classes, we assume that transformers can approximately perform in-context ERM.  
Hypothesis~\ref{hypo erm} states that the algorithms induced by the transformer can compete with empirical risk minimization over a family of hypothesis classes. 

With this hypothesis, instead of searching over the entire hypothesis space $\Fca:=\bigcup_{h=1}^H\Fc_i$, given prompt length $m$ we search over the hypothesis space $\Fc_{h_m}$ only, and $\dim(\Fc_{h_m})\leq\dim(\Fca)$ where $\dim(\cdot)$ captures the complexity of a hypothesis class.

In Hypothesis~\ref{hypo erm}, we assume that $\FB$ is a family of countable hypothesis classes with $|\FB|=H$. As stated in Section~\ref{sec approx}, $\FB$ is not necessary to be discrete. The following provides some examples of $\FB$, where the first three correspond to discrete model selection whereas the left are continuous.
\begin{itemize}
    \item $\FB^{\text{sparse}}=\{\Fc_s: s\text{-sparse linear model}\}$,
    \item $\FB^{\text{NN}}=\{\Fc_s: 2\text{-layer neural net with width $s$}\}$,
    \item $\FB^{\text{RF}}=\{\Fc_s: \text{Random forest with $s$ trees}\}$,
    \item $\FB^{\text{ridge}}=\{\Fc_\lambda: \text{Linear model with parameter bounded by $\tn{\bt}\leq\lambda$}\}$\quad\quad(akin to ridge regression)
    \item $\FB^{\text{weighted}}=\{\Fc_\bSi:\text{Linear model with covariance-prior $\bSi$, $\bt^\top\bSi^{-1}\bt\leq 1$}\}$\quad\quad(akin to weighted ridge).
\end{itemize}

{To proceed, let us introduce the following classical result that controls the test risk of an ERM solution in terms of the Rademacher complexity \cite{mohri2018foundations,maurer2016vector}.}
\begin{theorem}
\label{thm F risk bound}
    Let $\Fc:\Xc\rightarrow\Yc$ be a hypothesis set and let $\Sc=(\x_i,\y_i)_{i=1}^n\in\Xc\times\Yc$ be a dataset sampled i.i.d. from distribution $\Dc$. Let $\ell(\y,\hat \y)$ be a loss function takes values in $[0,B]$. Here $\ell(\y,\cdot)$ is $L$-Lipschitz in terms of Euclidean norm for all $\y\in\Yc$. Consider a learning problem that 
    \begin{align}
    \hat f:=\arg\min_{f\in\Fc}\frac{1}{n	}\sum_{i=1}^n\ell(\y_i,f(\x_i)).\label{erm prob}
    \end{align}
    Let $\Lc^\st=\min_{f\in\Fc}\Lc(f)$ where $\Lc(f)=\E[\ell(\y,f(\x))]$. Then we have that with probability at least $1-2\delta$, the excess test risk obeys
    \[
    \Lc(\hat f)-\Lc^\st\leq8L\Rc_n(\Fc)+4B\sqrt{\frac{\log\frac{1}{\delta}}{n}},
    \]
    where $\Rc_n(\Fc)=\E_\Sc\E_{\bsi_i}[\sup_{f\in\Fc}\frac{1}{n}\sum_{i=1}^n\bsi_i^\top f(\x_i)]$ is the Rademacher complexity of $\Fc$  \cite{mohri2018foundations} and $\bsi_i$'s are vectors with Rademacher random variable in each entry.
\end{theorem}

\begin{lemma}[Formal version of Observation~\ref{thm approx}] Let $\Lc_\Tc^\st:=\min_{\bal\in\Ac}\Lc_\Tc(\bal)$ be the optimal target risk as stated in Section~\ref{sec setup}. Assume that Hypothesis~\ref{hypo erm} holds, 
then the approximation error obeys
    \begin{align}
        \Lc_\Tc^\st\leq\frac{1}{n}\sum_{m=1}^{n-1}\min_{h\in[H]}\left\{\Lc_h^\st+8L\Rc_m(\Fc_h)+\ept^{h,m}\right\}+\frac{cB}{\sqrt{n}},    
    \end{align}
where $\Rc_m(\Fc)$ is the Rademacher complexity over data distribution $\Dc_\Tc$, and $\Lc^\st_{{h}}=\min_{f\in\Fc_h}\E[\ell(\y,f(\x))]$.
\end{lemma}
\begin{proof} Let us assume Hypothesis~\ref{hypo erm} holds for algorithm $\bat\in\Ac$. 
Since $\Lc^\st_\Tc$ is the minimal test loss, we have that
\begin{align*}
    \Lc_\Tc^\st\leq\Lc_\Tc(\bat)=\E_{\Sc_\Tc}\left[\frac{1}{n}\sum_{i=1}^n\ell(\y_i,f^\bat_{\Scnt{i-1}{}}(\x_i))\right]=\frac{1}{n}\sum_{i=1}^n\E_{(\x,\y,\Scnt{i-1}{})}\left[\ell(\y,f^\bat_{\Scnt{i-1}{}}(\x))\right].
\end{align*}
Then by directly applying Hypothesis~\ref{hypo erm} we have that 
\begin{align}
    \Lc_\Tc^\st&\leq\frac{1}{n}\E_{(\x,\y)}\left[\ell(\y,f^\bat_{\Scnt{0}{}}(\x))\right]+\frac{1}{n}\sum_{i=1}^{n-1}\E_{(\x,\y,\Scnt{i}{})}\left[\ell(\y,f^\bat_{\Scnt{i}{}}(\x))\right]\\
    &\leq\frac{B}{n}+\frac{1}{n}\sum_{m=1}^{n-1}\min_{h\in[H]}\left\{\riskhm+\ept^{h,m}\right\}.\label{bound first loss}
\end{align}
Here the first term in \eqref{bound first loss} comes from the fact that loss function is bounded by $B$, and we assume $\Scnt{0}{}=\emptyset$, and the second term follows the Hypothesis~\ref{hypo erm}. Next, we turn to bound $\riskhm$. To proceed, let $X^{h,m}:=\E_{(\x,\y)}\left[\ell(\y,\hat f^{(h)}_{\Scnt{m}{}}(\x))\right]$ be the random variables, where we have $|X^{h,m}|\leq B$. 
Following Theorem~\ref{thm F risk bound}, we have that for any $m\in[n],h\in[H]$
\begin{align*}
    \P\left(X^{h,m}-\Lc^\st_h-8L\Rc_m(\Fc_h)\geq\tau\right)\leq2e^{-\frac{m\tau^2}{16B^2}}.
\end{align*}
{The upper-tail bound of the last line implies that there exists an absolute constant $c>0$ such that
\begin{align*}
    \riskhm=\E_{\Scnt{m}{}}\left[X^{h,m}\right]\leq\Lc^\st_h+8L\Rc_m(\Fc_h)+\frac{cB}{\sqrt{m}}.
\end{align*}
}
Combining it with \eqref{bound first loss} and following the evidence $\sum_{m=1}^n\frac{1}{\sqrt{m}}\leq2\sqrt{n}$ complete the proof.
\end{proof}

\section{Further Related Work on Multitask/Meta learning}\label{app related}


In order for ICL to work well, the transformer model needs to train with large amounts of related prompt instances. This makes it inherently connected to meta learning \cite{finn2017model}. However, a key distinction is that, in ICL, adaptation to a new task happens implicitly through input prompt. Our analysis has some parallels with recent literature on multitask representation learning \cite{maurer2016benefit,du2020few,tripuraneni2020theory,cheng2022provable,li2022provable,kong2020meta,qin2022non,tripuraneni2021provable,collins2022maml,modi2021joint,faradonbeh2022joint,zhang2022multi} since we develop excess MTL risk bounds by training the model with $T$ tasks and quantify these bounds in terms of complexity of the hypothesis space (i.e.~transformer architecture), the number of tasks $T$, and the number of samples per task. In relation to \cite{sun2021towards,chen2022understanding}, our experiments on linear regression with covariance-prior (Figure \ref{fig:main_exp}(b)) demonstrate ICL's ability to implicitly implement optimally-weighted linear representations.

\end{document}
